\documentclass[conference]{IEEEtran}
\usepackage{times}

\usepackage[numbers]{natbib}
\usepackage{multicol}
\usepackage[bookmarks=true]{hyperref}
\usepackage{float}
\usepackage{amsmath}
\usepackage{amsfonts}
\newtheorem{theorem}{Theorem}
\usepackage{booktabs}
\usepackage{amssymb}
\usepackage{textcomp}
\usepackage{multirow}

\renewcommand{\IEEEQED}{}

\usepackage{xspace}
\usepackage{bbm}
\usepackage{graphicx}
\usepackage{subfig} 
\usepackage{annotate-equations}
\definecolor{dark_red}{RGB}{122, 0, 0}
\definecolor{coral}{RGB}{255, 119, 94}
\definecolor{pink_orange}{RGB}{255, 72, 126}
\definecolor{vibrant_pink}{RGB}{255, 0, 104}
\definecolor{pink_pink}{RGB}{255, 37, 153}
\definecolor{wine}{RGB}{204, 0, 102}

\definecolor{light_orange}{RGB}{255, 198, 107}
\definecolor{orange(sae/ece)}{rgb}{1.0, 0.49, 0.0}
\definecolor{dark_orange}{RGB}{216,92,0}

\definecolor{org-purp-0}{RGB}{165, 76, 0}
\definecolor{org-purp-1}{RGB}{250, 130, 28}
\definecolor{org-purp-2}{RGB}{226, 89, 68}
\definecolor{org-purp-3}{RGB}{206, 92, 124}
\definecolor{org-purp-4}{RGB}{116, 80, 146}
\definecolor{org-purp-5}{RGB}{110, 78, 157}

\definecolor{teal(sae/ece)}{rgb}{0, 0.47, 0.52}
\definecolor{aqua}{RGB}{52,172,139}
\definecolor{dark_aqua}{RGB}{35,115,93}
\definecolor{dark_green}{RGB}{0, 92, 34}
\definecolor{solid_green}{RGB}{39, 163, 50}

\definecolor{grape}{RGB}{112,48,160}
\definecolor{purple}{rgb}{0.74, 0.65, 1.0}
\definecolor{dark_purple}{rgb}{0.58, 0.0, 0.82}
\definecolor{periwinkle}{RGB}{191, 140, 230}

\definecolor{light_gray}{rgb}{0.9, 0.9, 0.9}
\definecolor{medium_gray}{rgb}{0.6, 0.6, 0.6} 
\definecolor{dark_gray}{rgb}{0.2, 0.2, 0.2} 

\definecolor{sky_blue}{RGB}{37, 166, 213}
\definecolor{light_blue}{rgb}{0.33, 0.80, 1}
\definecolor{dark_blue}{rgb}{0.098, 0.239, 0.52}
\definecolor{ocean}{RGB}{13, 121, 202}
\definecolor{light_ocean}{RGB}{18, 178, 235}
\definecolor{dark_ocean}{RGB}{10, 89, 148}
\definecolor{vibrant_blue}{RGB}{14, 120, 255}

\definecolor{dark_brown}{rgb}{0.3255, 0.004, 0.001}

\newcommand{\para}[1]{\medskip\noindent\textbf{#1. }}

\newcounter{qnum}
\newcounter{tnum}
\newcommand{\obs}{o}
\newcommand{\img}{I}
\newcommand{\proprio}{s}
\newcommand{\imgSpace}{\mathcal{I}}
\newcommand{\proprioSpace}{\mathcal{S}}
\newcommand{\obstraj}{\mathbf{o}}
\newcommand{\obsSpace}{\mathcal{O}}

\newcommand{\enc}{\mathcal{E}_\phi}
\newcommand{\dec}{\mathcal{D}_\phi}
\newcommand{\dyn}{f_\phi}
\newcommand{\world}{\mathcal{W}_\phi}

\newcommand{\task}{\mathcal{L}}
\newcommand{\lang}{\ell}

\newcommand{\vlm}{\text{VLM}}
\newcommand{\vlmtrans}{\mathcal{T}^{\vlm}}

\newcommand{\latentSpace}{\mathcal{Z}}

\newcommand{\action}{a}
\newcommand{\acttraj}{\mathbf{\action}}
\newcommand{\actSpace}{\mathcal{A}}

\usepackage{amsmath}
\DeclareMathOperator*{\argmax}{arg\,max}

\newcommand{\cpinput}{x}
\newcommand{\cpseqinput}{x}
\newcommand{\score}{\mathbb{P}^{\text{VLM}}}

\begin{document}

\title{When to Act, Ask, or Learn: \\ Uncertainty-Aware Policy Steering}

\author{Jessie Yuan$^{1*}$, \quad Yilin Wu$^{1*}$, \quad Andrea Bajcsy$^{1}$ \\
$^{1}$Carnegie Mellon University \quad $^{*}$Equal Contribution}

\maketitle

\begin{abstract}

Policy steering is an emerging way to adapt robot behaviors at deployment-time: a learned verifier analyzes low-level action samples proposed by a pre-trained policy (e.g., diffusion policy) and selects only those aligned with the task. While Vision-Language Models (VLMs) are promising general-purpose verifiers due to their reasoning capabilities, existing frameworks often assume these models are well-calibrated. In practice, the overconfident judgment from VLM can degrade the steering performance under both high-level semantic uncertainty in task specifications and low-level action uncertainty or incapability of the pre-trained policy. We propose \textit{uncertainty-aware policy steering} (UPS), a framework that jointly reasons about semantic task uncertainty and low-level action feasibility, and selects an uncertainty resolution strategy: execute a high-confidence action, clarify task ambiguity via natural language queries, or ask for action interventions to correct the low-level policy when it is deemed incapable at the task. We leverage conformal prediction to calibrate the composition of the VLM and the pre-trained base policy, providing statistical assurances that the verifier selects the correct strategy. After collecting interventions during deployment, we employ residual learning to improve the capability of the pre-trained policy, enabling the system to learn continually but with minimal expensive human feedback. We demonstrate our framework through experiments in simulation and on hardware, showing that UPS can disentangle confident, ambiguous, and incapable scenarios and minimizes expensive user interventions compared to uncalibrated baselines and prior human- or robot-gated continual learning approaches. Videos can be found at: \href{https://jessie-yuan.github.io/ups/}{https://jessie-yuan.github.io/ups/}.

\end{abstract}

\IEEEpeerreviewmaketitle
\section{Introduction}



Imitation-based generative robot policies, such as diffusion policies and vision-language-action (VLA) models, are a promising way to model low-level action distributions conditioned on perceptual (and language) inputs. 
However, directly executing samples from the learned action distribution does not always lead to success at deployment time. 
This has motivated the paradigm of \textit{policy steering} which adapts the behavior of a pre-trained policy online, e.g., via sampling and verification. 

Recent works \cite{wu2025foresight, kwok2025robomonkey} demonstrated that vision-language models (VLMs) can be harnessed as ``open world'' verifiers during policy steering, reasoning about the outcomes of action samples and selecting those which align most with task instructions. 
However, to-date, all approaches have implicitly assumed that the verifier is well-calibrated:
i.e., it selects an appropriate action sample (or correctly rejects all candidates) with high confidence. 
In practice, this assumption often fails. 
VLM verifiers can be overconfident when task instructions are ambiguous or under-specified, leading to confident selection of misaligned behaviors. 
Moreover, when the low-level policy is fundamentally \textit{incapable} of accomplishing the task under the current conditions, all candidate samples may be unsatisfactory. Our experiments show that an uncalibrated verifier may still select one of these faulty options rather than 
identifying the limitation and stopping execution. 


In this work, we propose \textit{\textbf{U}ncertainty-aware \textbf{P}olicy \textbf{S}teering} (\textbf{UPS}), a method that maps a VLM verifier's uncertainty over action samples into a resolution strategy: \textbf{execute} a high-confidence action sample, \textbf{clarify} task ambiguity through natural-language interactions, or request to \textbf{re-train} the low-level policy when it is incapable.
Specifically, we calibrate the composition of the VLM verifier and the pre-trained policy using conformal prediction \cite{angelopoulos2023conformal}, returning a \textit{set} of action sample(s) and possibly a ``none of the above'' option.  
This set is constructed such that, with a user-specified level $1-\epsilon$, the verifier can steer the low-level policy to success in $1-\epsilon$ of scenarios by either executing the right action sample or selecting the appropriate resolution strategy (asking clarification questions when ambiguous or eliciting human interventions when incapable).
Our conformal prediction approach also minimizes the average size of prediction sets, ensuring that the robot selectively chooses to ask semantic clarification questions (when set size is $>1$) or chooses to ask for low-level retraining (when  set only contains  ``none of the above''). 
When low-level human intervention is required, we improve the policy via residual learning~\cite{jiang2025transic,xiao2025self}, enabling continual improvement of the robot's capabilities.

We evaluate UPS in a series of simulation and hardware experiments with robotic manipulation. Compared to prior policy steering methods that do \textit{not} calibrate the verifier, we show $30\%$ improvement in ambiguous scenarios. On the uncertainty-quantification side, we show that our approach to conformal prediction, including a new score function with VLMs, yields strong empirical coverage guarantees with tight prediction sets compared to prior UQ methods for VLMs. On the continual learning side, we show that UPS's ability to reason about both high-level semantic uncertainty and low-level action uncertainty enables it to more effectively balance cheap language queries vs. expensive low-level action interventions from a human. Furthermore, compared to human-\cite{kelly2019hg} or robot-gated~\cite{menda2019ensembledagger} interactive imitation learning approaches, the data collected by UPS
enables the robot to improve over successive deployment rounds while minimizing human effort.

\section{Related Work}
\para{Policy Steering in Robotics}Recent work increasingly use VLMs as deployment-time verifiers to steer robot policies by scoring multiple sampled trajectories \cite{wu2025foresight,kwok2025robomonkey,wu2025you}. However, these pipelines suffer from the same problems as LLM-based verification, including poor calibration and systematic agreement bias, where plausible but incorrect actions are blindly accepted \cite{liu2024calibrating,andrade2025let}. These failures are especially problematic in robotics because verifier outputs are often reduced to a scalar score  (accept/reject decision), obscuring \emph{what} is uncertain and \emph{how} the system should respond. 
Our work addresses this gap by quantifying uncertainty in VLM verification and mapping different uncertainty modes to distinct resolution strategies, from high-level clarification to low-level data collection. 

\para{Uncertainty Quantification for Learned Robot Policies} While uncertainty quantification (UQ) has been a long-standing problem in robotics and machine learning~\cite{vovk2005algorithmic,gal2016dropout}, in this paper, we focus on learning-based robot policies that generate low-level actions from multimodal inputs~\cite{chi2024diffusionpolicy, ren2023robots} such as images and language instructions. Prior work typically addresses UQ at two decoupled levels of abstraction: semantic ambiguity and low-level action uncertainty.
For high-level semantic uncertainty, approaches like KnowNo~\cite{ren2023robots} and RCIP~\cite{lidard2024risk} employ conformal prediction to provide statistical guarantees that the generated plans (or intent predictions) capture the user's true instruction. However, these methods largely operate in discrete textual or symbolic space with an unrealistic assumption that the low-level policy can always execute the chosen plan.
A separate body of work quantifies uncertainty directly within the continuous action space of learned policies, such as VLA or diffusion models~\cite{valle2025evaluating, zollo2025confidence}. 
While these methods can flag distribution shifts or execution noise~\cite{zhao25conformaldagger}, low-level action uncertainty does not always correlate with task-level failures (e.g., there are multiple ways to successfully grasp a cup). 
Our work takes a step towards more tightly coupling the uncertainty of the low-level action policy (approximated by repeatedly sampling from the policy) with high-level semantic uncertainty about the task (quantified by how well the action outcomes align with the user's instruction).

\para{Continual Learning via Interactive Imitation} Interactive Imitation Learning (IIL) improves policies through human feedback. 
The DAgger family of methods~\cite{ross2011reduction,hoquethriftydagger} typically relies on the human operator to know when to intervene, or on action-level uncertainty signals to trigger interventions, which has challenges as described in the above section. 
In contrast, our approach leverages calibration to capture semantic uncertainty over different meaningful phases (e.g., grasping, placing) to minimize human feedback.
Recent work also uses the imagination of a world model to evaluate long-horizon execution failures~\cite{liumulti}.
However, these methods typically evaluate a single policy trajectory. Because they do not sample multiple rollouts within the world model, they can't effectively ``search'' through the base policy's distribution to find an alternative action plan; this can result in asking for human interventions even when a feasible solution exists within the support of the policy distribution.
Our method mitigates this by performing action sampling and calibrated verification.
Finally, to adapt efficiently from human intervention, recent works utilize residual policy learning~\cite{jiang2025transic,xiao2025self} instead of fine-tuning the entire base model.
We adopt this paradigm as it enables our system to integrate user corrections into the generation-imagination loop without destroying the diverse capabilities of the base policy.
\begin{figure*}[t!]
    \centering
    \includegraphics[width=0.9\textwidth]{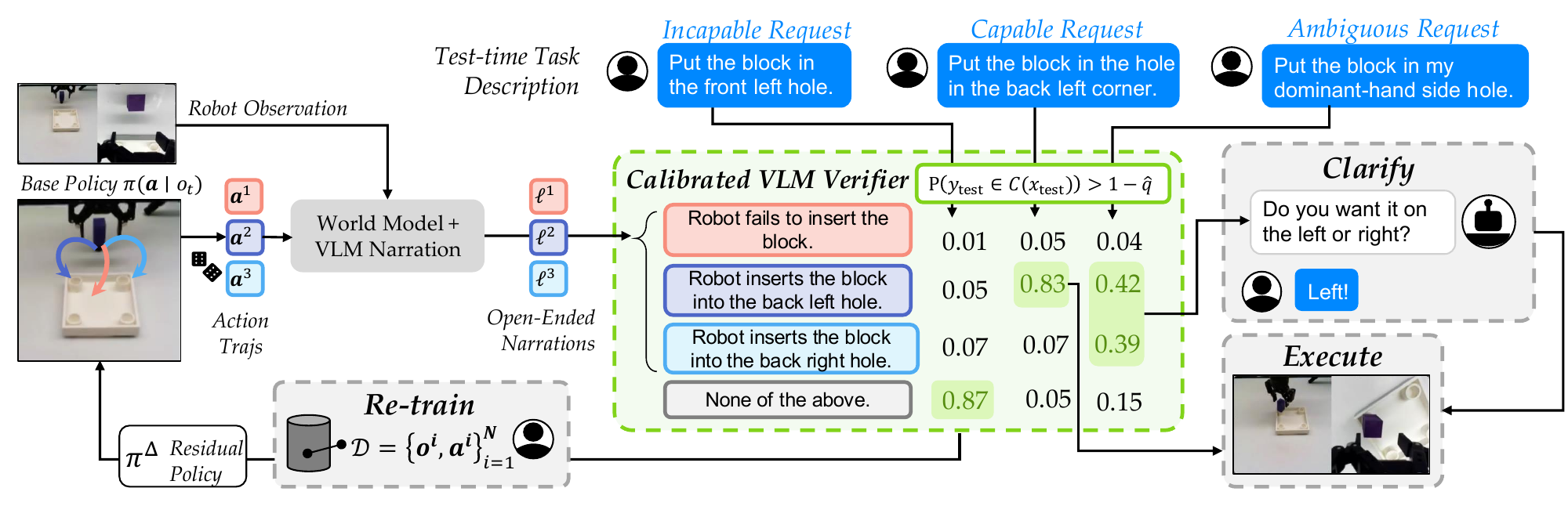}
    \caption{\textbf{Uncertainty-Aware Policy Steering.} Our framework calibrates the VLM verifier used for policy steering via conformal prediction. This enables the VLM to select an appropriate way to resolve uncertainty, from querying the end-user in natural language to asking to re-train the low-level control policy.}
    \vspace{-0.9em}
    \label{fig:framework}
\end{figure*}

\section{Problem Formulation}
\label{sec:problem_formulation}
\para{Setup \& Notation} 
We define the user’s task instruction $\task$ as a single sentence defining the goal of a long-horizon task that involves multiple $T$-timestep subtasks. 
The robot’s observation space $\obs \in \obsSpace := \imgSpace \times \proprioSpace$ integrates RGB images $\img$ with proprioceptive states $\proprio$, such as end-effector poses. Control is governed by a pre-trained base policy $\pi(\action \mid \obs_t)$, typically a generative visuomotor model such as a diffusion policy \citep{chi2024diffusionpolicy} or a vision-language-action model \cite{black2024pi_0}. 
At any timestep $t$, the policy generates a short action chunk $\action_{t:t+H}$ ($H \ll T$), which is aggregated into a full sequence $\acttraj_t = \action_{t:t+T}$ with multiple continuous generations. Given an observation $\obs_t$, sampling from the policy $\acttraj_t \sim \pi(\cdot \mid \obs_t)$ and executing the plan yields a corresponding future observation sequence $\obstraj_t \sim \mathbb{P}(\cdot \mid \obs_t, \acttraj_t)$. This formulation allows us to evaluate the policy's predicted outcomes against the high-level task instruction across multiple stages of execution.

\label{sec:problem_formulation_problem}
\para{Problem} Given the robot's current observation $\obs$ and $K$ \textit{i.i.d.} samples from the base policy, $\{\acttraj^i\}^K_{k=1} \sim \pi(\acttraj \mid \obs)$, our \textit{uncertainty-aware policy steering} problem seeks to return the index $y \in \mathcal{Y}=\{1,...,K, K+1\}$ of the action sample that best matches the textual task description $\task$ \textit{or} return a $K+1$-th index which indicates ``none of the above'' (i.e., no samples accomplish the instructed task). 
More formally, we seek to solve the following optimization problem:
\begin{equation}
\vspace{-0.1cm}
     y^{\star} = \arg\max_{y \in \mathcal{Y}}  \mathbb{P}(y \mid \{\acttraj^k\}^K_{k=1}, \obs; \task ).
    \label{eq:formulation-problem}
\end{equation}
In general, the distribution $\mathbb{P}(y \mid \{\acttraj^k\}^K_{k=1}, \obs; \task )$ above is very hard to characterize, which is why prior works  \cite{ren2023robots, wu2025foresight} have turned to vision-language models (VLMs) due to their ability to simultaneously reason about both natural language instructions and visual observations. However, prior work \cite{wu2025foresight} shows that effective zero-shot use of VLMs to approximate $\mathbb{P}(y \mid \{\acttraj^k\}^K_{k=1}, \obs; \task )$ is not trivial, since the model must implicitly translate low-level actions into outcomes (i.e., predict $\{\acttraj^k\}^K_{k=1} \rightarrow \{\obstraj^k\}^K_{k=1}$), reason about the outcomes (i.e., alignment between  the outcomes $\{\obstraj^i\}^K_{k=1}$ and the text of the task $\task$), as well as select the correct index. 

Following prior work \cite{wu2025foresight, wu2025you}, we factorize the problem into two stages: first,  predict the outcome of each low-level action and translate the outcomes into textual \textit{outcome narrations}; given these narrations, we ask a VLM to solve a multiple-choice Q\&A problem and select the action whose outcome narration best aligns with the task. Mathematically, our problem becomes:
\vspace{-0.2cm}
\begin{equation}
     y^{\star} = \arg\max_{y \in \mathcal{Y}} \mathbb{E}_{\eqnmarkbox[gray]{intent}{\{\lang^k\}^K_{k=1} \sim \mathbb{P}(\cdot \mid} \eqnmarkbox[red]{actions}{\{\acttraj^k\}^K_{k=1}}\eqnmarkbox[gray]{intent2}{, \obs)} } \Big[\eqnmarkbox[solid_green]{vlm}{\mathbb{P}(y \mid \{\lang^k\}^K_{k=1}; \task )}\Big].
    \label{eq:formulation-factorized}
\end{equation}
\annotate[xshift=-0.0em, yshift=-0.0em]{below,left}{intent}{\rmfamily{outcome narrations}}
\annotate[yshift=+0.5em, xshift=-0.8em]{above,right}{actions}{\rmfamily{$\sim\pi(\cdot \mid \obs)$}}
\annotate[yshift=-0.5em]{below, left}{vlm}{\rmfamily{VLM verifier}}

\bigskip 
\noindent Here, each $\lang^k$ is a language description of the outcome induced by the low-level action $\acttraj^k$ executed from an initial observation. For example, $\lang^k = \text{``The robot inserts the block into the back right hole.''}$ as shown in Fig.~\ref{fig:framework}. 
In practice, these narrations are obtained by first \textit{predicting future observations}  $\obstraj_t 
\sim \mathbb{P}(\cdot| \obs_t, \acttraj_t)$ (e.g., via a world model, shown in Sec.~\ref{sec:approach}), and then the predicted observations are described in text via video captioning \cite{team2023gemini}. 
By lifting the outcomes of low-level actions into textual descriptions, the VLM's reasoning capabilities can be used more effectively to verify which index $y \in \mathcal{Y}$ to choose; in other words, the VLM can be used as a proposal distribution in the optimization objective: $\mathbb{P}(y \mid \{\lang^k\}^K_{k=1}; \task )$.

\para{Challenges \& Opportunities}
\label{sec:problem_formulation_challenges}
A central challenge with Eq.~\ref{eq:formulation-factorized} is the assumption of a well-calibrated VLM verifier. 
Even with perfect outcome narrations to choose from, VLMs are often overconfident evaluators \cite{bai2024hallucination, mei2025reasoning} which can undermine the quality of the policy steering. 
Moreover, the task descriptions $\task$ from the user may be inherently \textit{ambiguous} or \textit{underspecified}; a VLM verifier which confidently selects an incorrect behavior narration can lead to misaligned or even unsafe actions being executed. A second challenge is that when the policy is \textit{incapable}, all $K$ action samples from the policy in Eq.~\ref{eq:formulation-factorized} will have unsatisfactory outcomes for the task instruction $\task$; an uncalibrated VLM verifier may choose one of the faulty samples instead of rejecting them all.

Finally, even if uncertainty or incapability is detected, the robot should \textit{not} always stop operation and instead should generate an appropriate resolution strategy. However, actionable strategies for resolving uncertainty can range from simple clarifications with the user (e.g., ``Did you want it on the left or right?'') to far more demanding interventions such as re-training the low-level control policy (e.g., refining the precision of the skill to insert the block into the hole or learning a new skill for placing into a new hole). 
In other words, the source of uncertainty can lie at different levels of abstraction, from \textit{semantic ambiguity} (e.g., multiple indices seem plausible for the task) to the base policy's \textit{incapability}. 
Our unified framework seeks to both calibrate the VLM verifier and equip the robot with appropriate resolution strategies.

\para{Aside: Comparison to KnowNo \cite{ren2023robots}} 
A closely related work \cite{ren2023robots} similarly performs uncertainty quantification on an LLM task planner that reasons about textual plans and asks questions when highly uncertain. 
However, it assumes an idealized low-level policy that executes any chosen plan accurately.  
We relax this assumption by grounding the uncertainty quantification in the capabilities of low-level policy.
We achieve this because our textual plans (that the VLM verifier reasons about) directly correspond to outcomes induced by samples from low-level policy
in Eq.~\ref{eq:formulation-factorized}, as well as an explicit ``incapability'' option. 
Thus, our uncertainty quantification jointly reasons about high-level semantic intent uncertainty and low-level action uncertainty. 
Together, our system not only resolves what the user wants, but also identifies whether the policy can reliably accomplish the task. 

\section{Approach: Uncertainty-Aware Policy Steering}
\label{sec:approach}

Our uncertainty-aware policy steering approach calibrates the VLM verifier to select, with high user-specified confidence, one of three uncertainty resolution strategies: \textbf{execute} a high-confidence action that fulfills the task; \textbf{clarify} task ambiguity when multiple actions seem plausible; or \textbf{re-train} the low-level policy via interactive imitation learning when it is incapable. 
Our overall framework
is visualized in Fig.~\ref{fig:framework}, and we describe each component in the following subsections. 

\subsection{VLM-in-the-loop Policy Steering}
We follow the framework from~\cite{wu2025foresight,wu2025you} and pose VLM-in-the-loop steering as an open-ended multiple-choice Q\&A problem. This involves three stages: predicting the outcomes $\obstraj$ of an action sample $\acttraj$ with a world model, generating behavior narrations $\lang$ for the predicted outcomes, and selecting from these options with the VLM verifier.

\para{Outcome Prediction}
We leverage latent world models ~\cite{hafner2023mastering, zhoudino} to predict the outcomes of low-level action samples directly from a high-dimensional observation, $\obs_t$.  
The world model $\world$ consists of an encoder $\enc: \obsSpace \rightarrow \latentSpace$  which encodes an observation $\obs_t$, a latent dynamics model $\dyn: \latentSpace \times \actSpace \rightarrow \latentSpace$ which  predicts latent state transitions, and a decoder $\dec: \latentSpace \rightarrow \obsSpace$ which reconstructs an observation from latent space. 
Recall that most low-level policies \cite{chi2024diffusionpolicy, black2024pi_0} predict relatively short-horizon action ``chunks''. This presents a challenge for policy steering, since on very short time-horizons, the outcomes are often indistinguishable for the verifier. 

To address this, we interleave action generation and imagination to obtain a longer-horizon action-outcome sequence that is informative for the verifier to reason about (visualized in Figure~\ref{fig:interleaved-narration}). 
Specifically, given the current \textit{real} observation $\obs_t$, we query $\action_{t:t+H} \sim \pi(\cdot \mid \obs_t)$ to generate an action chunk of length $H$ and pass this to the world model to obtain predicted future observations: $\hat{\obs}_{t+1:t+H} = \world(\obs_{t}, \action_{t+1:t+H})$. 
The last predicted observation $\hat{o}_{t+H}$ is provided as input to the policy $\pi(\cdot \mid \hat{o}_{t+H})$ to generate the next action chunk.
We repeat this process $T/H$ times and aggregate all action sequences together into one $T$-length sequence 
$\action_{t:t+T}$
with corresponding imagined observations $\hat{\obs}_{t:t+T}$. By executing this process in parallel for $K$ samples, we produce a set of predicted observation sequences $\{\hat{\obs}_{t:t+T}^k\}_{k=1}^K$ that are (approximate) future outcomes. 


\para{Narration}
Next, the imagined outcomes $\{\hat{\obs}_{t:t+T}^k\}_{k=1}^K$ are translated into textual narrations by leveraging the strong reasoning capabilities of the VLM verifier.  
In our experiments, we utilize a Gemini model \cite{team2023gemini} to produce these narrations. Let $\vlmtrans$ be the VLM narration model and for any $\hat{\obstraj}^k_t \in \{\hat{\obstraj}_t^k\}_{k=1}^K$, let the corresponding narration be $\lang_t^k = \vlmtrans(\hat{\obstraj}^k_t, \task)$. 
We augment the set of narrations with a $(K+1)$-th option stating that ``none of the above samples are correct'' to ensure the system can detect incapability.

\begin{figure*}[t!]
    \centering
    \includegraphics[width=0.85\linewidth]{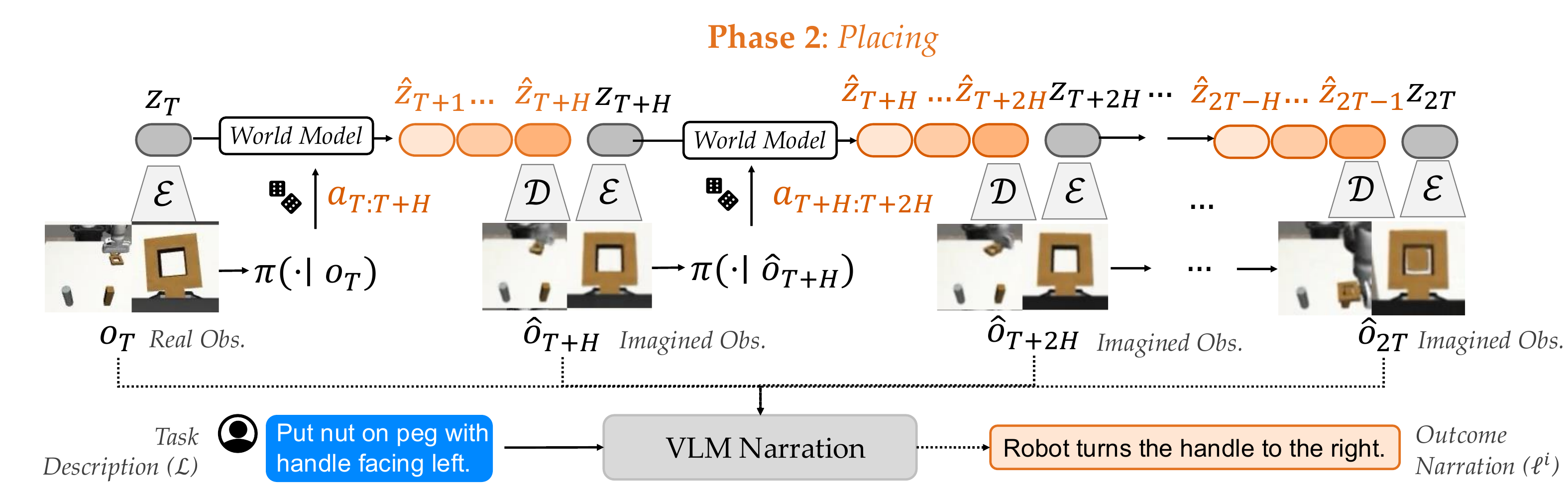}
    \caption{\textbf{Outcome Prediction \& Narration.} The policy and the world model are interleaved to predict long-horizon outcomes induced by the low-level policy. Decoded observations are fed into a VLM which narrates the outcomes in text. }
    \label{fig:interleaved-narration}
    \vspace{-1.4em}
\end{figure*}

\para{Verification}
Finally, the VLM verifier selects an option from the set of narrations (along with the ``none'' option) by answering an open-ended multiple-choice problem over these narrations.
Mathematically, this amounts to the VLM selecting the most likely single token corresponding to one of the multiple choice options: $y_t = \argmax_{y \in \mathcal{Y}} \score(y \mid \{\lang_t^{k}\}_{k=1}^{K+1}, \task)$. 
This formulation aligns with the VLM’s native log-likelihood loss and, as noted in \cite{ren2023robots}, and eliminates length bias when estimating the probability of open-ended textual generations.

\subsection{VLM Verifier Uncertainty Quantification}


We employ conformal prediction (CP) ~\cite{angelopoulos2023conformal} to quantify the VLM verifier's uncertainty during policy steering. 
Specifically, our goal is to use the potentially unreliable model likelihoods in $\score$ to construct tight prediction sets which are provably guaranteed to 
contain the correct action with probability of at least $1-\epsilon$, where $\epsilon$ is a user-defined error budget. 

\para{Conformal Prediction}
Given some \textit{input} $x$ to a prediction model, conformal prediction aims to construct a prediction set $C(\cpinput) \subseteq \mathcal{Y}$ on the \textit{outputs} of that model that contains the true label $y$ with a probability of at least $1-\epsilon$ while minimizing the set size $|C(\cpinput)|$. 
This calibration requires a dataset $(x, y) \in \mathcal{D}_{\text{calib}}$ of independent and identically distributed ($i.i.d.$) input-output pairs drawn from the same distribution as the test domain. 
However, because we are in the sequential-decision-making setting, any input to the VLM verifier depends on prior steering decisions and thus are not $i.i.d.$. 

To maintain formal coverage guarantees despite these temporal dependencies, we perform sequence-level calibration. 
Assume that the VLM verifier is called $M$ times during the task. 
This means we have $M$ true observations that the verifier uses, $M$ sets of behavior narrations, and---for calibration---$M$ true labels about which behavior narration(s) are correct at each timestep. Let the sequence-level calibration dataset be structured as follows:
$$
\{o_i, \{\{\lang^k\}^K_{k=1}\}_i, \task, \mathcal{Y}^*_i\}_{i=0}^M \in \mathcal{D}_{\text{calib}}, \quad |\mathcal{D}_{\text{calib}}| = N.
$$
For notational simplicity, denote any single \textit{input}
$x_i = (o_i, \{\{\lang^k\}^K_{k=1}\}_i, \task)$ and (set of) ground-truth label(s) $\mathcal{Y}^*_i$ obtained from an end-user. Then, denote a \textit{sequence} of inputs and labels as $\bar{x} = (x_0, \hdots, x_M)$ and $\bar{y} = (\mathcal{Y}^*_0, \hdots, \mathcal{Y}^*_M)$.
To construct $\mathcal{D}_{\text{calib}}$, we sample multiple initial observations and different task instructions, 
use our interleaved generation-imagination loop to obtain $K$ candidate narrations and the ``none'' option, annotate the correct set of indices, and execute the correct action; this results in the next observation that the verifier sees and serves as the beginning of next generation-imagination loop, and so on. 

Given this calibration dataset $\mathcal{D}_{\text{calib}}$, we use the VLM verifier's learned distribution $\score$ to compute the \textit{non-conformity score}, which is defined for any data point $(\bar{x},\bar{y}) \in \mathcal{D}_\textrm{calib}$ as:
\begin{equation}
\kappa(\bar{x}, \bar{y}) = 1- \min_{x_i \in \bar{x}}\min_{y \in \mathcal{Y}^*_{i}}\score(y \mid  x_i).
\label{eq:score-func-ours-multi-phase}
\end{equation}
We compute this quantity for all $N$ sequences in the calibration dataset. 
There are a few points worth noting here. 
First, intuitively, for any calibration data point, the higher the score, the less the data ``conforms'' to the training data of the verifier. 
Second, we design this score function to compute a \textit{minimum} over the entire sequence (outer min in Eq.~\ref{eq:score-func-ours-multi-phase}) to 
ensure coverage guarantees despite these temporal dependencies.
Finally, the inner minimium in Eq.~\ref{eq:score-func-ours-multi-phase} is a conservative score designed to penalize the model for not including \emph{all} suitable options in the prediction set. 
From a user perspective, this avoids scenarios where the verifier guesses the user’s unstated intent and incentivizes the verifier to ask the user about the entire set of plausible choices.

Finally, CP calibrates the VLM verifier by selecting the $\hat{q}=\frac{\lceil (N+1)(1-\epsilon)\rceil}{N}$ empirical quantile of the nonconformity scores $\{\kappa^1,...,\kappa^N\}$ and uses $\hat{q}$ to construct the prediction set $C(\cpinput)$. This includes all labels that the VLM verifier is at least $1-\hat{q}$ confident about, 
ensuring the $1-\epsilon$ coverage guarantee \cite{vovk2005algorithmic, ren2023robots}.

\para{Deployment Time}
At deployment time, the robot is given a new task description $\task^{\text{test}}$.  
Given the current observation $\obs^{\text{test}}_i$, we obtain $x^\text{test}_i = (\obs^\text{test}_i, \{\{\lang^k\}_{k=1}^K\}_i, \task^\text{test})$ via the same action sampling, prediction, and narration pipeline described above. 
We query the VLM verifier to estimate the probability $\score(y \mid x^\text{test}_i)$ for each choice $y\in \mathcal{Y} = \{1, ..., K, K+1\}$, and the prediction set is generated via \begin{equation}
C(\cpseqinput_i^{\text{test}}) = \{y\in \mathcal{Y} \mid 1-\score(y \mid x_i^{\text{test}}) \leq \hat{q}\}    
\end{equation}
i.e., selecting any index whose probability is larger than the calibrated empirical quantile $1-\hat{q}$. 
Assuming the narrations $\lang$ faithfully represent the outcomes of actions, we formally state the statistical assurance for our calibration procedure.
\begin{theorem}[Verification Coverage Guarantee]
\label{theorem-guarantee}
Let $D_{\text{calib}}= \{ \{(\cpseqinput_i^n, \mathcal{Y}_i^n)\}_{i=0}^M \}_{n=1}^N$ be the calibration dataset with non-conformity score:
\[
\kappa^n := \kappa(\bar{x}^n, \bar{y}^n) = 1- \min_{x_i \in \bar{x}^n}\min_{y \in \mathcal{Y}_i^n}\score(y \mid  \cpseqinput_i^n)  
\]
and let $\hat q$ be the $\frac{\lceil (N+1)(1-\varepsilon)\rceil}{N}$--quantile of $\{\kappa^n\}^N_{n=1}$.  
For a test point $x_i^{\text{test}}$ at any time when the VLM verifier is called $i \in [M]$, where the correct set of labels is $\mathcal{Y}^*_i \subseteq\mathcal{Y}$, define
\[
C(\cpseqinput^{\text{test}}_i)=\{y: 1-\score(y \mid \cpseqinput^\text{test}_i) \le \hat q\}.
\]
Then $\mathbb{P}(\mathcal{Y}^*_i \subseteq C(x_i^{\text{test}})) \ge 1-\varepsilon$.
\end{theorem}
\smallskip \textbf{\textit{Proof.}} See Appendix~\ref{appendix:proof}.
We follow the sequence-level calibration proof from \cite{ren2023robots}, but show that our non-conformity score includes all the ground-truth labels in our prediction set.



\para{Shaping the VLM's Distribution \textit{Before} CP}
\label{sec:bayesian_intent_score}
While in theory conformal prediction provides distribution-free coverage guarantees,
the tightness of the resulting prediction sets depends on the non-conformity score. 
Our score in Eq.~\ref{eq:score-func-ours-multi-phase} directly depends on the VLM's distribution, $\mathbb{P}^\text{VLM}$.
Naively querying a VLM for likelihoods over multiple choices results in significantly overconfident distributions, especially in ambiguous or incapable scenarios.  
Although CP can compensate for this, it leads to overly-large prediction sets that are not informative and result in the robot over-asking for help. 
Thus, the raw softmax probabilities from the logits of the VLM verifier are \textit{not} an ideal way to inform our score function   $\score$ from Eq.~\ref{eq:score-func-ours-multi-phase}.


Our key idea is shape $\mathbb{P}^\text{VLM}$ even \textit{before} calibration by \textit{factorizing} the VLM's reasoning process such that the uncertainty is more explicit.
Our factorization is inspired by Bayesian models of human intent \cite{baker2007goal,mullen2025lbap} which decouple inferring human intent (e.g., unstated aspects of the task) and the likelihood of behaviors given a specific intent:
\begin{align}
    \mathbb{P}^{\text{VLM}}(y \mid \{\lang^k\}^K_{k=1}; \task ) 
    &=\sum_{\theta \in \Theta}  \mathbb{P}(y\mid \{\lang^k
    \}^K_{k=1}, \theta)  \mathbb{P}(\theta \mid  \task ).
    \label{eq:factorized-reasoning}
\end{align}
Given the original task instruction $\mathcal{L}$ (e.g., ``Put the block in my dominant-hand side hole``), we query the VLM three times to estimate each component of Eq.~\ref{eq:factorized-reasoning}. 
First, we ask the model to hypothesize a set of hidden human intents $\theta \in \Theta$ (e.g., ``Insert the block into with the $\Theta=$[left/right] hole'') alongside their respective probabilities, $\mathbb{P}(\theta \mid \mathcal{L})$. 
Then, we ask the VLM to infer the likelihood of each option, \textit{given a hypothesized user intent}: $\mathbb{P}(y\mid \{\lang^k\}^K_{k=1}, \theta)$. 
The final distribution is obtained by marginalizing over the space of user intents.

Intuitively, this factorization leverages the commonsense reasoning abilities of the VLM in a structured way to  
re-calibrate $\mathbb{P}^\text{VLM}$ intrinsically even before non-conformity score computation.  
We find this to be especially useful in incapable scenarios or when one behavior mode is underrepresented in the action samples. 
For example, consider the scenario where all action samples insert the block into the right hole (i.e., $\lang^k$ = ``Robot inserts the block into the back right hole'', $\forall k \in \{1,\hdots,K\}$).
Naively querying $\mathbb{P}^\text{VLM}$ results in an extremely biased distribution in favor of selecting the option with block in the right hole, 
ignoring the possibility that the user is left-handed. 
In our factorized approach, the VLM is asked to reason about hidden aspects of the task or user intents, and thus includes the possibility that the user is right \textit{or} left-handed in $\Theta$. This reshapes the final distribution over options, forcing the VLM to place non-trivial probability mass on the $K+1$th ``none'' option. 
In turn, this results in the reduced overconfidence of the VLM in incapable scenarios and enables the robot to ask for help only when necessary.
In ambiguous scenarios, this same principle helps re-weight the distribution over several equally likely behaviors. 
In Sec.~\ref{sec:experiments}, we show empirically that our factorization results in  $\mathbb{P}^\text{VLM}$ that is substantially more calibrated even before conformal adjustment, and outperforms chain-of-thought reasoning (CoT) \cite{wei2022chain}. 
Ultimately, we obtain non-trivial prediction sets with high empirical coverage while also reducing the help rate.

\subsection{Resolving Uncertainty:  From High-level Clarifications to Low-Level Continual Learning}
\label{sec:method-continual-learning}

At deployment-time, the prediction sets constructed by our CP procedure provide the robot with a principled measure of both semantic task uncertainty and policy incapability, implying separate resolution strategies. 
Importantly, our statistical guarantee from Theorem~\ref{theorem-guarantee} ensures that the VLM verifier can steer the low-level policy to success in $1-\epsilon$ proportion of scenarios by either executing the correct action sample or selecting the correct resolution strategy. We describe each strategy in detail below.

\para{Confident and Capable Policy} 
If the prediction set is a singleton and contains an index corresponding to one of the low-level action plans, i.e. $|C(\cpseqinput_i^{\text{test}})| =1$ and $K+1 \notin C(\cpseqinput_i^{\text{test}})$, 
then the steering system confidently executes this sequence of actions until the next verification step (bottom right, Fig.~\ref{fig:framework}). 
 

\para{Resolving Task Uncertainty with Textual Clarifications}
\label{sec:method-intent-uncertainty}
If the set size $|C(\cpseqinput_i^{\text{test}})| > 1$, the verifier has identified high-level semantic uncertainty in the task and asks a textual clarification question. 
For example, consider the right part of Fig.~\ref{fig:framework}. Both inserting the block in the left hole and inserting into the right hole options are included in the prediction set when the user gives an ambiguous task description $\task=$``\emph{put the block in my dominant-hand side hole}''.
After asking the user to state their preferences about options, the VLM summarizes user's answer $\tilde{\task}$ 
and replaces the original ambiguous task instruction.
This verify and clarify loop repeats and continues until 
$|C(\cpseqinput_i^\text{test})|=1$. 
The right part of Fig.~\ref{fig:framework} demonstrates this process of resolving the task uncertainty.

\para{Resolving Policy Incapability with Continual Learning}
\label{sec:method-policy-incapability}
When the CP prediction set contains only ``\emph{none of the above}'', i.e., $C(x_i^\text{test}) = \{K+1\}$, the VLM verifier detects that none of the action samples align with the user-requested task with high probability. 
This triggers a low-level resolution strategy via interactive imitation learning. Specifically, the robot randomly selects a trajectory and explicitly asks human supervision to correct potential failures, as shown in Fig~\ref{fig:intervention-methods}. 
This approach improves upon standard interactive imitation learning methods \cite{menda2019ensembledagger, kelly2019hg} in two ways. 
First, unlike human-gated  methods \cite{kelly2019hg} that demand constant monitoring of every trajectory, our robot only requests help when the policy is incapable in \textit{all} samples. 
Second, unlike low-level action-uncertainty methods \cite{menda2019ensembledagger},
our approach gates interventions based on semantic task alignment rather than low-level action variance;
this avoids scenarios where low-level policy has noisy actions, but all result in good outcomes (i.e., the policy does not need to be re-trained). 

We employ residual policy learning~\cite{silver2018residual,jiang2025transic}
and train a lightweight model, $\Delta a \sim \pi^{\text{residual}}(\cdot \mid \obs, \action)$, which predicts the difference between human corrections and base policy outputs alongside a gating classifier that determines when this correction is active. 
At re-deployment, we sample $K$ actions from the frozen base policy and mix the residual policy into half of the samples (when triggered by the gating classifier), while 
half of the samples remain unmodified. 
This sampling strategy mitigates catastrophic forgetting by explicitly retaining the base policy's distribution. 
Finally, because the policy's action distribution shifts after residual policy learning, we recalibrate the VLM verifier, thereby closing the learning loop.

\section{Simulation \& Hardware Experiments}
\label{sec:experiments}

We evaluate our framework in both a simulation task and in robotic hardware with a Franka manipulator. 
In all our experiments, the base policy is an image-conditioned diffusion policy\footnote{Although we focus on diffusion models, our framework can be applied to any base policy that represents a stochastic action distribution and from which diverse action samples can be drawn, e.g., a vision-language-action model~\cite{black2024pi_0}.} \cite{chi2024diffusionpolicy} and the world model is Dreamer-v3~\cite{hafner2023mastering}.  In both simulation and hardware \textbf{PnP Cup} task, we call the verifier twice ($M = 2$) throughout the task and use $M=3$ for  \textbf{Insert Block} task. We perform calibration at this sequence-level.

\para{Simulation Setup}
We use the Robomimic \cite{robomimic2021} benchmark, specifically the square nut-on-peg task.
We chose this task because it requires high-precision and the human Robomimic demonstrations exhibit inherent diversity in \textit{how} the nut is placed on the peg (e.g., handle facing left or right).
We train our low-level diffusion policy with 120 demonstrations: 60 with the handle facing left and 60 facing right. We initially train the world model on 600 trajectories: the 120 demonstrations plus 480 rollouts from our base policy. We further fine-tune the model with 100 additional rollouts from the interleaved policy and world model prediction scheme.


\para{Real World Setup}
We use a Franka Emika robotic manipulator equipped with a Robotiq gripper 
and implement two real-world manipulation tasks: (1) \textbf{PnP Cup} has the robot picking a green cup and placing it into one of two bins (labeled ``clean'' and ``dirty''), and (2) \textbf{Insert Block} has the robot inserting a purple table leg (block with circular peg) into the top left or top right holes of a tabletop.
Images are perceived from a wrist camera and an external camera (see Appendix~\ref{appendix:real-robot}). The second task is more challenging because it requires high insertion precision.
Task uncertainty comes from the user's instructions as to where to place the cup or insert the block. 
For both tasks, the base policy is a diffusion policy that controls the cartesian pose and gripper. 
We train it with 100 demonstrations, covering two behavioral modes: placing the cup in the left/right bin or inserting the block into the left/right hole.
The world model is trained on 350 trajectories: the 100 expert demonstrations plus 250 policy rollouts.

\para{VLM Narration \& Verification} We use the same set of VLM models in simulation and hardware. 
For narrating the imaginations, we use Gemini-3-flash-preview due to its strong video summarization capabilities. 
For verification, we use Gemini-2.0-flash because it exposes the token logits via the API for ablations of conformal prediction in Appendix~\ref{appendix:calibration}.

\subsection{Uncertainty-Aware Steering}

\para{Calibration Dataset}
We collect 80 pairs of initial observation $\obs_0$ and language instruction $\task$. Our tasks have multiple phases $M=2$ or $3$ and we obtain new observation $o_T$ after the each action sequence is executed, gathering data of all phases as our calibration set. 40 of these instructions are ambiguous, such as \textit{``move the cup to a bin''}, and 40 are straightforward, such as \textit{``can you move it to the dirty bin?''} for \textbf{PnP Cup} task. Among straightforward instructions, sometimes the policy does not generate any target actions, leading the scenario to be labeled incapable. From the observation, we obtain the behavior narrations for the imagined rollouts of $K=10$ action samples. 
After narration, we group semantically identical narrations and randomly select a narration from each group to appear as a multiple choice option. We choose $1-\epsilon = 0.85$ for calibration. Details of instructions are in Appendix~\ref{appendix:real-robot}.

\para{Evaluation Dataset}\label{sec:eval_dataset}Similar to our calibration dataset, we maintain a set of $40$ samples in the test set sampled from the same distribution as the calibration dataset. 
During evaluation, we predict the set for each sample and compare that to the ground-truth set. 
If the ground-truth option(s) lie in the prediction set, we have achieved coverage. 
If the set size is larger than 1, we need clarification.

\para{Metrics} We define three metrics commonly used in prior CP work to evaluate our calibrated VLM verifier. \textbf{Coverage} is the proportion of test samples whose ground-truth option(s) are included in the prediction set. The \textbf{Clarification rate} is the proportion of times that the prediction set is larger than 1 and thus the robot asks the human a question. The \textbf{Set size} is the size of the prediction set. 
For coverage, we want to achieve the desired coverage rate $1-\epsilon$ in all three scenarios. 
For the clarification rate, we want a high clarification rate in ambiguous cases and a low one in the other two. For the prediction set, we want the set size to be 2 for ambiguous cases and 1 for other cases.
\begin{figure}[h!]
    \centering
    \includegraphics[width=0.9\linewidth]{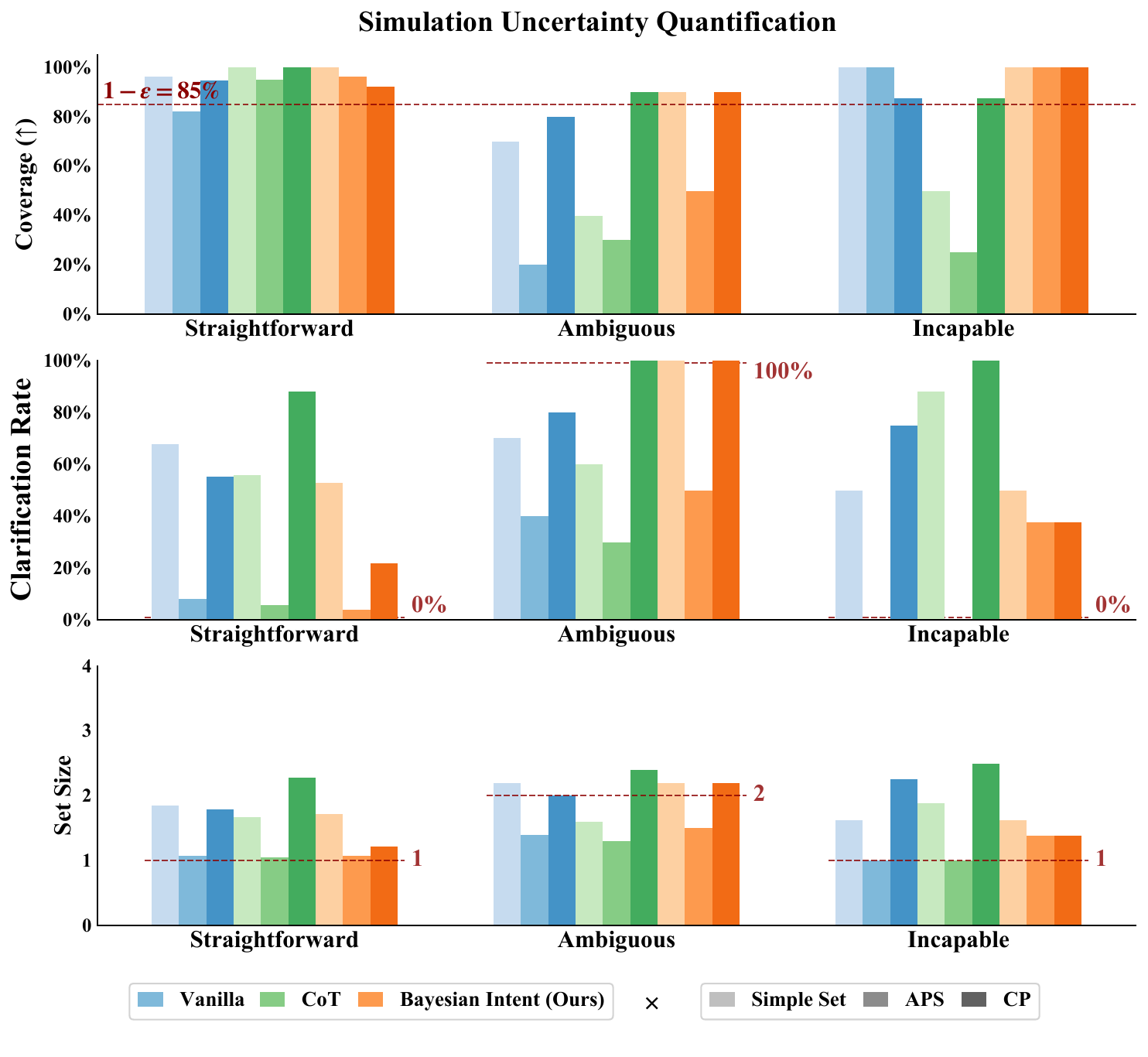}
    \caption{\textbf{Uncertainty Quantification Results: Simulation}. We compare the combination of Vanilla, CoT and Bayesian Intent (Ours) models for UQ. Dashed lines indicate the target coverage ($1-\epsilon = 0.85$), clarification rate, and set size. } 
    \label{fig:results-UQ-simulation}
    \vspace{-1.3em}
\end{figure}
\smallskip  
\begin{figure}[h!]
    \centering
    \vspace{-0.3em}
\includegraphics[width=0.9\linewidth]{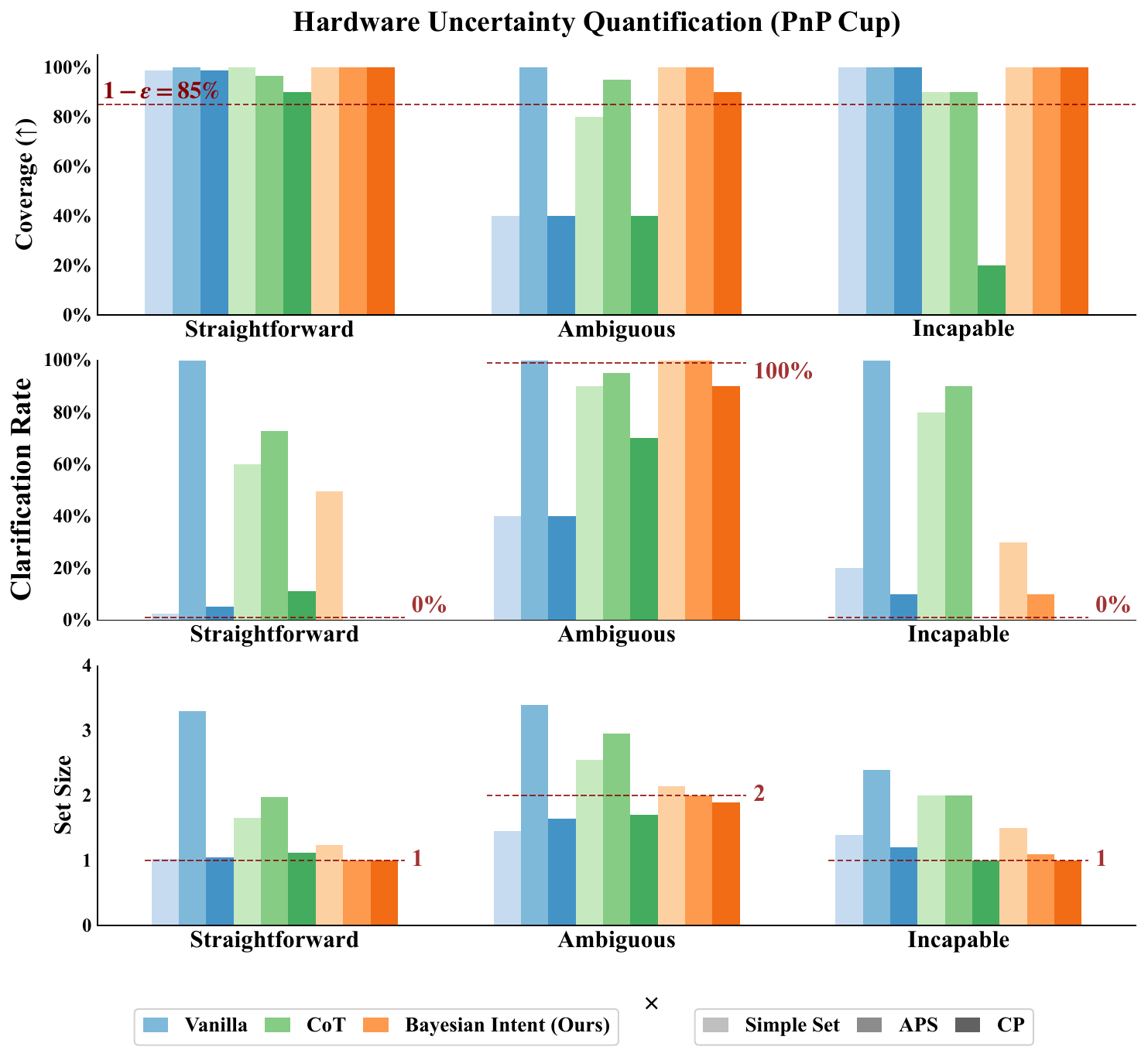}
    \caption{\textbf{Uncertainty Quantification Results: Hardware (PnP Cup)}. Combination of Vanilla, CoT and Bayesian Intent (Ours) models for UQ. Dashed lines indicate the target coverage rate ($1-\epsilon = 0.85$), clarification rate, and set size.}
    \vspace{-1.2em}
    \label{fig:results-UQ-hardware}
\end{figure}

\begin{figure}[h!]
    \centering
    \includegraphics[width=0.9\linewidth]{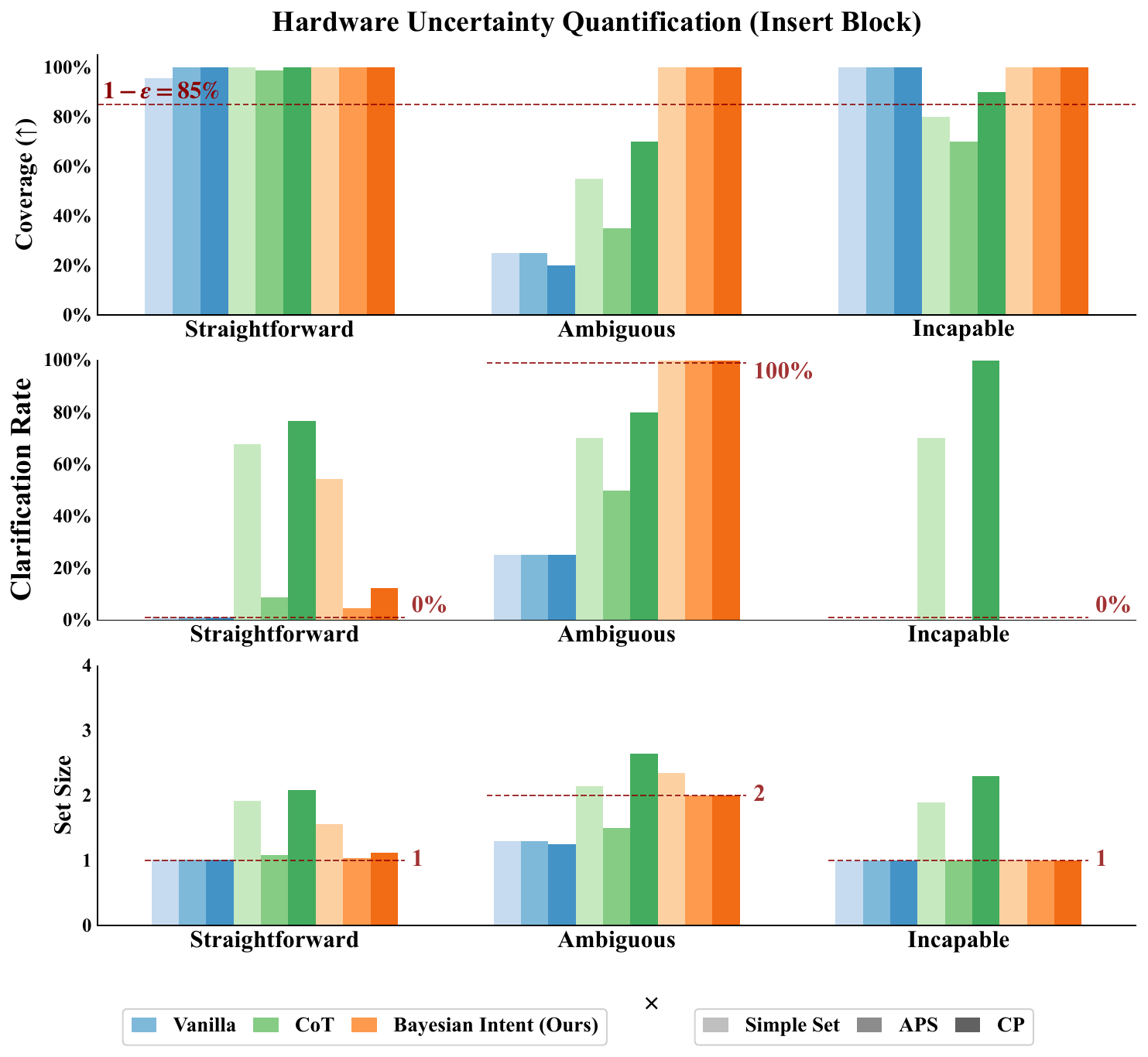}
    \caption{\textbf{Uncertainty Quantification Results: Hardware (Insert Block)}. We compare the combination of Vanilla, CoT and Bayesian Intent (Ours) models for UQ. Dashed lines indicate the target coverage, clarification rate, and set size. } 
    \label{fig:results-UQ-block}
    \vspace{-1.3em}
\end{figure}

\subsubsection{Our Score Function Balances Coverage \& Clarification}\quad
\label{sec:experiment-coverage}

\para{Methods} We evaluate  three different UQ methods: Simple Set, APS~\cite{romano2020classification}, CP~\cite{ren2023robots}, and three different score functions including Vanilla, CoT and Bayesian Intent (Ours). 
This results in 3x3 methods we compare.
\textbf{CP}: we adopt the same conformal prediction scheme  from ~\cite{ren2023robots} but modify the score function to be Eq~\ref{eq:score-func-ours-multi-phase}.
\textbf{Simple Set}: sorts the options from high scores to low and adds the options until the sum exceeds $1-\epsilon$. \textbf{APS} sorts the options from high to low scores and adds the options until the sum exceeds $1-\hat{q}$.
\textbf{Vanilla} uses the distribution $\mathbb{P}^\text{VLM}$ by asking the model to self-generate the probability scores between 0 and 1. 
\textbf{CoT} shapes $\mathbb{P}^\text{VLM}$ by first prompting the VLM to reason step-by-step and then, conditioned on the reasoning, generates the self-reported probability. \textbf{Bayesian Intent} computes intent-conditioned scores $\mathbb{P}^\text{VLM}$ as in Eq.~\ref{eq:factorized-reasoning}.

\para{Evaluation} We collect the evaluation dataset in Sec.~\ref{sec:eval_dataset}. When the policy is incapable under straightforward instruction, it is counted as an incapable scenario. Details of the dataset composition are in Appendix~\ref{appendix:calibration}.

\para{Results} Across simulation and hardware, Fig.~\ref{fig:results-UQ-simulation}, Fig.~\ref{fig:results-UQ-hardware} and Fig.~\ref{fig:results-UQ-block} show that our proposed Bayesian intent score function with CP can achieve strong coverage ($\geq 85\%$) while minimizing the clarification rate compared to Vanilla and CoT. Our proposed Bayesian intent score function achieves the best balance between coverage and clarification rate across all three tasks. 
In contrast, CoT-based methods perform poorly due to conservative reasoning; they frequently overestimate ambiguity (leading to near $100\%$ clarification rates) or exhibit sharp drops in coverage when set sizes are reduced (APS CoT). 
Similarly, Vanilla methods struggle with poor calibration, resulting in either overconfidence or severe under-coverage. 
This indicates that shaping the VLM's uncertainty via factorization  leads to better calibration. 
With our intent-conditioned score function, CP strikes the best balance: it avoids Simple Set's overconservativeness in straightforward cases and APS's undercoverage in ambiguous ones.

\smallskip 
\subsubsection{Uncertainty Improves Policy Steering Performance}
\label{sec:experiment-uq-steering}
Next, we study if our well-calibrated VLM verifier (CP + Bayesian Intent) results in closed-loop performance improvements within our policy steering framework.

\para{Methods} We compare three methods in this section including \textbf{Base Policy}, the low-level diffusion policy, \textbf{Forewarn}~\cite{wu2025foresight}, an \textit{uncalibrated} variant of VLM-in-the-loop policy steering, and \textbf{UPS w/ Clarification} which
uses the calibrated threshold $\hat{q}$ to select the prediction set. When set size $> 1$, it communicates with user through Q \& A and updates its task instructions accordingly to verify again as described in Sec.~\ref{sec:method-intent-uncertainty}.

\para{Evaluation} We take the same dataset of 40 samples as  in Sec.\ref{sec:experiment-coverage}.  
We compute a confusion matrix where true positive (TP) indicates the selected action sample leads to the correct outcome; true negative (TN) indicates the system properly elected to ask for help, false positive (FP) indicates the selected action actually fails; false negative (FN) indicates the system chose to ask for unnecessary help. 
We report the success rate as the number of true positives divided by total number of samples. The detailed analysis is in Appendix~\ref{appendix:real-robot} and ~\ref{appendix:sim}.

\para{Results} Fig.~\ref{fig:continual-learn} shows that although Forewarn improves the success rate in both straightforward scenarios and ambiguous scenarios by steering the base policy to better align with the task instructions, it achieves much smaller gain in ambiguous cases in both tasks (e.g., in the \textbf{PnP Cup} task improvement in straightforward = $45\%$, ambiguous = $15\%$) because the VLM verifier is overconfident. 
In contrast, UPS + w/ Clarification addresses uncertainty in ambiguous cases, further increasing success rate by $15\%$. 
We note that with UQ and clarification, we can already achieve the desired coverage rate of $85\%$ in straightforward cases for the \textbf{PnP Cup} task.

\begin{figure}[h!]
    \centering
    \includegraphics[width=0.9\linewidth]{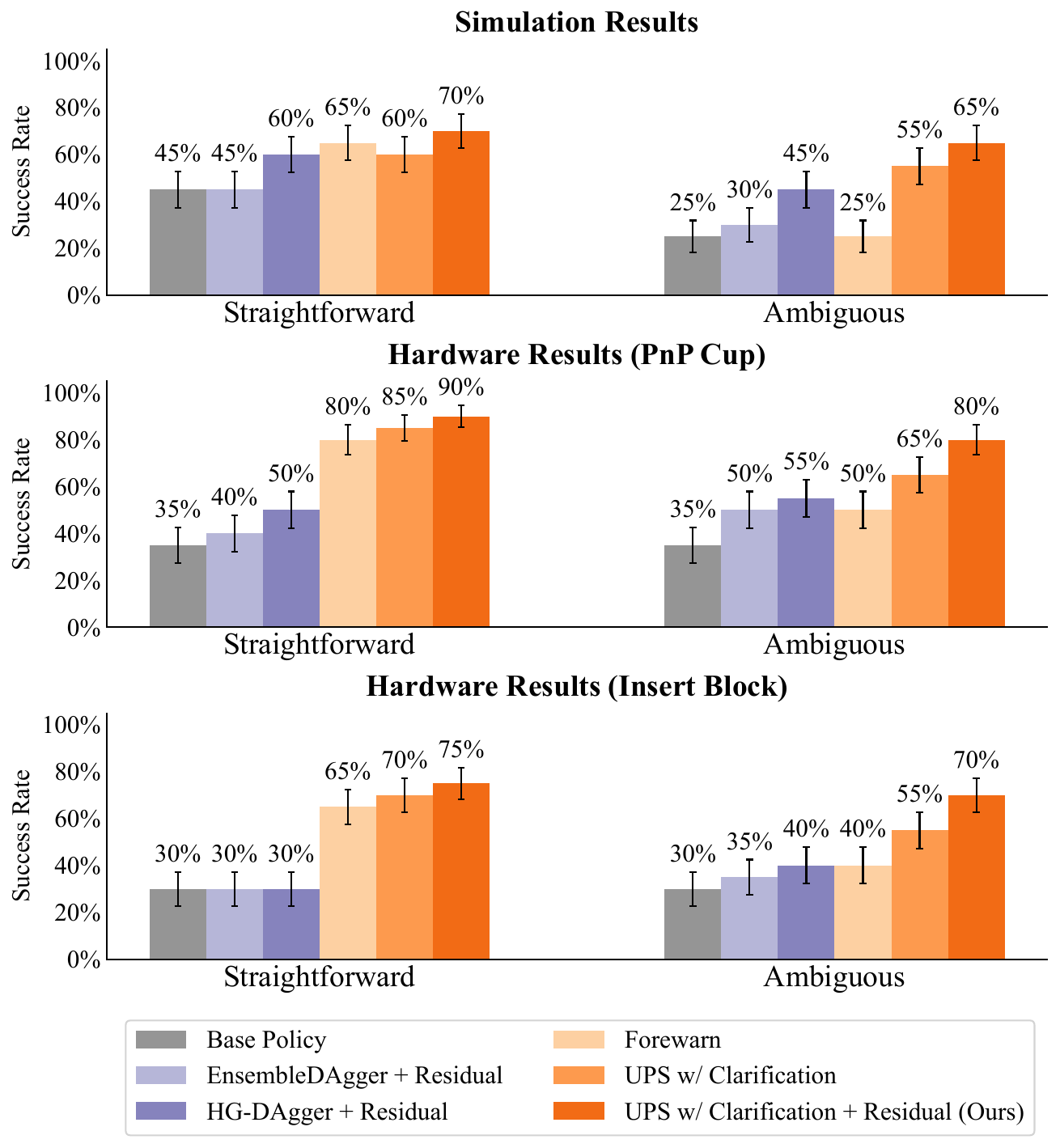}
    \caption{\textbf{Success Rates Pre- and Post-Continual Learning: Hardware and Simulation.} We deploy the robot with 20 straightforward (left) and 20 ambiguous (right) instructions. We average the success rate over 20 trials for each scenario. UPS solicits data in a way which maximizes the final success rate after residual policy training, compared to baselines. }
    \label{fig:continual-learn}
    \vspace{-1.2em}
\end{figure}

\begin{figure}[h!]
    \centering
    \includegraphics[width=0.9\linewidth]{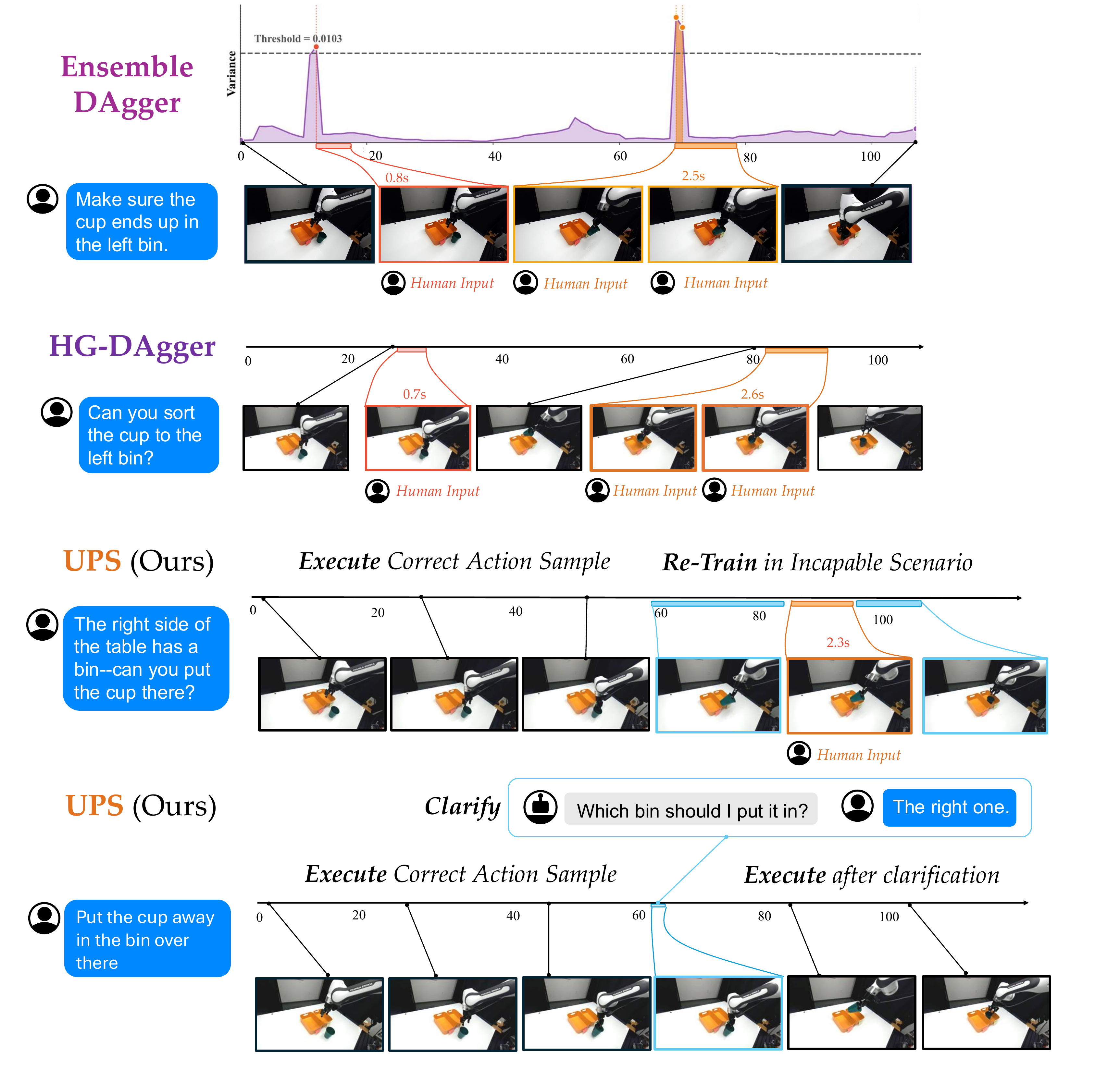}
    \caption{\textbf{Hardware: Robot Asking for Interventions.} 
    EnsembleDAgger (top): the human intervenes whenever the ensemble disagreement exceeds a threshold. 
    Human-Gated (HG) DAgger (middle): a human monitors and intervenes when the robot's behavior deviates from their intention. 
    Our approach (bottom two) asks for ``cheap'' clarifications and only requests interventions when no action sample achieves the task. 
    }
    \label{fig:intervention-methods}
    \vspace{-1.2em}
\end{figure}

\subsection{Continual Learning}

Finally, we close the loop and evaluate how our uncertainty-aware steering method informs low-level data collection for improving the base policy, and how this influences human intervention rate and re-deployment performance.

\smallskip 
\subsubsection{Semantic-level UQ Minimizes Human Feedback}\quad
\label{sec:experiment-mini-feedback}

\para{Evaluation} Similar to  Sec.~\ref{sec:experiment-uq-steering}, we evaluate the policy for 40 trials where each trial has different initial condition and different language instructions. During intervention, the human corrects the robot by controlling the end-effector and gripper through a teleoperation device (e.g., SpaceMouse). We use the same success rate metric as in  Sec.~\ref{sec:experiment-uq-steering}. 
\begin{figure}[h!]
    \centering
\includegraphics[width=0.9\linewidth]{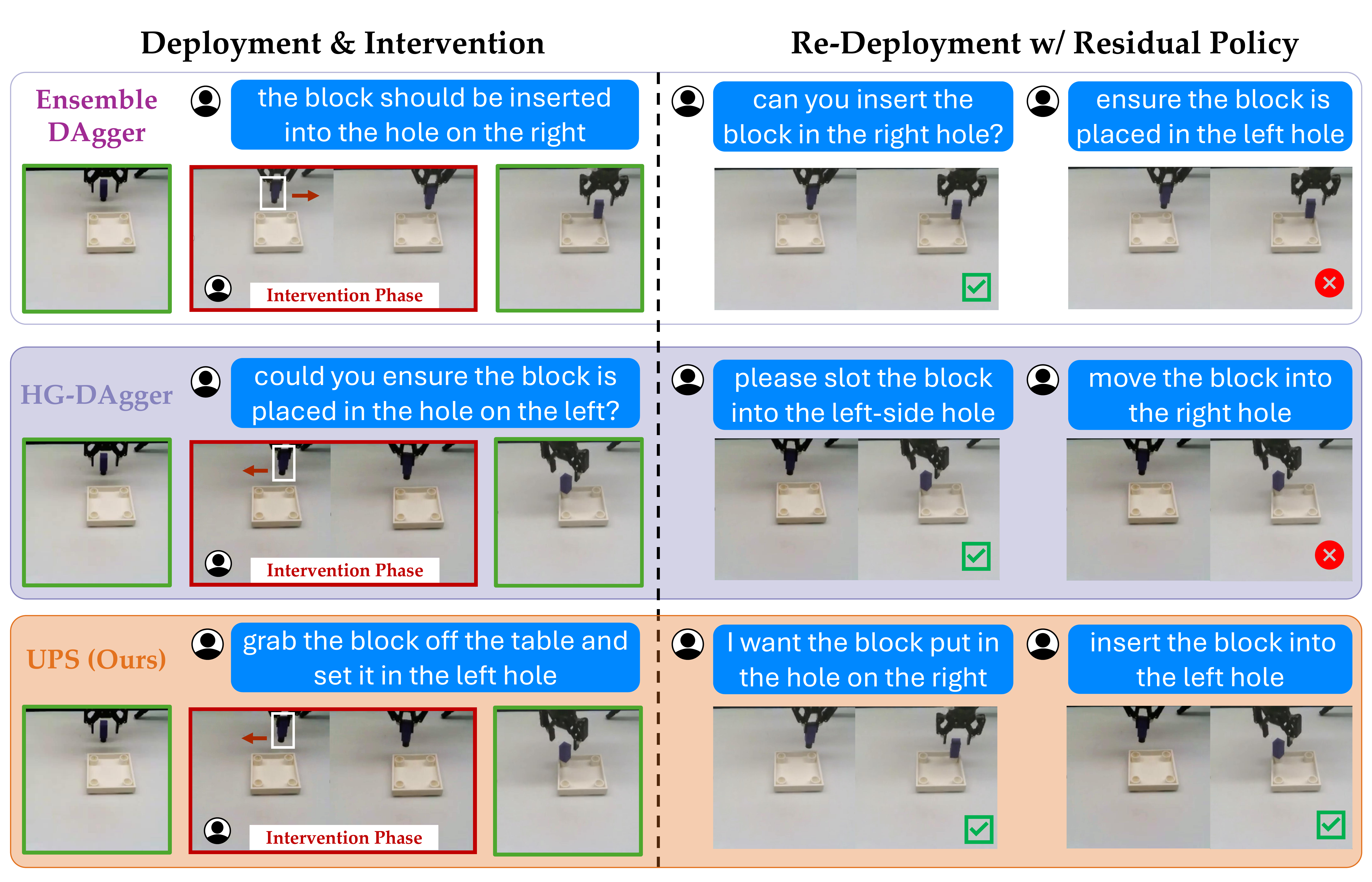}
    \caption{\textbf{  Intervention \& Re-deployment for Insert Block Task}. (left) Different strategies to elicit human interventions: HG-DAgger (top), EnsembleDAgger (middle) and UPS (bottom). (right) UPS maintains the multi-modality, while other methods exhibit failures or mode collapse.}
    \vspace{-0.7cm}
    \label{fig:intervention_method_block}
\end{figure}

\para{Methods} We compare our proposed high-level semantic guided human interventions \textbf{UPS + w/ Clarification + Residual} against two foundational approaches, shown in Fig.~\ref{fig:intervention-methods}: \textbf{HG-DAgger}~\cite{kelly2019hg}, where a human intervenes whenever they see fit, and \textbf{EnsembleDAgger}~\cite{menda2019ensembledagger}, which asks for human interventions when the policy ensemble disagreement is high.

\para{Metrics} For each trajectory, we divide the number of human intervention steps by trajectory length, and then we average this normalized rate across all trajectories to measure overall human effort. 
When it asks a clarification question \textit{without} low-level interventions, we count this as one human intervention step since Q\&A is easier than physical intervention. 
        
\para{Results} \textbf{UPS} has the lowest human intervention rate because it asks for low-level intervention data only when the policy cannot generate a  low-level behavior that is ``semantically aligned'' with the task. 
Our method achieves significantly lower intervention rates than the baselines in both settings. For \textbf{PnP Cup} task, UPS requires interventions at a rate of 0.06, compared to 0.1 for HG-DAgger and 0.16 for EnsembleDAgger, whose low-level uncertainty quantification demands even more human assistance. Similarly, UPS' intervention rate is 0.05, compared to HG-DAgger's rate of 0.06 and EnsembleDAgger's rate of 0.26 in \textbf{Insert Block}. In simulation, the gap widens: UPS's rate is 0.058, while HG-DAgger has an intervention rate of 0.20 and EnsembleDAgger has 0.27.


\smallskip

\subsubsection{UPS Improves Re-Deployment Performance}\quad
\label{sec:experiments-ups-redeployment}

\para{Residual Policy Training} 
For fair comparison among different intervention methods, we fix an intervention budget with the same number of intervention trajectories for all methods. 
We compute the delta action between the intervention and the base action and train the residual policy $\pi^\text{residual}$ with observation and base action as inputs, as shown in Appendix~\ref{appendix:residual}. 

\para{Methods} Similar to Sec.~\ref{sec:experiment-mini-feedback}, we learn a residual policy from interventions collected with different methods: \textbf{EnsembleDAgger + Residual}, \textbf{HG-DAgger + Residual} and \textbf{UPS + w/ Clarification + Residual}. We also compare to base policy without the residual with the metric described in  Sec.~\ref{sec:experiment-uq-steering}. 

\para{Results} Fig.~\ref{fig:continual-learn} shows that our residual policy continues to improve from \textbf{UPS w/ Clarification} only because it reduces the proportion of incapable scenarios; it obtains a $15\%$ increase in success rate in ambiguous settings for both hardware tasks.  
Across both types of scenarios, our final success rate is $85\%$ for \textbf{PnP Cup} task. 
This empirical result  corroborates  Theorem~\ref{theorem-guarantee}, demonstrating that the 
verifier can steer the low-level policy to success in $1-\epsilon$ of scenarios by either executing the right sample or selecting the appropriate strategy (asking clarification questions when ambiguous or eliciting human interventions when incapable). In simulation and \textbf{Insert Block} task, we achieve slightly lower coverage due to the challenges of simulating precise contact-rich manipulation tasks with the current world model.
On the other hand, both baseline DAgger methods struggle to learn new behaviors from intervention data while maintaining the original capabilities, as shown in Fig.~\ref{fig:intervention_method_block}. 
We hypothesize
the learned residual policy may bias the final policy distribution towards certain behavior modes of doing the task, or may push the model towards new incapable modes. 
\section{Conclusion \& Limitations}

In this work, we propose an uncertainty-aware policy steering system which rigorously calibrates a verifier's uncertainty over low-level action samples. Our approach also resolves uncertainty with high-level clarifications or low-level re-training to achieve overall success rate in $1-\epsilon$ of scenarios. 
Since we focus on uncertainty quantification of VLM verifiers, 
we assumed the imaginations from world model always match future outcomes and behavior narrations are correct and complete. 
Interesting future work incorporates the uncertainty of the world model (and narrations)  \cite{mei2025world}, 
for more robust system.

\section{acknowledgements}
The authors were partially supported by the National Science Foundation (NSF) award $[\#2246447]$, NSF CAREER award $[\#2441014]$, and Samsung Research through the LEAP-U program. The views expressed are those of the authors and do not necessarily reflect those of NSF or Samsung. We thank Junwon Seo and Michelle Zhao for helpful discussions. 
\newpage

\bibliographystyle{plainnat}
\bibliography{references}
\clearpage
\appendix
\section{Appendix}

\subsection{Proof of Coverage Guarantees}
\label{appendix:proof}

To maintain formal coverage guarantees despite temporal dependencies, we follow the same proof style as in~\cite{ren2023robots} and perform sequence-level calibration using the minimum score over phases.
    
\begin{proof}
Let the calibration dataset be \textit{i.i.d.} sequences of $(\bar{x}^n, \bar{y}^n)$,
\(
(\bar x^{\,n},\bar y^{\,n}) \stackrel{\text{i.i.d.}}{\sim} \mathcal D
\):
\[
D_{\text{calib}}=\{\{(x_i^n, \mathcal{Y}_i^n)\}_{i=0}^M\}_{n=1}^N
=\{(\bar x^{\,n},\bar y^{\,n})\}_{n=1}^N,
\]
where \(\bar x^{\,n}\) is a sequence of query points and \(\bar y^{\,n}\) assigns to each
\(x_i^n\in\bar x^{\,n}\) a set of correct labels \(\mathcal{Y}_i^n\subseteq\mathcal Y\).
For any sequence \((\bar x,\bar y)\), define the per-point ``multi-option'' nonconformity score
\[
s(x_i,\mathcal{Y}_i) \;:=\; 1-\min_{y\in \mathcal{Y}_i} \mathbb{P}^\text{VLM}(y\mid x_i),
\label{def:per_point_score}
\tag{1}\]
and define the sequence-level nonconformity score (equivalently to the theorem's definition)
\[
\kappa(\bar x,\bar y)
\;:=\;
1-\min_{x_i\in \bar x}\ \min_{y\in \mathcal{Y}_i} \mathbb{P}^\text{VLM}(y\mid x_i)
\;=\;
\max_{x_i\in \bar x}\, s(x_i,\mathcal{Y}_i).
\tag{2}
\label{def:score}
\]
Let \(\kappa^n:=\kappa(\bar x^{\,n},\bar y^{\,n})\) be the $n$-th calibration score for
\(n\in[N]\), and let \(\hat q\) be the split-conformal quantile of
\(\{\kappa^n\}_{n=1}^N\) at level \(1-\varepsilon\), i.e.\ the
\(\lceil (N+1)(1-\varepsilon)\rceil\)-th order statistic (standard split conformal rule).

\paragraph{Step 1: The standard conformal prediction guarantee for $i.i.d.$ samples}\quad

Draw an independent test sequence
\[
(\bar x^{\,\text{test}},\bar y^{\,*})\sim \mathcal D,
\qquad
\bar x^{\,\text{test}}=\{x_i^{\text{test}}\}_{i=0}^M,\ 
\bar y^{\,*}=\{\mathcal Y_i^*\}_{i=0}^M,
\]
and define its sequence-level score \(\kappa^{\text{test}}:=\kappa(\bar x^{\,\text{test}},\bar y^{\,*})\).
Since \((\bar x^{\,1},\bar y^{\,1}),\dots,(\bar x^{\,N},\bar y^{\,N}),(\bar x^{\,\text{test}},\bar y^{\,*})\)
are \textit{i.i.d.} and \(\kappa(\cdot,\cdot)\) is a deterministic function of a sequence,
the $(N+1)$ random variables \((\kappa^1,\dots,\kappa^N,\kappa^{\text{test}})\) are \textit{i.i.d.}.
Hence, by the standard split conformal quantile guarantee,
\[
\mathbb{P}\!\left(\kappa^{\text{test}}\le \hat q\right)\ \ge\ 1-\varepsilon.
\tag{3}
\label{eq:sequence-prob}
\]

\paragraph{Step 2: Sequence-level calibration guarantees the per-point coverage as shown in Claim 1 in ~\cite{ren2023robots}}\quad 

Define the conformal prediction set for any query point \(x\) in any sequence by
\[
C(x)\;:=\;\{y\in\mathcal Y:\ 1-p^{\text{VLM}}(y\mid x)\le \hat q\}.
\tag{4}
\label{def:prediction_set}
\]
From Definition ~\ref{def:score}, we have $\kappa^{\text{test}}=\max_{x_i^{\text{test}}\in \bar x^{\,\text{test}}}s(x_i^{\text{test}},\mathcal Y_i^*)$
Since from Eq.~\ref{eq:sequence-prob}, we have \[
\mathbb{P}\!\left(\kappa^{\text{test}}\le \hat q\right)\ \ge\ 1-\varepsilon.
\]
Then \[
\mathbb{P}\!\left(\max_{x_i^{\text{test}}\in \bar x^{\,\text{test}}}s(x_i^{\text{test}},\mathcal Y_i^*)\le \hat q\right)\ \ge\ 1-\varepsilon.
\]
Given that 
\[\forall x^{\text{test}} \in \bar{x}^{\text{test}}, s(x^{\text{test}},\mathcal Y_i^*) \leq \max_{x_i^{\text{test}}\in \bar x^{\,\text{test}}}s(x_i^{\text{test}},\mathcal Y_i^*)\]
We have
\[\forall x^{\text{test}} \in \bar{x}^{\text{test}}, \mathbb{P}\!\left(s(x_i^{\text{test}},\mathcal Y_i^*)\le \hat q\right)\ \ge\ 1-\varepsilon \tag{5}\label{eq:per-point-guarantee} \]
This indicates if the test sequence $\bar{x}^\text{test}$ satisfies the guarantee, then any query point $x^\text{test}$ in the test sequence also satisfies the guarantee.


\paragraph{Step 3: Minimum score over true option sets guarantees all true options are included}\quad

From Definition~\ref{def:per_point_score} and Eq.~\ref{eq:per-point-guarantee}, we know $\forall x^\text{test} \in \bar{x}^\text{test}$, 
\[\mathbb{P}\!\left(1- \min_{y \in \mathcal{Y}^*} \score (y \mid x_i^{\text{test}})\le \hat q\right)\ \ge\ 1-\varepsilon \]
Given that
\[\forall y' \in \mathcal{Y}_i^*, 1- \score(y' \mid x_i^\text{test})\leq 1-\min_{y\in \mathcal{Y}_i^*}\score(y \mid x_i^\text{test})\]
We have $\forall y' \in \mathcal{Y}_i^*$,
\[\mathbb{P}\!\left(1- \score(y' \mid x_i^\text{test})\leq 1-\min_{y\in \mathcal{Y}_i^*}\score(y \mid x_i^\text{test}) \leq \hat{q}\right)\geq 1-\varepsilon\]

Therefore, combined with Definition~\ref{def:prediction_set}, we get 
\[
\mathbb{P}\!\big(\mathcal Y_i^*\subseteq C(x_i^{\text{test}})\big)\ \ge\ 1-\varepsilon,
\] which proves the claim.
\end{proof}
We have proved that \( \mathbb{P}\!\big(\mathcal{Y}_i^* \subseteq C(x_i^{\text{test}})\big) \ge 1-\epsilon \), where the probability is taken over the random draw from the calibration set \(D_{\text{calib}}\) and the test pair \(\big(x_i^{\text{test}}, \mathcal{Y}^*_i\big)\); that is, the result is a marginal coverage guarantee. Consequently, achieving the target coverage for each new test point would, in principle, require a freshly sampled calibration set. However, following~\cite{ren2023robots}, we use the dataset-conditional guarantee
--- conditioning on the realized calibration dataset --- which can be applied to new test samples without re-calibration.

\subsection{Aside on our score function design}
Prior work \cite{ren2023robots}
uses the \textit{max} score to get an existence guarantee in Eq~\ref{eq:score-func-knowno-multi-phase} and is designed to ensure the prediction set contains \emph{at least one} acceptable choice and then commits to a single selected option.
We aim to include \emph{all} suitable options in the prediction set by using a more conservative score that minimizes over options to achieve a completeness guarantee (Eq.~\ref{eq:score-func-ours-multi-phase}). From a user perspective, this avoids guessing the user’s unstated intent and instead displays the set of plausible choices for quick disambiguation.

\begin{equation}
\{\Tilde{\kappa}_i^n = 1- \min_{1\leq i\leq h}\eqnmarkbox[gray]{labelset}{\max_{y \in \mathcal{Y}_i^n}}\score(y \mid  \cpseqinput_i^n)\}_{n=1}^N    
\label{eq:score-func-knowno-multi-phase}
\end{equation}
\annotate[ xshift=0.5em, yshift=0.2em]{below,left}{labelset}{KnowNo maximizes over \\multiple ground-truth options }

\vspace{0.2em}
\begin{equation}
\{\Tilde{\kappa}_i^n = 1- \min_{1\leq i\leq h}\eqnmarkbox[solid_green]{labelset}{\min_{y \in \mathcal{Y}_i^n}}\score(y \mid  \cpseqinput_i^n)\}_{n=1}^N    
\label{eq:score-func-ours-multi-phase-color}
\end{equation}
\annotate[xshift=0.4em]{below,right}{labelset}{Ours minimizes over
\\multiple ground-truth options}
\vspace{0.4em}
\subsection{Real Robot (Hardware)}
\para{Hardware setup: PnP Cup}
\label{appendix:real-robot}
\begin{figure*}
    \centering
    \includegraphics[width=\linewidth]{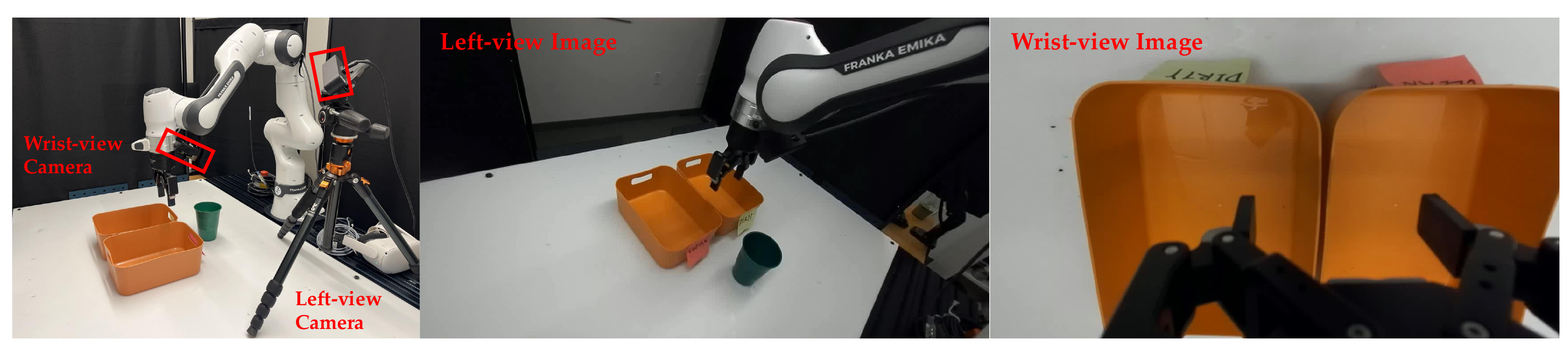}
    \caption{\textbf{Hardware Setup: PnP Cup.} We demonstrate our hardware environment setup with a Franka Emika Panda arm and two Zed cameras (left image). In the middle, we show the left-view image captured by a Zed 2i camera and on the right, we show the wrist-view image captured by a Zed Mini camera. }
    \vspace{-0.3cm}
    \label{fig:hardware_setup}
\end{figure*}
Fig.~\ref{fig:hardware_setup} shows the setup for our real-world task of placing a cup into a bin. On the table, a green cup is positioned in front of two orange bins: the left bin is marked ``clean'' with a pink tag, and the right bin is marked ``dirty'' with a yellow tag. Both the left-view and wrist-view images have dimensions of $1920\times1080$. The height is zero-padded to match the width, and the resulting square images are resized to $256\times256$ for input to the policy and world model. To achieve the goal of placing the cup in a bin, the robot can complete the task in two ways by selecting either the left or right bin, as shown in Fig.~\ref{fig:hardware_rollout}.
\begin{figure}
    \centering
    \includegraphics[width=\linewidth]{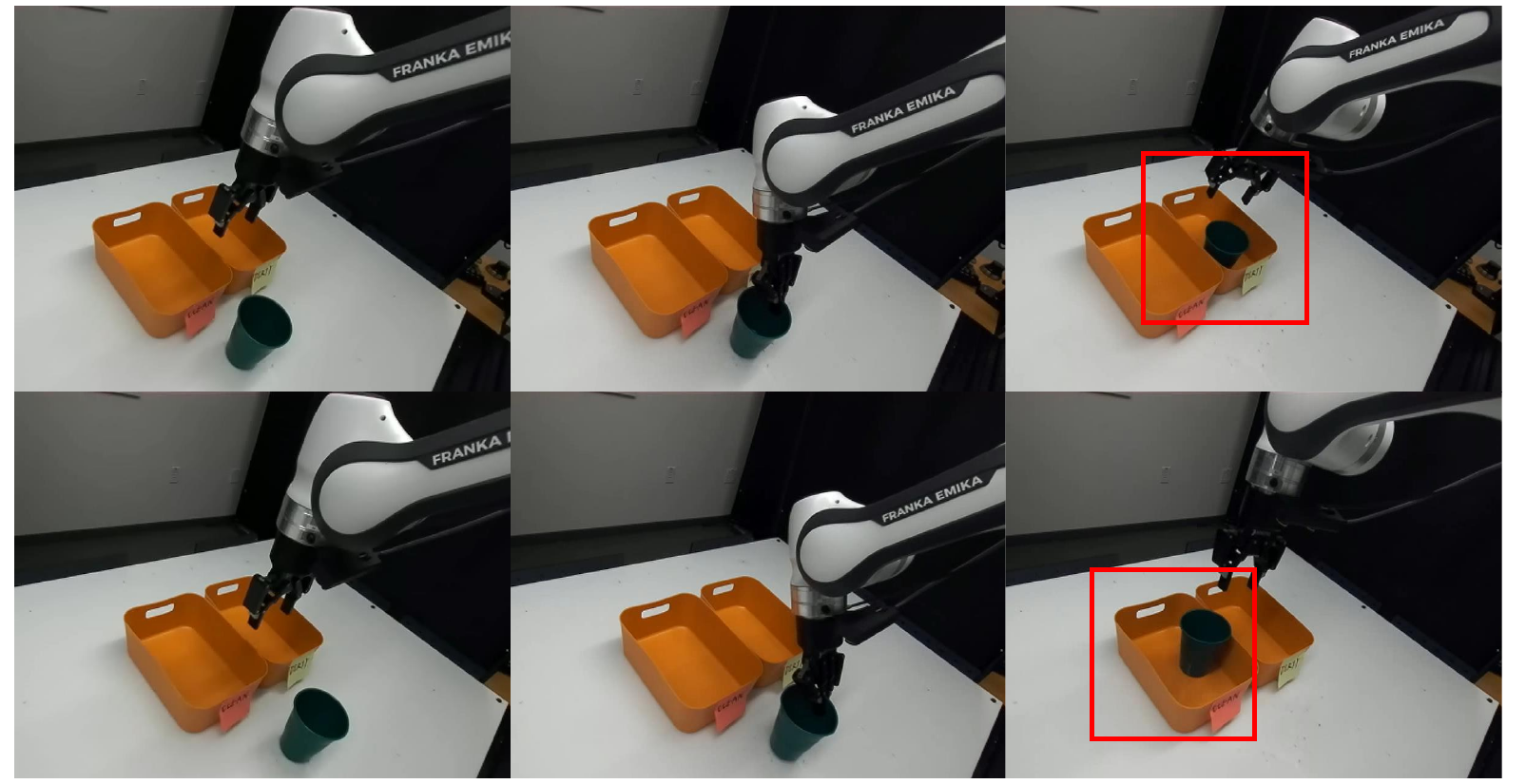}
    \caption{\textbf{Hardware Examples of the PnP Cup Task.} We show two different ways of achieving the task of placing the cup in the bin. The top row shows the cup placed in the left bin while the bottom row shows the cup placed in the right bin.}
    \vspace{-0.3cm}
    \label{fig:hardware_rollout}
\end{figure}

\para{Hardware setup: Insert Block}
The second real-world task requires greater precision: a purple block with a cylindrical base sits on a table in front of a white base with four round holes; the goal is to place the block into the user-specified hole. A RealSense camera captures a front view ($640\times480$) and a Zed camera captures a wrist view ($1280\times720$). For each, the width is cropped from the center to match the height, yielding a square image that is then resized to $256\times256$ for input to the diffusion policy and world model. As shown in Fig.~\ref{fig:hardware_rollout}, the robot places the block in either the top-left or top-right hole.
\begin{figure*}
    \centering
    \includegraphics[width=\linewidth]{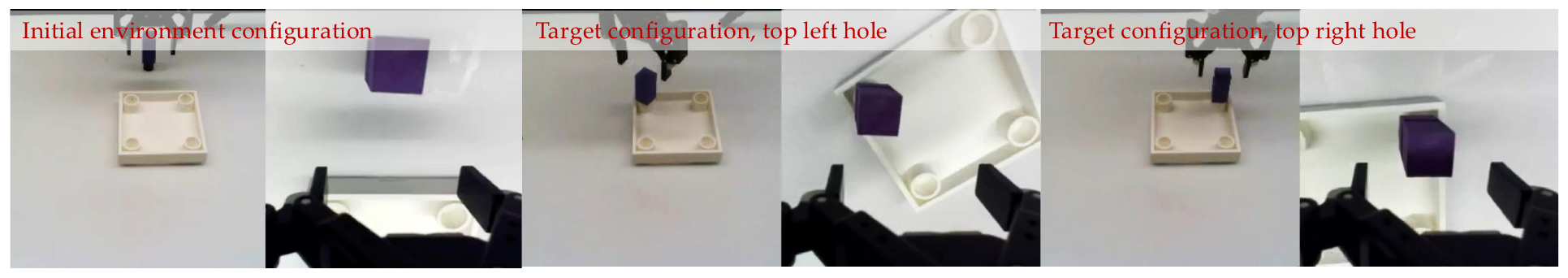}
    \caption{\textbf{Hardware Setup: Block.} We demonstrate our hardware setup for the block task using a Franka Emika Panda arm. The leftmost two images show the initial environment setup, while the right four images show the target configurations for the two behavior modes, as labeled. In each image pair, the left image provides the third-person camera view and the right image provides the wrist view.}
    \vspace{-0.3cm}
    \label{fig:hardware_rollout}
\end{figure*}

\para{Task Instructions} We divide the different instructions a user might give for the robot to complete the task into two categories: 1) \emph{straightforward} scenarios; 2) \emph{ambiguous} scenarios. In Table~\ref{tab:straightforward_hardware_instruction} and Table~\ref{tab:ambiguous_hardware_instruction}, we list all task instructions used to construct the calibration dataset for \textbf{PnP Cup} task. In Table~\ref{tab:ambiguous_block_instruction}, we list all task instructions used to construct the calibration dataset for \textbf{Insert Block} task. The instructions at test-time are randomly sampled from this same set of instructions.

\para{Prompts} In our proposed framework and corresponding baselines, we use a VLM for a variety of purposes. Therefore, we list all the prompts we use in our experiments. 

Specifically, to enable the VLM to serve as a behavior narration translator $\vlmtrans$, we use the prompt listed in Fig.~\ref{fig:hardware_prompt_translate_grasp} during the grasping phase and Fig.~\ref{fig:hardware_prompt_translate_place} during the placing phase in our hardware settings. For VLM verification in our Forewarn baseline, we follow prior work~\cite{wu2025foresight} by directly prompting the VLM to select an option. The prompts are listed in Fig.~\ref{fig:hardware_prompt_forewarn_verifier}.

We also include the prompts for different score functions evaluated in Sec.~\ref{sec:experiment-coverage}. For the \textbf{Vanilla} baseline, the prompts are listed in Fig.~\ref{fig:hardware_prompt_vanilla_baseline}. For the \textbf{CoT} baseline, we include the two-step prompts that reason first about the ambiguity and then, based on this reasoning, select the option and generate the scores. For the hardware setting, the two-step prompts are included in Fig.~\ref{fig:hardware_prompt_cot_step1} and Fig.~\ref{fig:hardware_prompt_cot_step2}. For our \textbf{Bayesian Intent} method, there are three steps: 1) generate the set of possible intents; 2) evaluate each intent's probability among the set; 3) evaluate each behavior mode's probability under each specific intent. In the first step, we use the prompt in Fig.~\ref{fig:hardware_prompt_intent_generation}. Next, we generate intent probability with the prompt in Fig.~\ref{fig:hardware_prompt_intent_probability}. Finally, we infer behavior probability with the prompt in Fig.~\ref{fig:hardware_prompt_behavior_probability}.

In addition to translation and verification, the VLM generates clarification questions and automatically updates user instructions based on the responses. The prompts for this clarification phase are the same for hardware and simulation. We use the prompt in Fig.~\ref{fig:hardware_prompt_clarification_question} to ask clarification questions and the prompt in Fig.~\ref{fig:hardware_prompt_update_instruction} to update instructions once we receive the user's answer.
\begin{table*}[]
    \centering
       \renewcommand{\arraystretch}{1.5}
     \setlength{\tabcolsep}{3pt}
    \begin{tabular}{c|c|c|c|c|c|c|c|c}
    \hline
      \multirow{2}{*}{Method}& \multicolumn{4}{c|}{Straightforward } & \multicolumn{4}{c}{Ambiguous}\\
      \cline{2-9}
         & TP & TN & FP & FN &TP & TN & FP & FN\\
       \hline
        Base Policy & 0.35  & 0.65 & -- &-- & 0.35 & 0.65 &-- & -- \\
        EnsembleDagger + Residual& 0.40 & 0.60 & -- & -- & 0.50 & 0.50 & -- & -- \\
        HG-Dagger + Residual & 0.50 & 0.50 & -- & -- & 0.55 & 0.45 & -- & -- \\
        Forewarn & 0.80 & 0.10 & 0.10 & 0.0 & 0.50 & 0.0 & 0.40 & 0.10 \\
        UPS w Clarification & 0.85 & 0.05 & 0.05 & 0.05 & 0.65 & 0.25 & 0.10 & 0.0  \\
        UPS w Clarification + Residual (Ours) & \textbf{0.90}& 0.0 & 0.10 & 0.0 & \textbf{0.80} & 0.0 & 0.20 & 0.0 \\
        \hline
        
    \end{tabular}
    \caption{\textbf{Hardware - PnP Cup: Policy Steering Results Breakdown}. We include the True Positive Rate (TP), True Negative Rate (TN), False Positive Rate (FP) and False Negative Rate (FN) in the table as an extension to the results we show in Fig.~\ref{fig:continual-learn}. For straightforward and ambiguous scenarios, the reported success rate is averaged across 20 trials, each with a unique task instruction.}
    \label{tab:hardware-policy-steering}
\end{table*}
\begin{table*}[]
    \centering
    \begin{tabular}{c|c|c}
         Category&User Intent&Task Instruction  \\
         \hline
         \multirow{40}{*}{Straightforward}&left & ``on the left side, there is a bin—place the cup there''\\
    &left &``the cup should end up in the left bin''\\
    &left &``take the cup off the table and finish by putting it in the left bin''\\
    &left &``make the left bin the final location for the cup''\\
    &left &``ensure the cup is placed inside the left bin''\\
    &left &``please place the cup in the left bin''\\
    &left &``the cup should be transferred to the left orange bin''\\
    &left &``clean up the table by placing the cup in the left bin''\\
    &left &``grasp the cup and place it in the left bin''\\
    &left &``could you ensure the left bin is the final location for the cup''\\
    \cline{2-3}
    & right &``on the right side, there is a bin—send the cup there''\\
      & right &``the cup should end up in the right bin''\\
      & right &``take the cup off the table and finish by setting it in the right bin''\\
      & right &``make the right bin the final location for the cup''\\
      & right &``ensure the cup is inside the right bin''\\
      & right &``could you move the cup to the bin on the right side?''\\
      & right &``can you make sure the right-hand container is where the cup ends up?''\\
      & right &``the bin on the right needs the cup inside it; could you handle that?''\\
      & right &``the right side of the table has a bin—can you put the cup there?''\\
      & right &``could you ensure the cup is placed in the bin to your right?''\\
      \cline{2-3}
      & dirty (right) &``since the cup is dirty, it belongs in the dirty bin''\\
      & dirty (right) &``do not store it with clean items—send the cup to the dirty bin''\\
      & dirty (right) &``treat this as a dirty item and place the cup accordingly''\\
      & dirty (right) &``the destination depends on its condition: dirty cup → dirty bin''\\
      & dirty (right) &``put the cup wherever dirty cups are supposed to go (the matching bin)''\\
      & dirty (right) &``could you put the cup in the dirty bin because it is soiled?''\\
      & dirty (right) &``this cup isn't clean; can you make sure it goes in the dirty container?''\\
      & dirty (right) &``the dirty bin is the correct spot for this cup; could you put it there?''\\
      & dirty (right) &``since we have a dirty cup here, can you move it to the dirty bin?''\\
      & dirty (right) &``could you match this cup's condition by putting it in the dirty bin?''\\
      \cline{2-3}
       & clean (left) &``since the cup is clean, it belongs in the clean bin''\\
      & clean (left) &``do not mix it with dirty items—send the cup to the clean bin''\\
      & clean (left) &``treat this as a clean item and place the cup accordingly''\\
      & clean (left) &``the destination depends on its condition: clean cup → clean bin''\\
      & clean (left) &``put the cup wherever clean cups are supposed to go (the matching bin)''\\
      & clean (left) &``could you put the cup in the clean bin since it is already washed?''\\
      & clean (left) &``this cup is clean; can you ensure it ends up in the clean container?''\\
      & clean (left) &``the clean bin is the right place for this; could you set it inside?''\\
      & clean (left) &``because this is a clean cup, can you move it to the clean bin?''\\
      & clean (left) &``could you place the cup in the clean bin so it stays sanitary?''\\
 
      \hline

    \end{tabular}
    \caption{\textbf{Straightforward Hardware (PnP Cup) Task Instructions.} 40 straightforward task instructions used for the calibration dataset are listed here. During testing, we randomly sample 20 instructions from the set and obtain new options when sampling from the policy.}
    \label{tab:straightforward_hardware_instruction}
\end{table*}
\begin{table*}[]
    \centering
    \begin{tabular}{c|c|c}
         Category&User Intent&Task Instruction  \\
         \hline
       
        \multirow{40}{*}{Ambiguous}& left or right &``can you deal with the cup and put it away?''\\
      & left or right &``the cup should not be on the table anymore—put it somewhere appropriate''\\
      & left or right &``tidy this up: move the cup into a bin''\\
      & left or right &``pick up the cup and store it in a container''\\
      & left or right &``put the cup away in the bin over there''\\
      & left or right &``could you move the cup out of the working area?''\\
      & left or right &``take care of the cup and put it in a bin, please''\\
      & left or right &``can you get the cup into one bin?''\\
      & left or right &``the cup has some water—could you put it away somewhere?''\\
      & left or right &``there is water in the cup; can you relocate it to a bin?''\\
      & left or right &``just put the cup away''\\
      & left or right &``could you remove the cup from the table and place it in a container?''\\
      & left or right &``put the cup in the orange bin''\\
      & left or right &``can you put the cup into the orange container?''\\
      & left or right &``could you put the cup in a container so it is not in the way?''\\
      & left or right &``store the cup in the bin, please''\\
      & left or right &``clear the space by moving the cup into a bin''\\
      & left or right &``the cup goes in a bin—can you handle it?''\\
      & left or right &``could you move the cup to a bin and clean up the table?''\\
      & left or right &``can you put the cup somewhere it belongs, like a bin?''\\
      & left or right & "the cup should be settled in one bin''\\
      & left or right &``can you place the cup into a bin to finish cleaning?''\\
      & left or right &``put the cup into a bin to organize the scene''\\
      & left or right &``remove the cup and store it in a bin-like container''\\
            & left or right &``put the cup into a bin and keep the area neat''\\
      & left or right &``move the cup into any bin so the table is usable''\\
      & left or right &``take the cup and put it in a bin for now''\\
      & left or right &``could you put the cup into a bin so it’s not in the working zone?''\\
      & left or right &``place the cup into a bin and return the surface to empty''\\
      & left or right &``the cup should be contained—put it in a bin''\\
      & left or right &``put the cup in a bin so it won’t spill on the table''\\
      & left or right &``since the cup has liquid, move it into a bin''\\
      & left or right &``move the cup with water into a bin to avoid mess''\\
      & left or right &``put the cup somewhere secure, like inside a bin''\\
      & left or right &``could you store the cup inside a bin and clear the workspace?''\\
      & left or right &``take the cup away from the table and place it in a bin''\\
     & left or right &``can you put the cup away in one of the bins over there?''\\
      & left or right &``put the cup in a bin over there and tidy up''\\
      & left or right &``move the cup into a bin so the table is clean''\\
      & left or right &``the cup needs to be put away—use one of the bins''\\
      \hline

    \end{tabular}
    \caption{\textbf{Ambiguous Hardware (PnP Cup) Task Instructions.} 40 ambiguous task instructions used for the calibration dataset are listed here. During testing, we randomly sample 20 instructions from the set and obtain new options with the new action samples from the base policy.}
    \label{tab:ambiguous_hardware_instruction}
\end{table*}

\begin{figure*}
    \centering
    \includegraphics[width=\linewidth]{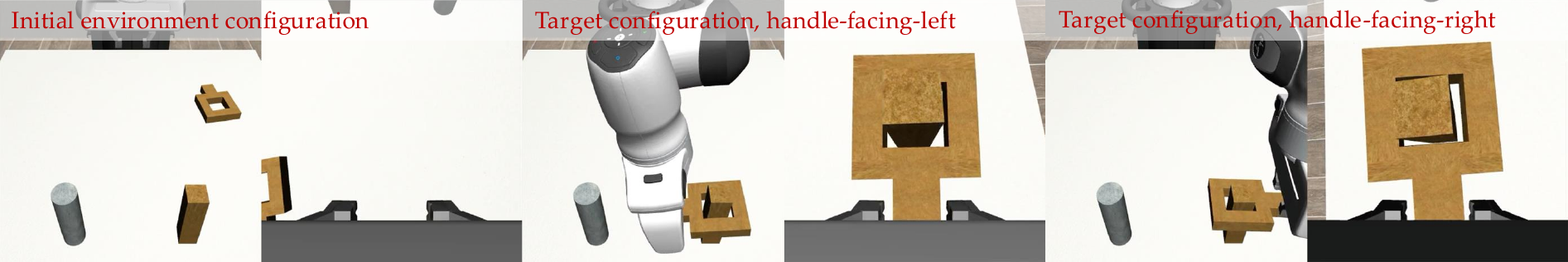}
    \caption{\textbf{Simulation Setup.} We demonstrate our simulated task setup using a Franka Emika Panda arm. The leftmost two images show the initial environment setup, while the right four images show the target configurations for the two behavior modes, as labeled. In each image pair, the left image provides the third-person camera view and the right image provides the wrist view.}
    \label{fig:sim_setup}
\end{figure*}
\begin{table*}[]
    \centering
       \renewcommand{\arraystretch}{1.5}
     \setlength{\tabcolsep}{3pt}
    \begin{tabular}{c|c|c|c|c|c|c|c|c}
    \hline
      \multirow{2}{*}{Method}& \multicolumn{4}{c|}{Straightforward } & \multicolumn{4}{c}{Ambiguous}\\
      \cline{2-9}
         & TP & TN & FP & FN &TP & TN & FP & FN\\
       \hline
        Base Policy & 0.45  & 0.55 & -- &-- & 0.25 & 0.75 &-- & -- \\
        EnsembleDagger + Residual& 0.45 & 0.55 & -- & -- & 0.30 & 0.70 & -- & -- \\
        HG-Dagger + Residual & 0.60 & 0.40 & -- & -- & 0.45 & 0.55 & -- & -- \\
        Forewarn & 0.65 & 0.20 & 0.15 & 0.0 & 0.0 & 0.25 & 0.15 & 0.60 \\
        UPS w Clarification & 0.60 & 0.20 & 0.10 & 0.10 & 0.55 & 0.25 & 0.20 & 0.0  \\
        UPS w Clarification + Residual (Ours) & \textbf{0.70}& 0.20 & 0.05 & 0.05& \textbf{0.65} & 0.10 & 0.15 & 0.10 \\
        \hline
        
    \end{tabular}
    \caption{\textbf{Simulation Policy Steering Results Breakdown}. We include the True Positive Rate (TP), True Negative Rate (TN), False Positive Rate (FP) and False Negative Rate (FN) in the table as an extension to the results we show in Fig.~\ref{fig:continual-learn}. For straightforward and ambiguous scenarios, the reported success rate is averaged across 20 trials respectively, where each trial has an unique task instruction.}
    \label{tab:simulation-policy-steering}
\end{table*}
\para{Results Breakdown} To provide a detailed analysis of our approach and related baselines regarding policy steering performance, we include a breakdown of the True Positive Rate (TP), True Negative Rate (TN), False Positive Rate (FP) and False Negative Rate (FN) in Table~\ref{tab:hardware-policy-steering}. The success rate metrics shown in Fig.~\ref{fig:continual-learn} display the number of True Positives divided by the total number of trials. As shown in Fig.~\ref{tab:hardware-policy-steering}, all policy  steering methods significantly reduce the TN compared to the base policy and DAgger variants, indicating the effectiveness of generation and verification. Among those steering methods, our UPS w/ Clarification can already significantly improve the TP and reduce the FP rate in ambiguous situations because the VLM verifier is more calibrated with conformal prediction and our proposed Bayesian Intent Score Function. With Residual Learning, UPS w/ Clarification + Residual can further reduce the TN and improve the TP, which is more obvious in ambiguous scenarios. This detailed analysis demonstrates that by calibrating high-level semantic uncertainty with low-level action infeasibility together and selectively choosing an appropriate resolution strategy, we can steer the policy to better achieve user-desired goals, especially when the instruction is ambiguous.

\subsection{Simulation}
\label{appendix:sim}

\para{Simulation Setup} We use the Robomimic NutAssemblySquare environment for our simulated task. Initially, a square nut with a handle protruding from one side lies flat on the table, while a square peg stands upright. The goal is to pick up the nut and place it onto the peg. We focus on two behavior modes to complete this task: placing the nut with its handle facing left versus placing it with the handle facing right, as shown in Fig.~\ref{fig:sim_setup}.

Our observation data consists of two simulated camera views: a wrist-mounted camera on the robot and a third-person camera providing an angled top-down view of the table and robot. Both cameras render at $96 \times 96$ resolution for diffusion policy training and deployment. For world-model training and deployment, we downsample these to $64 \times 64$, then re-upsample our decoded observations to $96 \times 96$ before passing them to the diffusion policy during interleaving.

\para{Task Instruction} We divide the different instructions that a user might give to the robot to complete the task into two categories: 1) \emph{straightforward} instructions, listed in Table~\ref{tab:sim_instruction_straightforward}; 2) \emph{ambiguous} instructions, listed in Table~\ref{tab:sim_instruction_ambiguous}. We randomly sample 40 instructions of each category to make up the calibration set, and the remaining instructions form the test set.

\para{Prompts}In our proposed framework and corresponding baselines, we use a VLM for a variety of purposes. Therefore, we list all the prompts we use in our experiments. 

Specifically, to enable the VLM to serve as a behavior narration translator $\vlmtrans$, we use the prompt listed in Fig.~\ref{fig:sim_prompt_translate_grasp} during the grasping phase and Fig.~\ref{fig:sim_prompt_translate_place} during the placing phase in our simulation setting. For VLM verification for our Forewarn baseline, we follow prior work~\cite{wu2025foresight} by directly prompting the VLM to select an option. The prompts are listed in Fig.~\ref{fig:sim_prompt_forewarn_verifier}.

We also include the prompts for different score functions evaluated in Sec.~\ref{sec:experiment-coverage}. For the \textbf{Vanilla} baseline, the prompts are listed in Fig.~\ref{fig:sim_prompt_vanilla_baseline}. For the \textbf{CoT} baseline, we include the two-step prompts that reason first about the ambiguity and then, based on this reasoning, select the option and generate the scores. For the simulation setting, the two-step prompts are included in Fig.~\ref{fig:sim_prompt_cot_step1} and Fig.~\ref{fig:sim_prompt_cot_step2}. For our \textbf{Bayesian Intent} method, there are three steps: 1) generate the set of possible intents; 2) evaluate each intent's probability among the set; 3) evaluate each behavior mode's probability under each specific intent. In the first step, we use the prompt in Fig.~\ref{fig:sim_prompt_intent_generation}. Next, we generate intent probability with the prompt in Fig.~\ref{fig:sim_prompt_intent_probability}. Finally, we infer behavior probability with the prompt in Fig.~\ref{fig:sim_prompt_behavior_probability}.

In addition to translation and verification, the VLM generates clarification questions and automatically updates user instructions based on the responses. The prompts for this clarification phase are the same as the hardware setting.

\para{Results Breakdown} Table~\ref{tab:simulation-policy-steering} demonstrates the detailed results of all the methods evaluated in Sec.~\ref{sec:experiment-uq-steering} and Sec.~\ref{sec:experiments-ups-redeployment}, including True Positive Rate (TP), True Negative Rate (TN), False Positive Rate (FP), False Negative Rate (FN). Similar to the trend we observed in hardware experiments, policy steering methods significantly reduce the TN compared to base policy and Dagger variants. In ambiguous scenarios, UPS w/ Clarification significantly improves the TP by $55\%$ through quantifying uncertainty with conformal prediction and further clarifying user intent. With Residual Learning, our method further reduces the TN and FP to improve the TP and overall success rate.
\begin{figure*}
    \centering
    \includegraphics[width=0.8\linewidth]{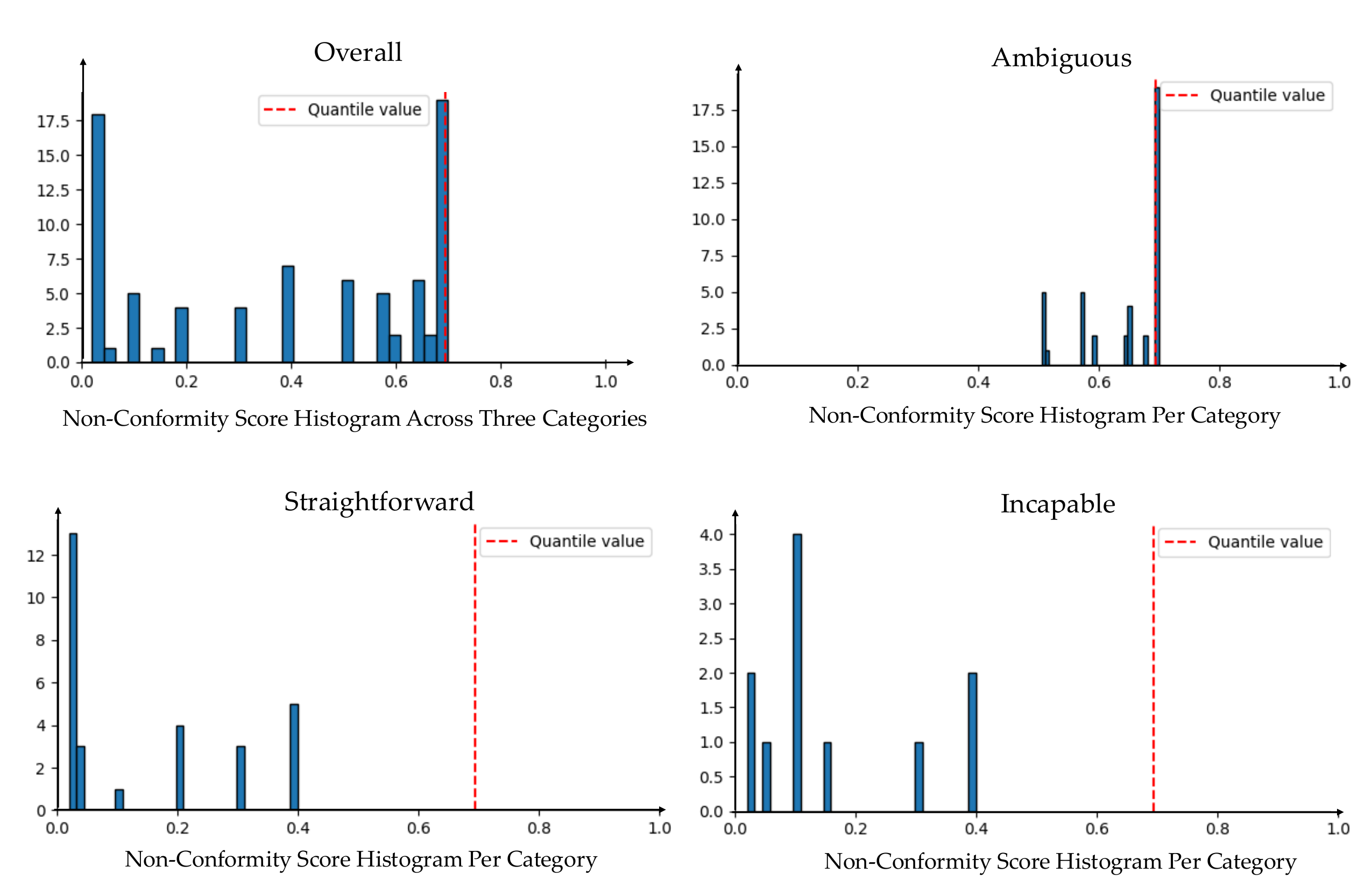}
    \caption{\textbf{Hardware PnP Cup Histogram for Empirical Quantile.} We list the overall histogram with all three categories combined, as well as per category histograms. The red vertical line indicates the $\hat{q}$ value based on user-desired coverage rate $1-\varepsilon$.  }
    \label{fig:hardware_q_hat_quantile}
\end{figure*}
\begin{figure*}
    \centering
    \includegraphics[width=0.8\linewidth]{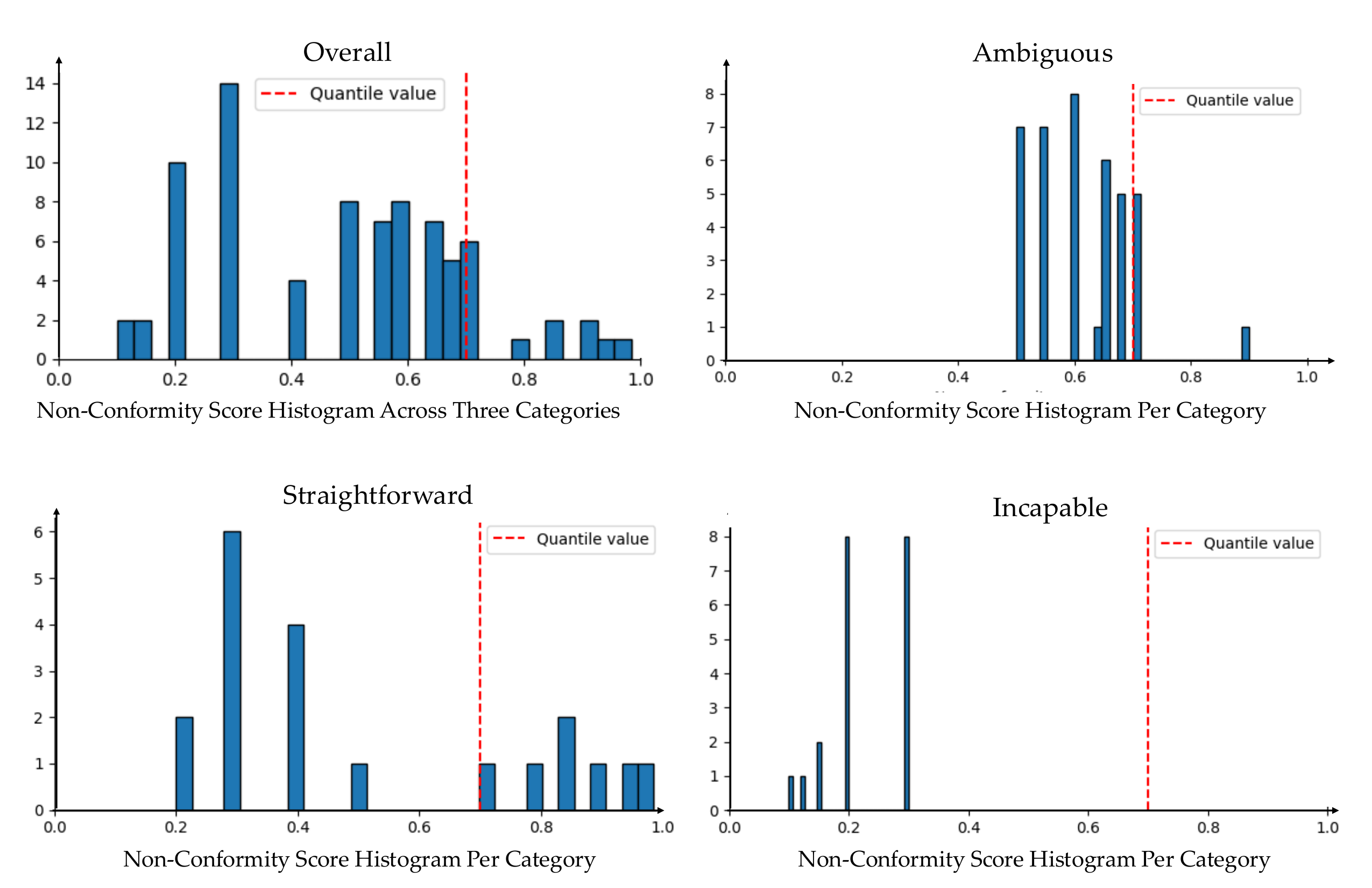}
    \caption{\textbf{Hardware Insert Block Histogram for Empirical Quantile.} We list the overall histogram with all three categories combined, as well as per category histograms. The red vertical line indicates the $\hat{q}$ value based on user-desired coverage rate $1-\varepsilon$.  }
    \label{fig:hardware_q_hat_quantile_block}
\end{figure*}
\begin{figure*}
    \centering
    \includegraphics[width=0.9\linewidth]{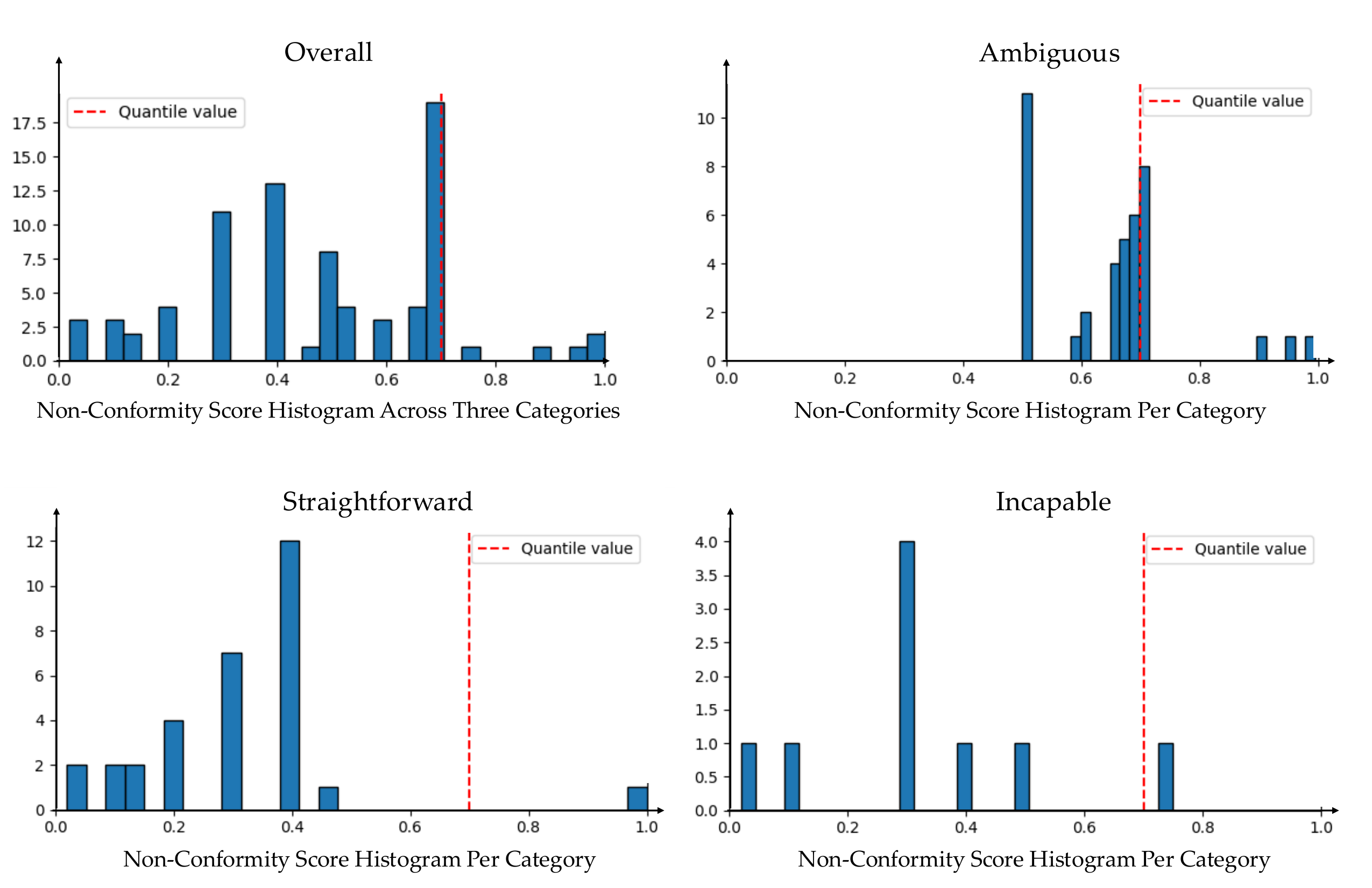}
    \caption{\textbf{Simulation Histogram for Empirical Quantile.} We list the overall histogram with all three categories combined, as well as per category histograms. The red vertical line indicates the $\hat{q}$ value based on user-desired coverage rate $1-\varepsilon$. }
    \label{fig:sim_q_hat_quantile}
\end{figure*}
\subsection{Calibration Details}
\label{appendix:calibration}

For both hardware and simulation, our calibration dataset contains 80 instructions and our test dataset contains 40 instructions, with each split evenly between ambiguous and straightforward instructions. For each instruction, we sample $K=10$ action trajectories for phase 1 (grasping) and translate them to text using our scheme from Sec~\ref{sec:approach}. We repeat this process for phase 2 (placing), then combine each instruction with its phase 1 and phase 2 narrations to form a complete scenario.

Among scenarios with \textbf{straightforward} instructions, the incapability rates (where no narrations achieve the user's instruction) are as follows. In simulation: calibration scenarios show $0/40$ grasping failures and $9/40$ placing failures, while test scenarios show $0/20$ grasping failures and $5/20$ placing failures. In hardware: calibration scenarios show $2/40$ grasping failures and $10/40$ placing failures, while test scenarios show $1/20$ grasping failures and $5/20$ placing failures.

\begin{figure}
    \centering
    \includegraphics[width=\linewidth]{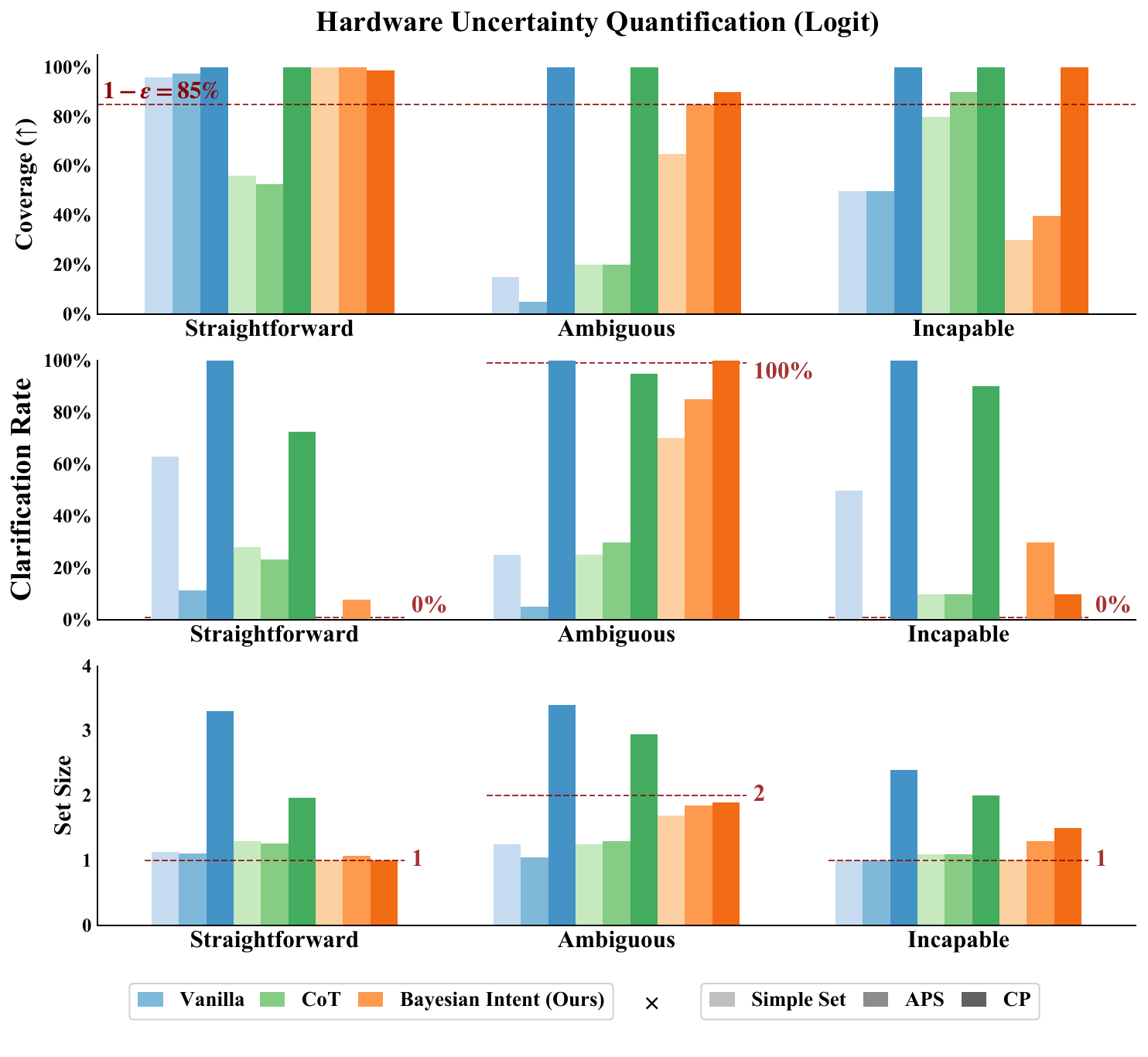}
    \caption{\textbf{Hardware \textbf{PnP Cup} Uncertainty Quantification with Logits}. We conduct an ablation study where we directly use the softmax over the logits of the first generated token as the score rather than the self-generated score. Empirically, we find that, compared to Fig.~\ref{fig:results-UQ-hardware}, the self-generated score is a better score function across all scenarios.}
    \label{fig:hardware_uq_plot_logit}
\end{figure}
\begin{figure}
    \centering
    \includegraphics[width=\linewidth]{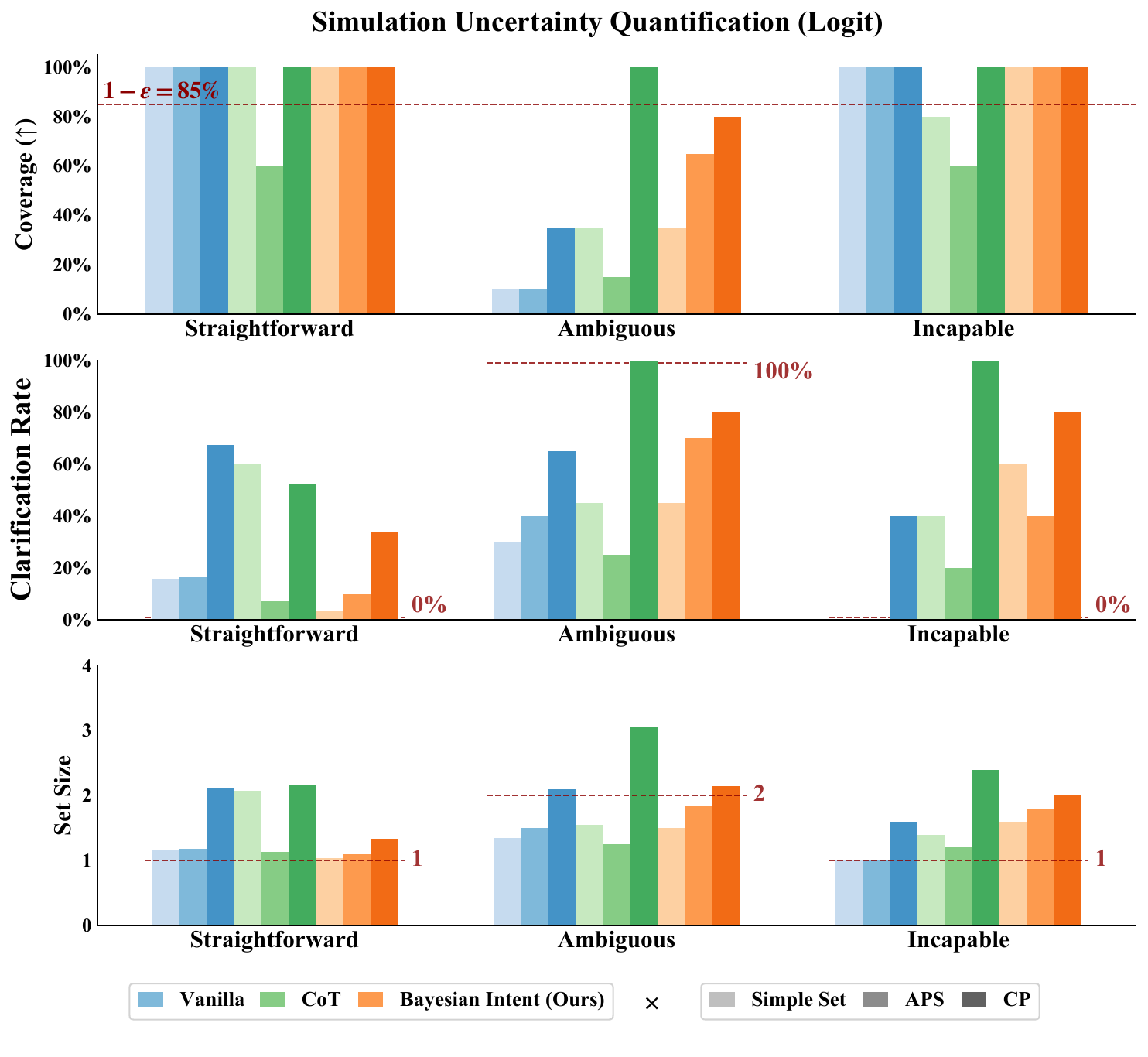}
    \caption{\textbf{Simulation Uncertainty Quantification with Logits}. Similar to hardware experiments, we compare the results of the softmax of logits as the score with self-generated score results in Fig.~\ref{fig:results-UQ-simulation}. None of the methods with logits can achieve the desired coverage while balancing the clarification rate and reaching the desired prediction set size.}
    \label{fig:sim_uq_plot_logit}
\end{figure}

\para{Results: Nonconformity Score Distribution} For the histograms depicted in Figs.~\ref{fig:sim_q_hat_quantile}, ~\ref{fig:hardware_q_hat_quantile}, \ref{fig:hardware_q_hat_quantile_block} low nonconformity scores are desirable, as this would indicate our score function correctly assigns high probability mass to the true label. In both hardware and simulation, examining the distribution of non-conformity scores for the calibration set reveals that the majority of samples for both the straightforward and incapable cases have nonconformity scores less than $0.5$ (with all but 3 samples meeting this threshold). In the hardware setting specifically, most straightforward and ambiguous samples have nonconformity scores less than $0.2$. This indicates that the true label---which is unambiguous in these cases---receives more than 50\% probability mass, demonstrating that our score function performs well in these scenarios.

In the ambiguous cases, on the other hand, we see the vast majority of samples (all but 3 samples) having nonconformity scores ranging between 0.5 and 0.7. Recall that in this instance, we take the maximum nonconformity score over two possible valid true labels, so it is not possible for any sample to have a nonconformity score under 0.5 as this would imply that both valid options are assigned greater than 50\% probability mass. Therefore, these results imply that across both simulation and hardware, both valid true labels receive at least 30\% probability mass, indicating that our score function appropriately distributes probability between the two valid options rather than erroneously favoring one over the other. Together, these calibration results demonstrate that our score function behaves as intended across all three case types.

\para{Results: Uncertainty Quantification with Logits} We compare the uncertainty quantification results from using the softmax of the logits from the VLM in Figs. ~\ref{fig:sim_uq_plot_logit} and ~\ref{fig:hardware_uq_plot_logit} to those that ask the VLM to directly self-generate probability score for each option in Figs.~\ref{fig:results-UQ-simulation} and ~\ref{fig:results-UQ-hardware}, in simulation and hardware. For the score, we compute a softmax function over the logits of the first token (e.g., A, B, C, D) with a temperature of 4.  We find that the score functions derived from the softmax over the logits are often biased towards certain options, making other valid options unlikely to be included in the prediction set. Therefore, you can see very low coverage rate for \textbf{Vanilla} and \textbf{CoT} methods in ambiguous cases compared to self-generated probability scores. Other methods have the desired coverage in ambiguous cases but have larger clarification rate in straightforward and incapable cases because the $\hat{q}$ is increased to cover ambiguous cases. Only the CP with Bayesian Intent (Ours) exhibits good performance across all scenarios in the hardware setting but it is still worse than the performance of self-generated probability scores. 

\subsection{Residual Policy}
\label{appendix:residual}

    \para{Data Collection} For all interactive imitation learning approaches, when an intervention is deemed necessary (by the human for HG-DAgger or UPS, or by the robot for EnsembleDAgger), the base policy relinquishes control to the user, who teleoperates the robot via SpaceMouse before handing back authority. During human intervention, we continue to run the base policy without deploying its actions, computing $\Delta a = a_{\text{human}} - a_{\text{base}}$ at each timestep. Intervention labels are set to 1 during human operation and 0 otherwise.

    \para{Architecture} We separately train a classifier (to predict intervention labels) and a residual policy (to predict residual actions). Both models share the same input features: the observation (camera images and proprioceptive states) and the base policy action at each timestep. All inputs are normalized to the range of $(-1, 1)$, including residual action labels for the residual policy.

    Both models use the diffusion policy observation encoder, where each camera image is processed through its own copy of a modified ResNet-18 module. These outputs are concatenated with proprioceptive data and passed through a 2-layer MLP with 512 hidden dimensions \cite{chi2024diffusionpolicy}. All encoder weights are initialized from the base diffusion policy checkpoint but remain trainable throughout. The base policy action is processed through a single linear layer before being concatenated with the observation encoder output.
    
    For the classifier, this concatenated representation passes through a linear layer to produce a 1-dimensional output, followed by a sigmoid activation to obtain probabilities. We train using binary cross-entropy loss, weighting positive samples to address class imbalance.
    
    For the residual policy, the representation is passed through an action predictor MLP head to produce an output clamped to (-1,1) using tanh. We train using mean squared error loss.
    
      \para{Deployment} At deployment, we query the classifier at each timestep with the observation and the base action. If it predicts an intervention, we query the residual policy and execute the combined base and residual actions; otherwise, we execute the base action alone. For the continual learning baselines HG-DAgger and EnsembleDAgger, we use real observations. However, for UPS (ours), we use decoded latent observations from the world model to enable us to imagine the outcomes of actions sampled from the combined policy. We generate $K/2 = 5$ action sequences using the combined policy and $K/2$ using the base policy alone to mitigate catastrophic forgetting in the event that the residual policy overcorrects toward modes underrepresented in the base policy.

    \begin{table}[]
    \centering
    \begin{tabular}{c|c}
         Hyperparameter & Value  \\
         \hline
         Batch Size & 64 \\
         Learning Rate & 1e-4 \\
         Iterations & 10000 \\
         Dropout & 0.2 \\
         Base Action Encoder Output Dimension & 128 \\
         Action Predictor Hidden Dimension & 256 \\
         Action Predictor Number of Hidden Layers & 2 \\
         \hline
      
    \end{tabular}
    \caption{\textbf{Hyperparameters for Residual Policy Training.} }
    \label{tab:residual_policy_hyperparams}
\end{table}
\begin{table*}[]
    \centering
    \begin{tabular}{c|c}
         User Intent&Task Instruction  \\
         \hline
    left & ``Please place the nut on the peg with the handle facing left.''\\
    left &``I want the nut positioned on the peg with the handle directed left.''\\
    left &``Put the nut on the peg so that its handle is oriented leftward.''\\
    left &``Set the nut onto the peg with the handle pointing left.''\\
    left &``The nut goes on the peg with its handle turned to face left.''\\
    left &``Can the nut be placed on the peg so its handle is pointing to the left?''\\
    left &``Set it up so that the nut's handle faces left while on the peg.''\\
    left &``The nut should be positioned on the peg with its handle going left.''\\
    left &``Place the nut on the peg such that its handle is directed to the left.''\\
    left &``I need the nut placed with a left-pointing handle on the peg.''\\
    left &``Pick up the nut, turn its handle to the left, and place it on the peg.''\\
    left &``Can you rotate the nut's handle to the left and set it on the peg?''\\
    left &``The nut should go on the peg with its handle facing left.''\\
    left &``I want the nut placed on the peg so that I can easily grab it from the peg with my left hand.''\\
    left &``Please set the nut on the peg such that it is readily accessible using my left hand.''\\
    
    \cline{1-2}
    right &``The nut should go on the peg with its handle facing right.''\\
    right &``Make sure the nut's handle points right when you put it on the peg.''\\
    right &``I need you to place the nut so the handle is on the right side.''\\
    right &``Could you put the nut on the peg? The handle needs to face right.''\\
    right &``Place the nut on the peg such that its handle is directed to the right.''\\
    right &``Position the nut on the peg and turn the handle to face right.''\\
    right &``I need the nut placed with a right-pointing handle on the peg.''\\
    right &``Put the nut on the peg and make sure its handle faces right.''\\
    right &``The nut belongs on the peg with its handle turned rightward.''\\
    right &``When placing the nut on the peg, ensure its handle is oriented toward the right.''\\
    right &``With the handle pointing rightward, place the nut onto the peg.''\\
    right &``Turn the nut until its handle faces right, then set it on the peg.''\\
    right &``I'd like you to place the nut on the peg with its handle pointing to the right.''\\
    right &``Place the nut on the peg in a location my right hand can access with ease.''\\
    right &``Set the nut onto the peg so that right-hand retrieval is straightforward.''\\
      
      \cline{1-2}
    c.c.w. (left) &``Place the nut onto the peg with the handle directed at nine o'clock.''\\
    c.c.w. (left) &``I want the nut set on the peg so that its handle is pointing at nine o'clock.''\\
    c.c.w. (left) &``Can you please position the nut on the peg with its handle oriented toward nine o'clock.''\\
    c.c.w. (left) &``Set the nut onto the peg such that the handle faces the nine o'clock direction.''\\
    c.c.w. (left) &``The handle should be at nine o'clock when you set the nut on the peg.''\\
    c.c.w. (left) &``I need you to orient the nut's handle toward nine o'clock when you place it on the peg.''\\
    c.c.w. (left) &``When positioning the nut on the peg, make sure the handle faces nine o'clock please.''\\
    c.c.w. (left) &``The nut goes on the peg with its handle angled toward nine o'clock.''\\
    c.c.w. (left) &``Can you turn the nut counter-clockwise 90 degrees and set it on the peg?''\\
    c.c.w. (left) &``Please rotate the nut a quarter turn counter-clockwise before putting it on the peg.''\\
    c.c.w. (left) &``I need you to give the nut a 90-degree counter-clockwise twist then place it on the peg.''\\
    c.c.w. (left) &``I want the nut rotated a quarter turn counter-clockwise prior to placing it on the peg.''\\
    c.c.w. (left) &``The nut should be turned a quarter turn counter-clockwise as you put it on the peg.''\\
    c.c.w. (left) &``Give the nut a 90-degree counter-clockwise rotation and then mount it on the peg.''\\
    c.c.w. (left) &``Before you set the nut on the peg, turn it a quarter turn counter-clockwise.''\\
      
      \cline{1-2}
    c.w. (right) &``When positioning the nut on the peg, make sure the handle faces three o'clock please.''\\
    c.w. (right) &``The nut goes on the peg with its handle angled toward three o'clock.''\\
    c.w. (right) &``I want you to place the nut on the peg while ensuring the handle is at three o'clock.''\\
    c.w. (right) &``Put the nut on the peg after rotating it so the handle points to three o'clock.''\\
    c.w. (right) &``With the handle at three o'clock, place the nut onto the peg.''\\
    c.w. (right) &``The nut should be placed on the peg with its handle pointing three o'clock.''\\
    c.w. (right) &``Set the nut on the peg, making sure the handle is directed at three o'clock.''\\
    c.w. (right) &``Please drop the nut onto the peg while keeping the handle oriented to three o'clock.''\\
    c.w. (right) &``Before you set the nut on the peg, turn it a quarter turn clockwise.''\\
    c.w. (right) &``Can you spin the nut 90 degrees clockwise while positioning it on the peg?''\\
    c.w. (right) &``Turn the nut clockwise 90 degrees, then place it on the peg.''\\
    c.w. (right) &``I'd like you to rotate the nut a quarter turn clockwise and put it on the peg.''\\
    c.w. (right) &``Please give the nut a clockwise quarter turn before placing it onto the peg.''\\
    c.w. (right) &``Twist the nut clockwise by 90 degrees, then set it down on the peg.''\\
    c.w. (right) &``I need the nut turned a quarter turn clockwise when you place it on the peg.''\\
    \hline

    \end{tabular}
    \caption{\textbf{Straightforward Simulation Task Instructions.} 60 straightforward task instructions used for the calibration and test datasets are listed here. 40 of these are randomly sampled to be part of the calibration set, while the remaining 20 make up the test set. We use ``c.c.w.'' and ``c.w.'' to denote ``counterclockwise'' and ``clockwise'' respectively.}
    \label{tab:sim_instruction_straightforward}
\end{table*}

\begin{table*}[]
    \centering
    \begin{tabular}{c|c}
         User Intent&Task Instruction  \\
         \hline
      
    left or right &``Please place the nut on the peg.''\\
    left or right &``I want the nut positioned on the peg.''\\
    left or right &``Pick up the nut and set it on the peg.''\\
    left or right &``Can you set the nut on the peg?''\\
    left or right &``Put the nut onto the peg.''\\
    left or right &``I need you to place the nut on the peg.''\\
    left or right &``Move the nut to the peg.''\\
    left or right &``Could you position the nut on the peg?''\\
    left or right &``Set the nut down on the peg.''\\
    left or right &``I want you to put the nut on the peg.''\\
    left or right &``Could you grab the nut and put it on the peg?''\\
    left or right &``Please move the nut onto the peg.''\\
    left or right &``The nut needs to go on the peg.''\\
    left or right &``Can you mount the nut on the peg?''\\
    left or right &``I'd like the nut placed onto the peg.''\\
    left or right &``I'm going to need the nut positioned on the peg.''\\
    left or right &``Would you mind putting the nut onto the peg?''\\
    left or right &``Go ahead and place the nut on the peg.''\\
    left or right &``Can you get the nut situated on the peg?''\\
    left or right &``I'd like you to move the nut over to the peg.''\\
    left or right &``Put that nut on the peg right now please.''\\
    left or right &``Transfer the nut so it sits on the peg.''\\
    left or right &``I'm asking you to get the nut onto the peg.''\\
    left or right &``Would you be able to secure the nut on the peg?''\\
    left or right &``Set that nut down on the peg for me please.''\\
    left or right &``I need the nut transferred onto the peg.''\\
    left or right &``Please get the nut mounted over the peg.''\\
    left or right &``Could you move the nut over and place it on the peg?''\\
    left or right &``I want that nut relocated to the peg.''\\
    left or right &``Would you fix the nut onto the peg?''\\
    left or right &``I want the nut placed on the peg so that its handle points toward my non-dominant hand side.''\\
    left or right &``Put the nut on the peg with its handle turned in the direction of my non-preferred hand.''\\
    left or right &``Could you move the nut onto the peg so that its handle is oriented toward my non-dominant hand?''\\
    left or right &``Please set the nut on the peg with the handle facing the side of my non-preferred hand.''\\
    left or right &``I need the nut positioned on the peg so that its handle points to my non-dominant hand side.''\\
    left or right &``Place the nut on the peg with its handle directed toward my non-preferred hand.''\\
    left or right &``Can you put the nut on the peg so that its handle is aligned with my non-dominant hand?''\\
    left or right &``Set the nut onto the peg with its handle facing in the direction of my non-preferred hand.''\\
    left or right &``I want you to place the nut on the peg so that its handle points toward my non-dominant hand side.''\\
    left or right &``Could you position the nut on the peg with its handle turned toward my non-preferred hand?''\\
    left or right &``Please put the nut on the peg so that its handle is oriented in line with my non-dominant hand.''\\
    left or right &``I need you to set the nut on the peg with its handle facing toward my non-preferred hand side.''\\
    left or right &``Mount the nut onto the peg so that its handle points in the direction of my non-dominant hand.''\\
    left or right &``Can you place the nut on the peg with its handle aligned toward my non-preferred hand?''\\
    left or right &``Set up the nut on the peg so that its handle is directed toward my non-dominant hand side.''\\
    left or right &``Mount the nut onto the peg so that the handle is directed towards my dominant-hand side.''\\
    left or right &``Can you please attach the nut onto the peg with the handle facing my preferred hand side?''\\
    left or right &``I need the nut placed on the peg so that its handle points toward my dominant hand side.''\\
    left or right &``Put the nut on the peg with its handle turned in the direction of my preferred hand side.''\\
    left or right &``Please set the nut on the peg with the handle facing the side of my preferred hand.''\\
    left or right &``I want the nut positioned on the peg so that its handle points to my dominant hand side.''\\
    left or right &``Place the nut on the peg with its handle directed toward my preferred hand.''\\
    left or right &``Can you put the nut on the peg so that its handle is aligned with my dominant hand?''\\
    left or right &``Set the nut onto the peg with its handle facing in the direction of my preferred hand.''\\
    left or right &``I want you to place the nut on the peg so that its handle points toward my dominant hand side.''\\
    left or right &``Could you position the nut on the peg with its handle turned toward my preferred hand side?''\\
    left or right &``Please put the nut on the peg so that its handle is oriented in line with my dominant hand.''\\
    left or right &``I need you to set the nut on the peg with its handle facing toward my preferred hand side.''\\
    left or right &``Mount the nut onto the peg so that its handle points in the direction of my dominant hand.''\\
    left or right &``Can you place the nut on the peg with its handle aligned toward my preferred hand?''\\
      \hline

    \end{tabular}
    \caption{\textbf{Ambiguous Simulation Task Instructions.} 60 ambiguous task instructions used for the calibration and test datasets are listed here. 40 of these are randomly sampled to be part of the calibration set, while the remaining 20 make up the test set.}
    \label{tab:sim_instruction_ambiguous}
\end{table*}

\begin{table*}[]
    \centering
    \begin{tabular}{c|c|c}
         Category & User Intent & Task Instruction  \\
         \hline

        \multirow{10}{*}{Straightforward} & left & ``Insert the purple block into the left hole.''\\
        & left & ``Attach the purple block to the left-side hole.''\\
        & left & ``Fit the purple block into the left corner.''\\
        & left & ``Push the purple block into the left slot.''\\
        & left & ``Align the purple block with the left hole and insert it.''\\
        & left & ``Fasten the purple block into the hole on the left.''\\
        & left & ``Place the purple block into the left hole.''\\
        & left & ``Slide the purple block into the left-side slot.''\\
        & left & ``Position the purple block in the left hole.''\\
        & left & ``Fix the purple block into the left corner.''\\
        \hline

        \multirow{10}{*}{Straightforward} & right & ``Put the purple block into the right hole.''\\
        & right & ``Place the purple block into the hole on the right.''\\
        & right & ``Secure the purple block in the right hole.''\\
        & right & ``Mount the purple block in the hole on the right side.''\\
        & right & ``Insert the purple block into the right-side slot.''\\
        & right & ``Set the purple block into the right corner.''\\
        & right & ``Connect the purple block to the right hole.''\\
        & right & ``Install the purple block in the hole on the right.''\\
        & right & ``Insert this purple block into the right hole.''\\
        & right & ``Put this purple block into the hole on the right side.''\\
        \hline

        \multirow{20}{*}{Ambiguous} & left or right & ``Insert the purple block into one of the holes.''\\
        & left or right & ``Put the purple block into the correct hole.''\\
        & left or right & ``Attach the purple block to the hole on the tabletop.''\\
        & left or right & ``Place the purple block where it belongs.''\\
        & left or right & ``Fit the purple block into an open hole.''\\
        & left or right & ``Secure the purple block in the proper slot.''\\
        & left or right & ``Push the purple block into the matching hole.''\\
        & left or right & ``Mount the purple block in one of the tabletop holes.''\\
        & left or right & ``Align the purple block with a hole and insert it.''\\
        & left or right & ``Insert the purple block into the designated slot.''\\
        & left or right & ``Fasten the purple block into the right spot.''\\
        & left or right & ``Set the purple block into the available hole.''\\
        & left or right & ``Place the purple block into the hole.''\\
        & left or right & ``Connect the purple block to the appropriate hole.''\\
        & left or right & ``Slide the purple block into the opening on the tabletop.''\\
        & left or right & ``Install the purple block in its hole.''\\
        & left or right & ``Position the purple block in the correct place.''\\
        & left or right & ``Insert this purple block into the intended hole.''\\
        & left or right & ``Fix the purple block into a slot.''\\
        & left or right & ``Put this purple block where it should go.''\\
        \hline
      
    \end{tabular}
    \caption{\textbf{Block Insertion Task Instructions.} Straightforward task instructions specify whether the purple block should be inserted into the left or right hole, while ambiguous instructions underspecify the target hole and require disambiguation.}
    \label{tab:ambiguous_block_instruction}
\end{table*}

\begin{figure}
    \centering
    \includegraphics[width=0.9\linewidth]{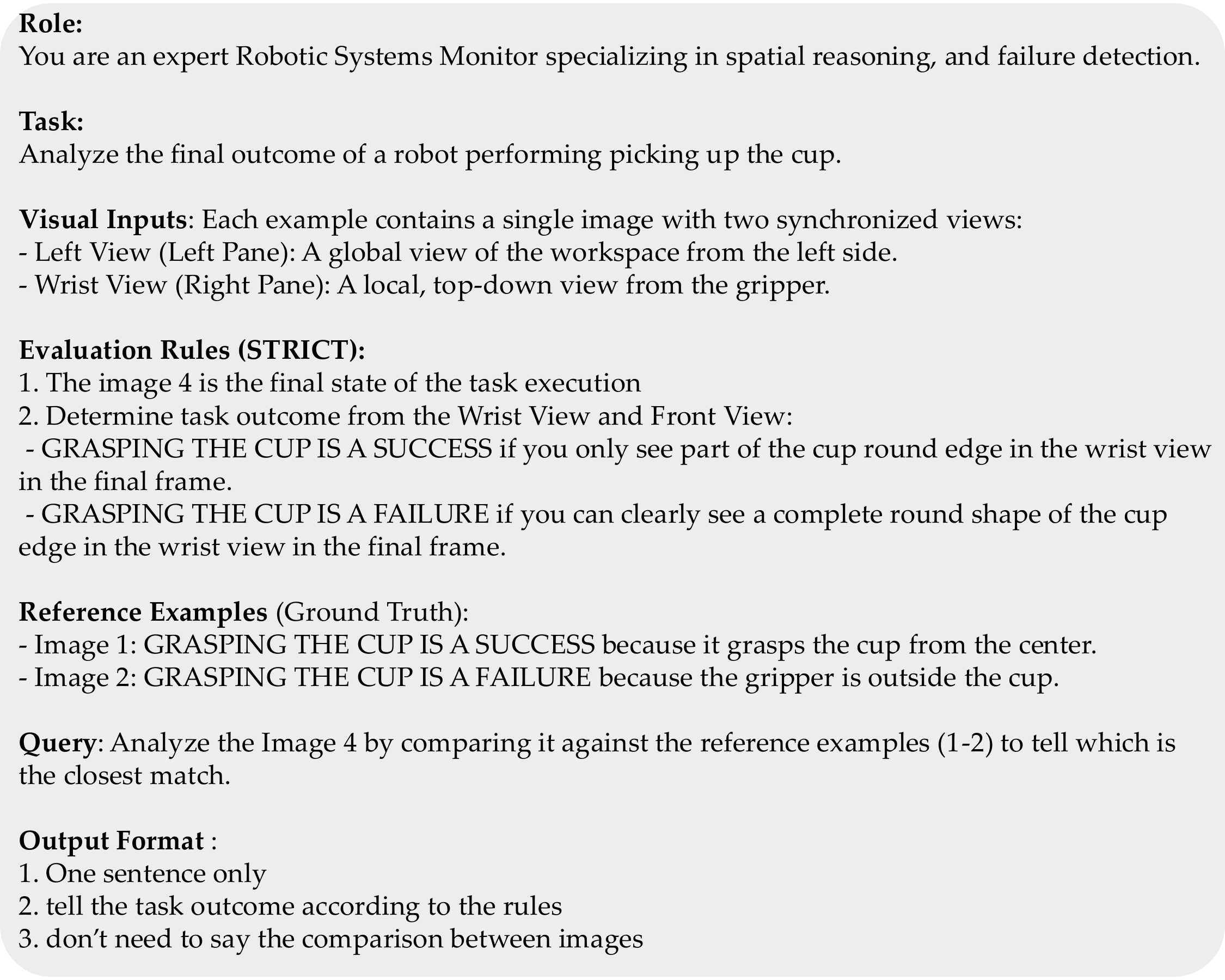}
    \caption{\textbf{Hardware Prompt For Behavior Translation (Grasp)}}
    \label{fig:hardware_prompt_translate_grasp}
\end{figure}

\begin{figure}
    \centering
    \includegraphics[width=0.9\linewidth]{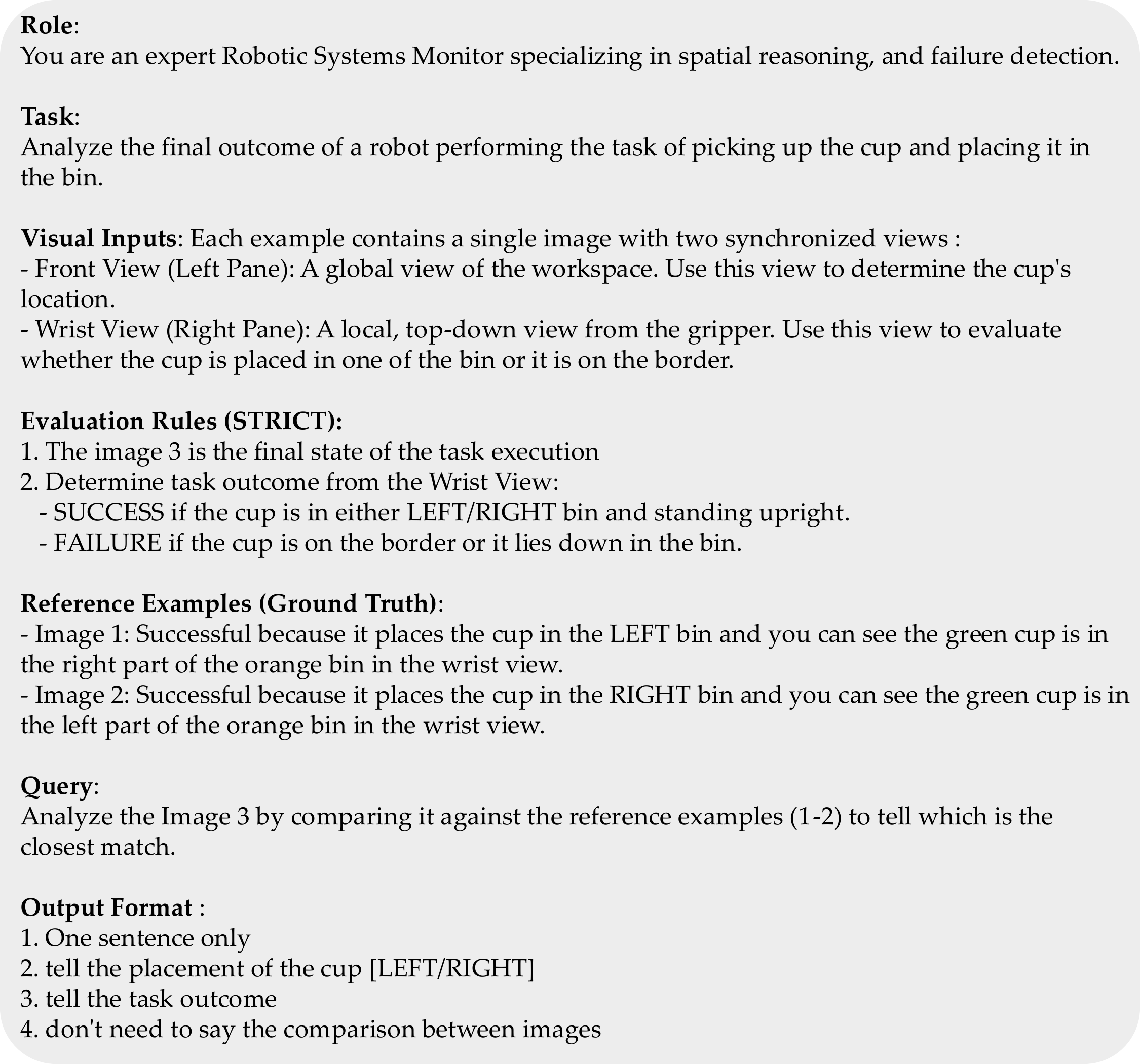}
    \caption{\textbf{Hardware Prompt For Behavior Translation (Place)}}
    \label{fig:hardware_prompt_translate_place}
\end{figure}
\begin{figure}
    \centering
    \includegraphics[width=0.9\linewidth]{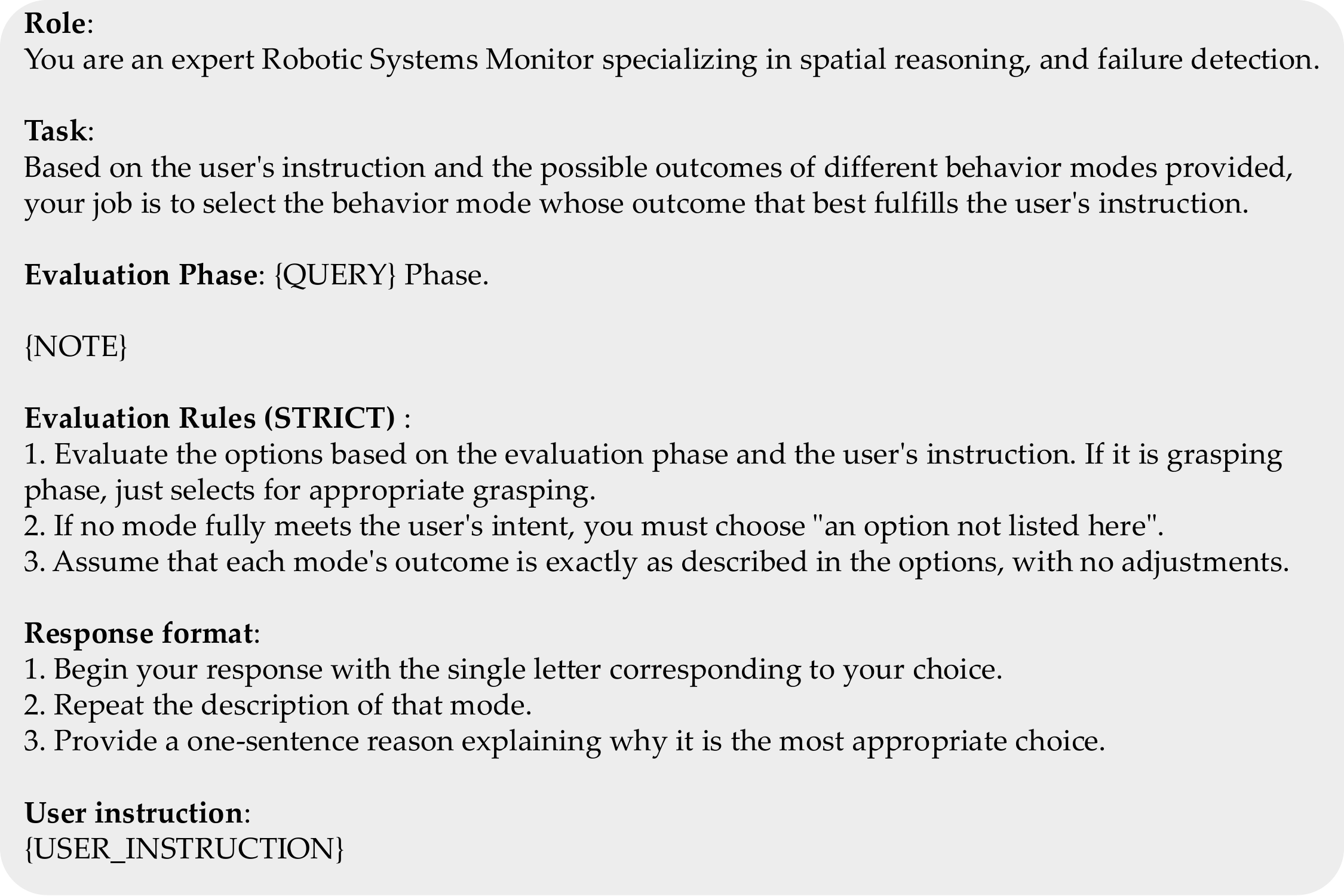}
    \caption{\textbf{Hardware Prompt for VLM Verification in Forewarn}}
    \label{fig:hardware_prompt_forewarn_verifier}
\end{figure}
\begin{figure}
    \centering
    \includegraphics[width=0.9\linewidth]{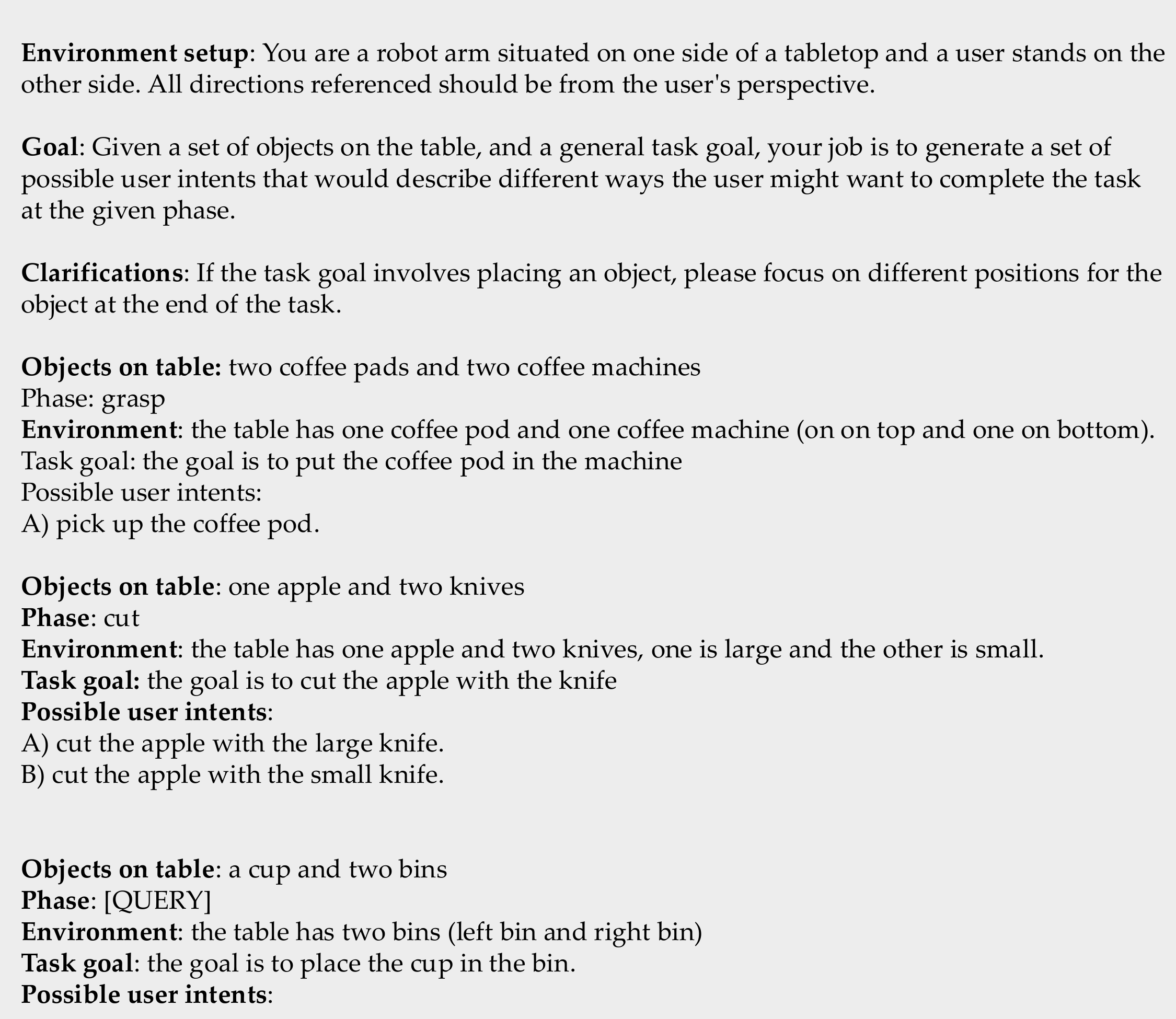}
    \caption{\textbf{Hardware Prompt for Generating Possible Intents}}
    \label{fig:hardware_prompt_intent_generation}
    \vspace{-1.0cm}
\end{figure}
\begin{figure}
    \centering
    \includegraphics[width=0.9\linewidth]{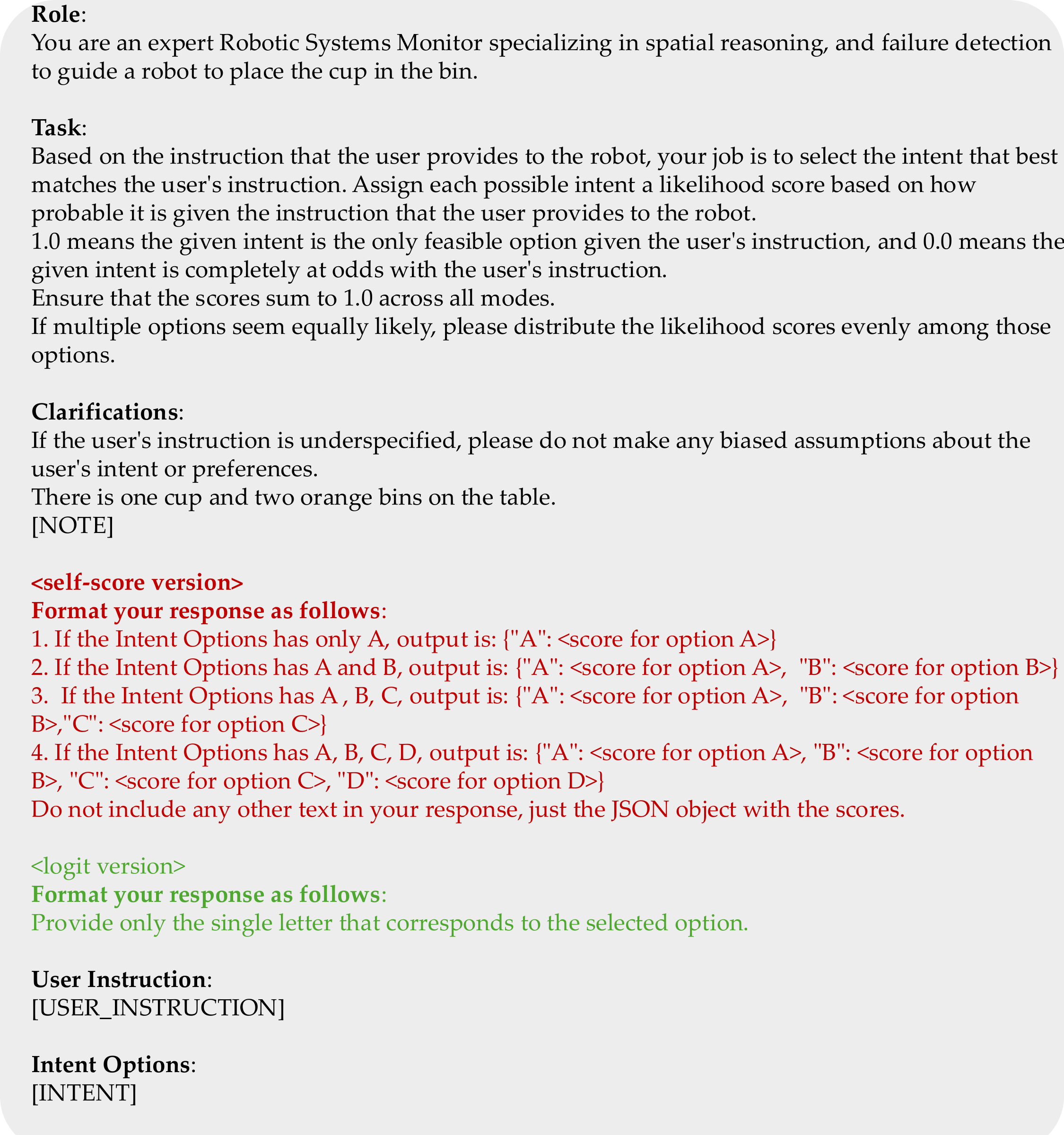}
    \caption{\textbf{Hardware Prompt for Inferring Intent Probability}}
    \label{fig:hardware_prompt_intent_probability}
    \vspace{-1.0cm}
\end{figure}
\begin{figure}
    \centering
    \includegraphics[width=0.9\linewidth]{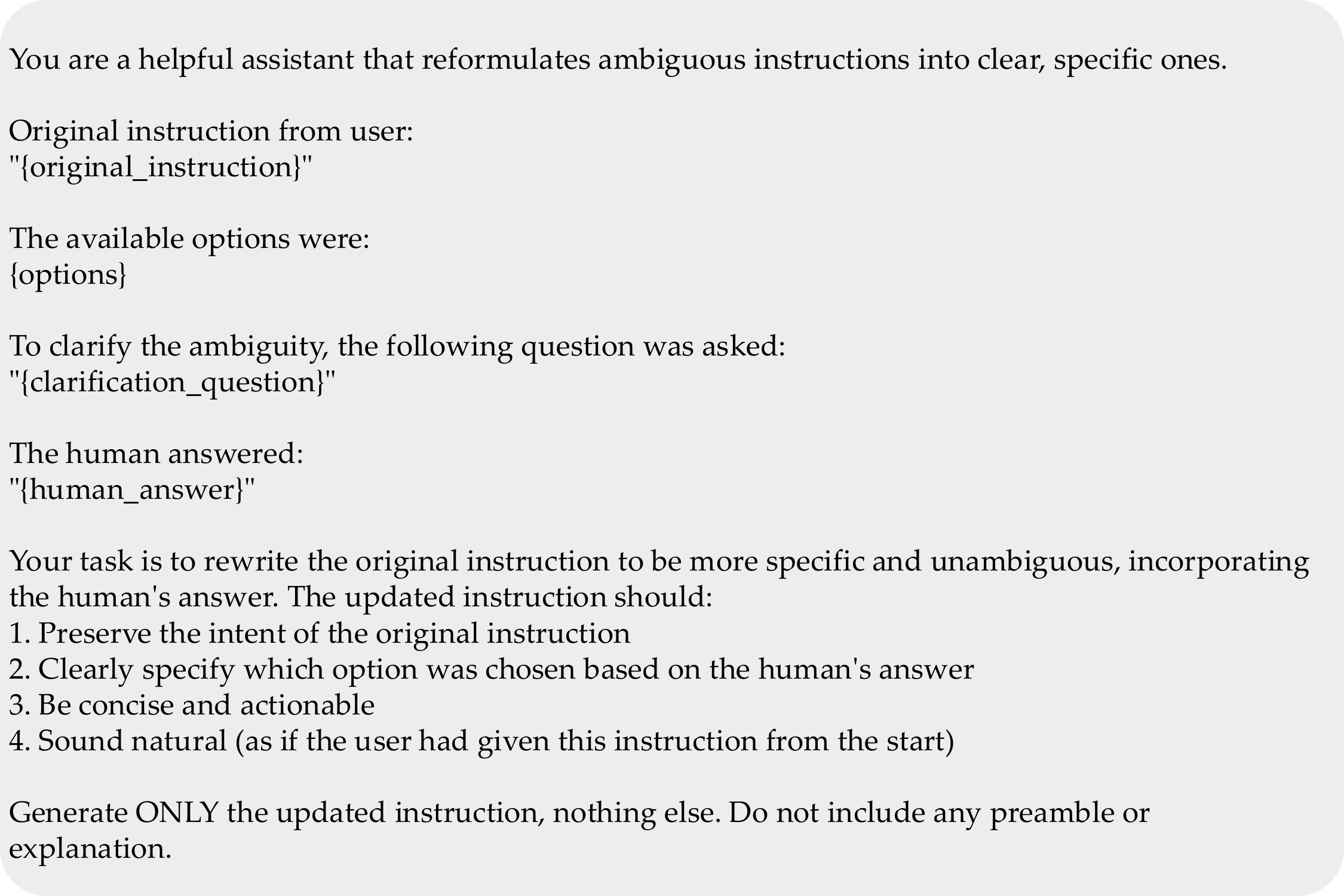}
    \caption{\textbf{Hardware Prompt for Updating Instructions}}
    \label{fig:hardware_prompt_update_instruction}
\end{figure}
\begin{figure}
    \centering
    \includegraphics[width=0.9\linewidth]{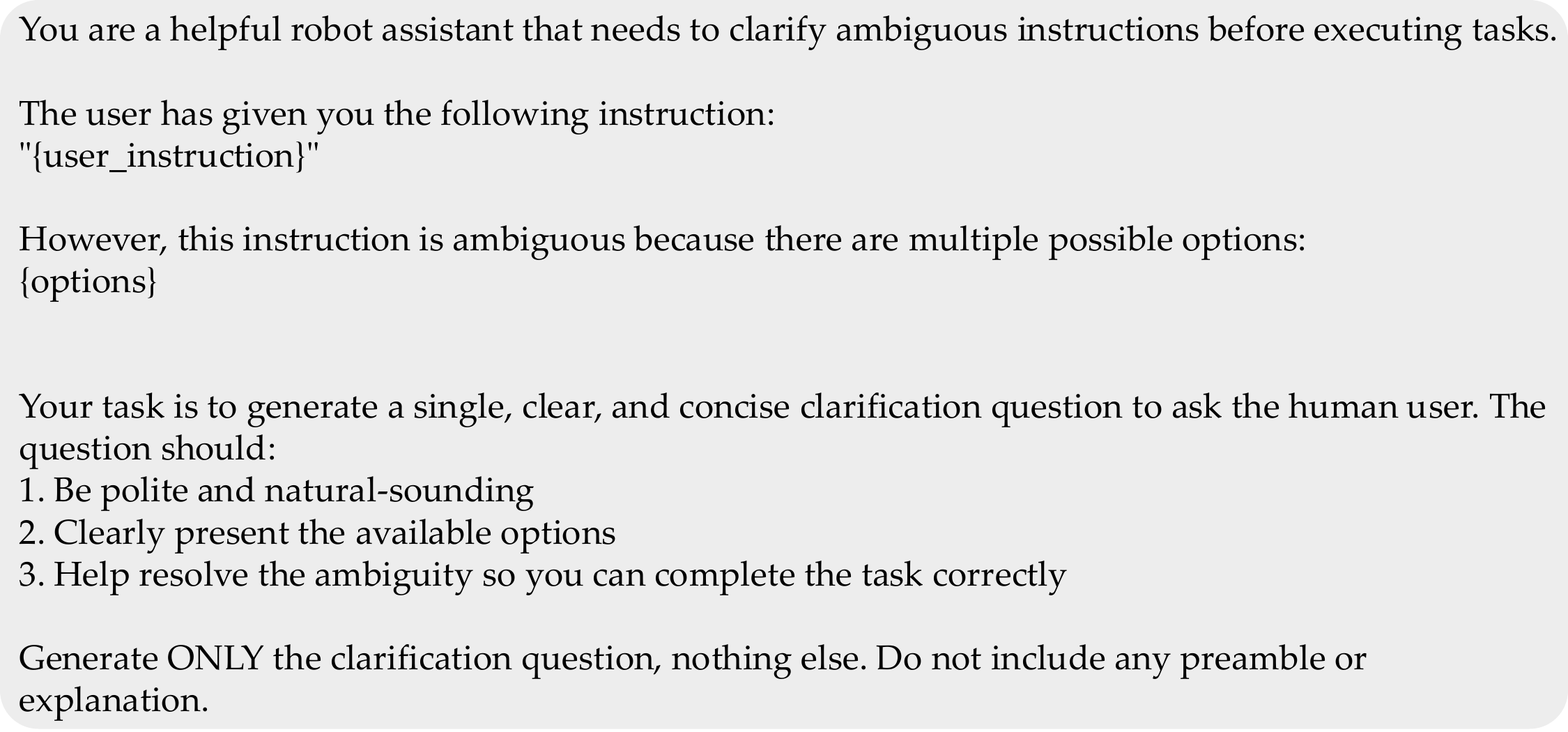}
    \caption{\textbf{Hardware Prompt for Clarification Questions}}
        \label{fig:hardware_prompt_clarification_question}
\end{figure}
\begin{figure}
    \centering
    \includegraphics[width=0.9\linewidth]{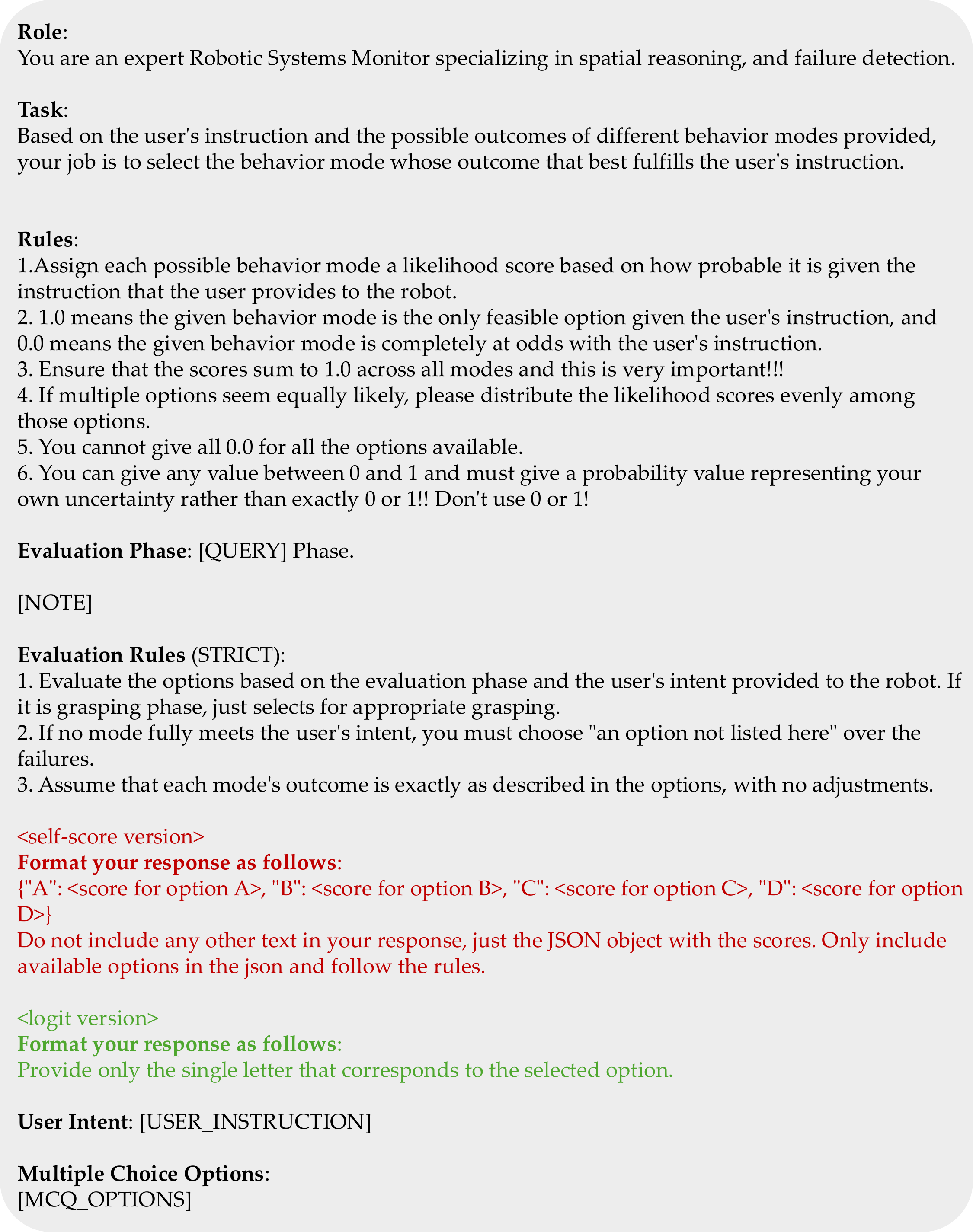}
    \caption{\textbf{Hardware Prompt for Generating Behavior Probability Conditioned On the Intent}}
    \label{fig:hardware_prompt_behavior_probability}
\end{figure}

\begin{figure}
    \centering
    \includegraphics[width=0.9\linewidth]{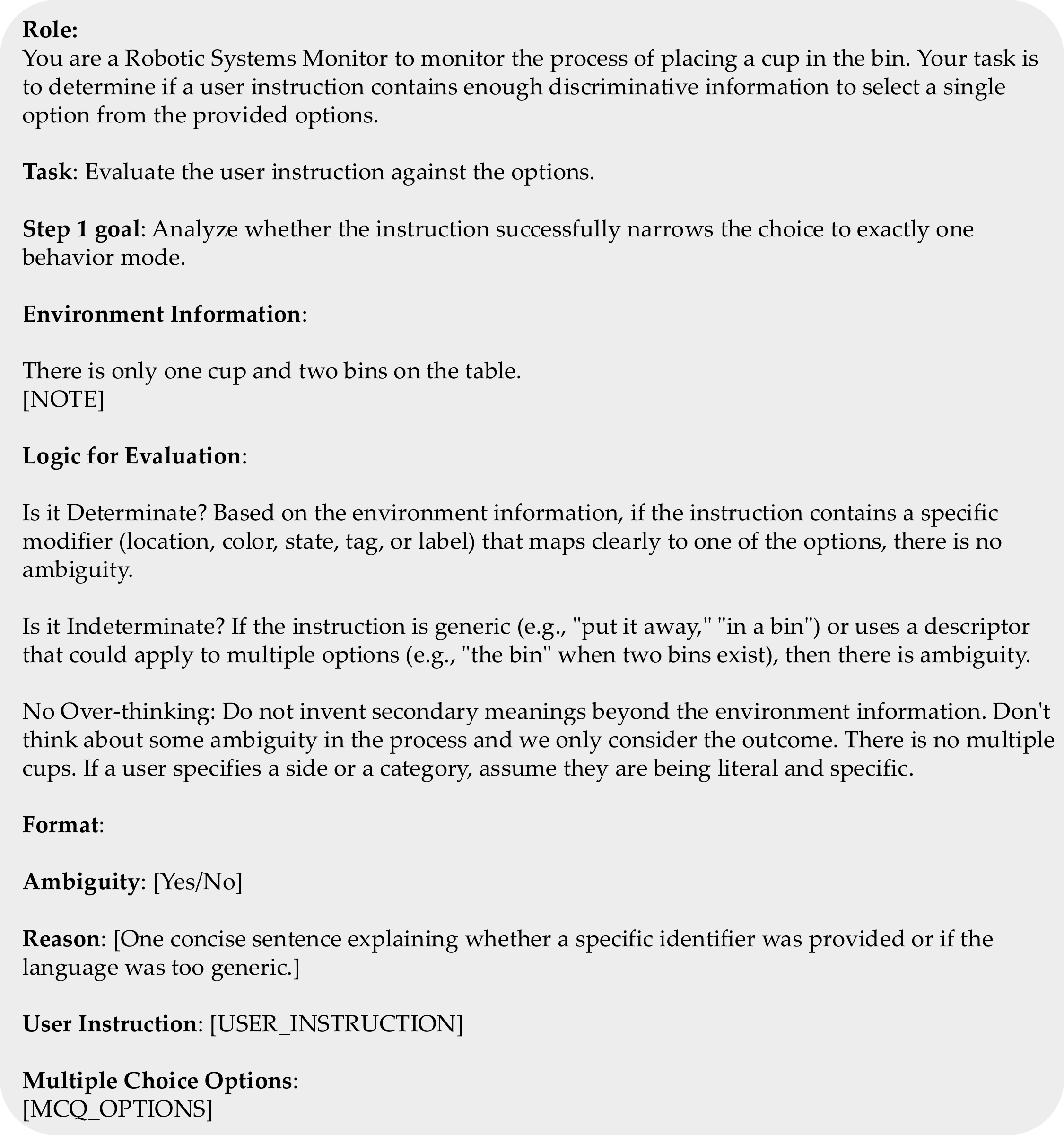}
    \caption{\textbf{Hardware Prompt For Asking VLM to Analyze Ambiguity First in CoT Reasoning}}
    \label{fig:hardware_prompt_cot_step1}
\end{figure}
\begin{figure}
    \centering
    \includegraphics[width=0.9\linewidth]{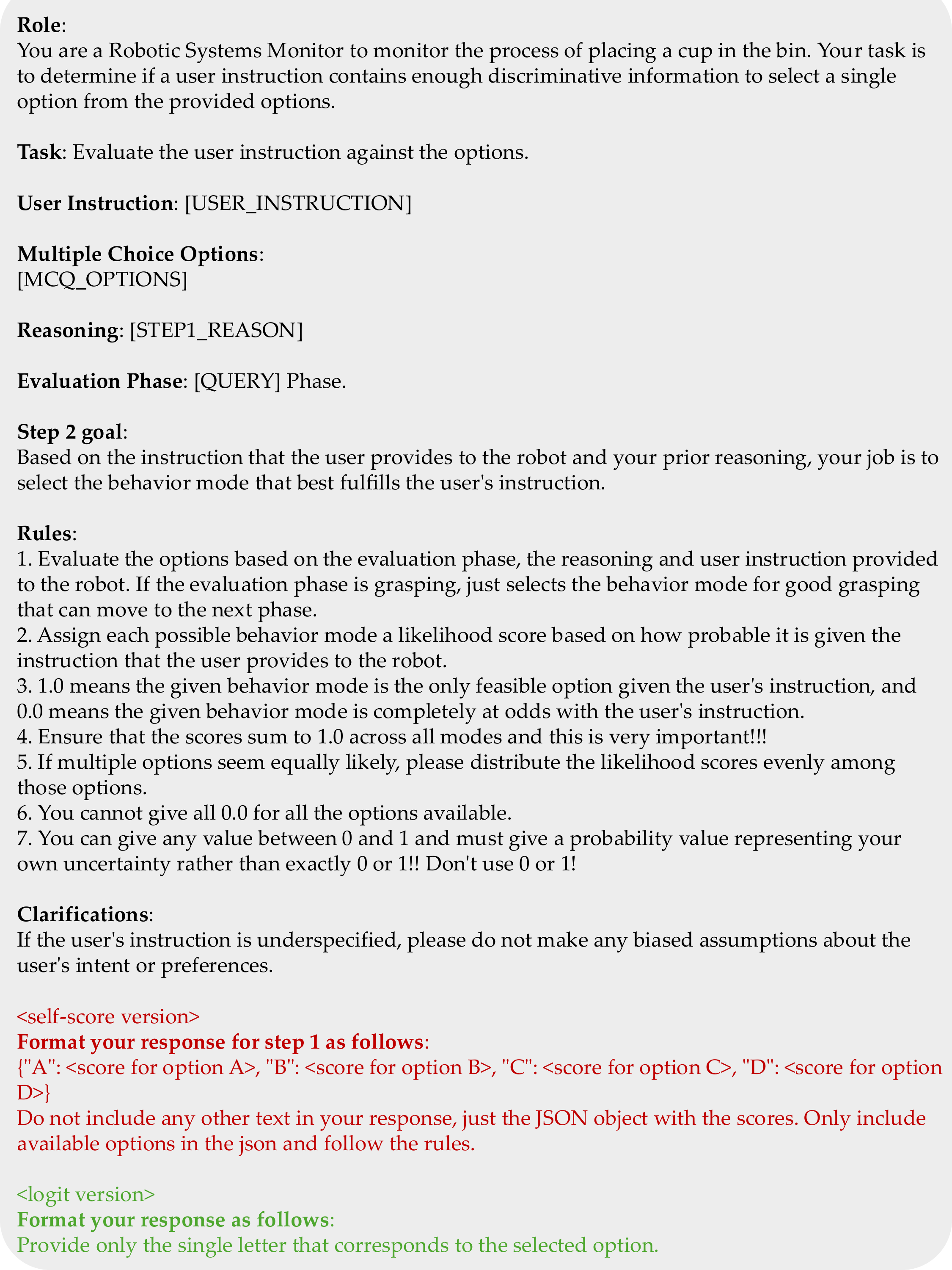}
    \caption{\textbf{Hardware Prompt For Asking VLM to Score Options Based On Previous Reasoning in CoT Reasoning}}
    \label{fig:hardware_prompt_cot_step2}
\end{figure}
\begin{figure}
    \centering
    \includegraphics[width=0.9\linewidth]{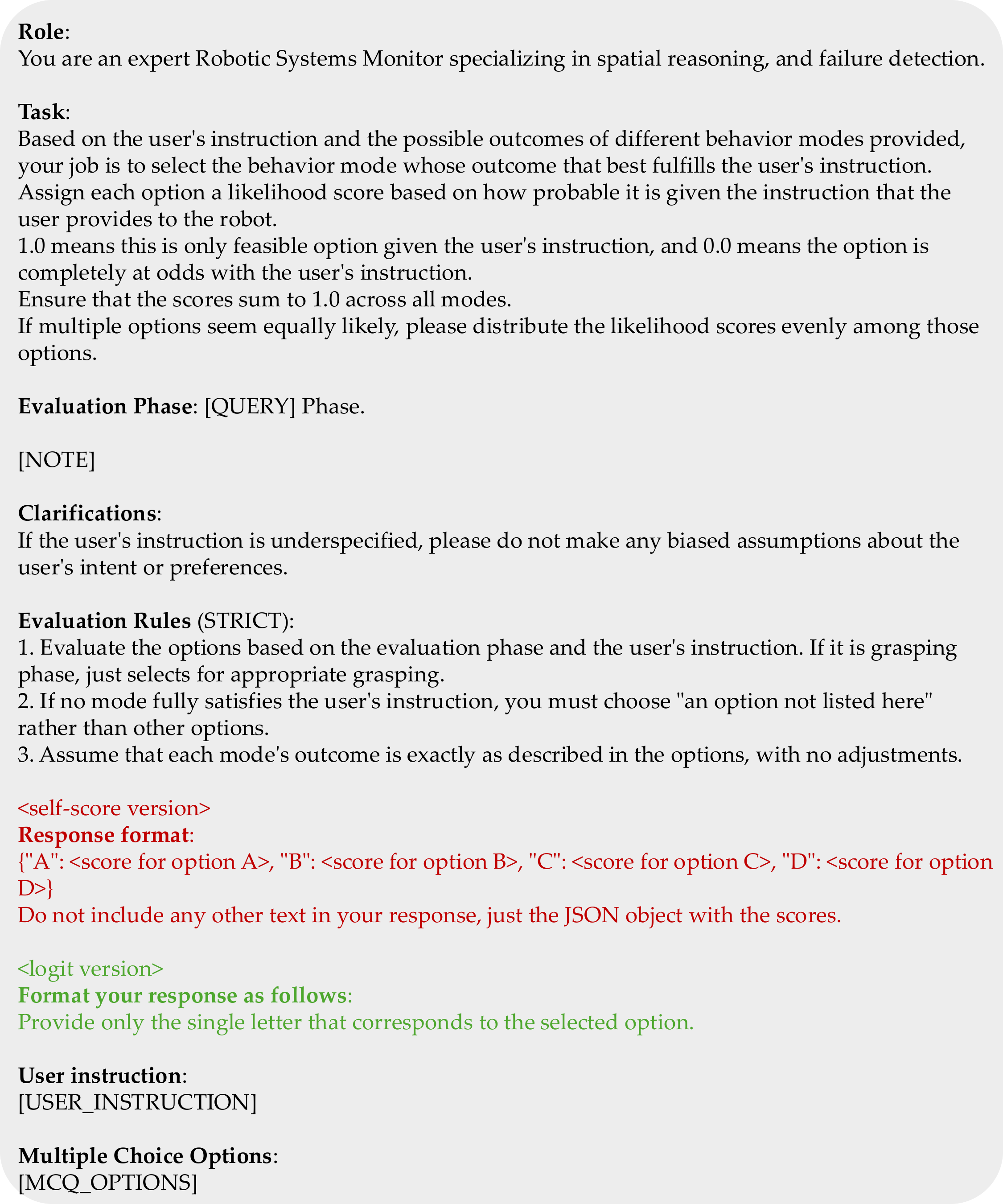}
    \caption{\textbf{Hardware Prompt For Directly Generating Scores for Available Options}}
    \label{fig:hardware_prompt_vanilla_baseline}
\end{figure}

\begin{figure}
    \centering
    \includegraphics[width=0.9\linewidth]{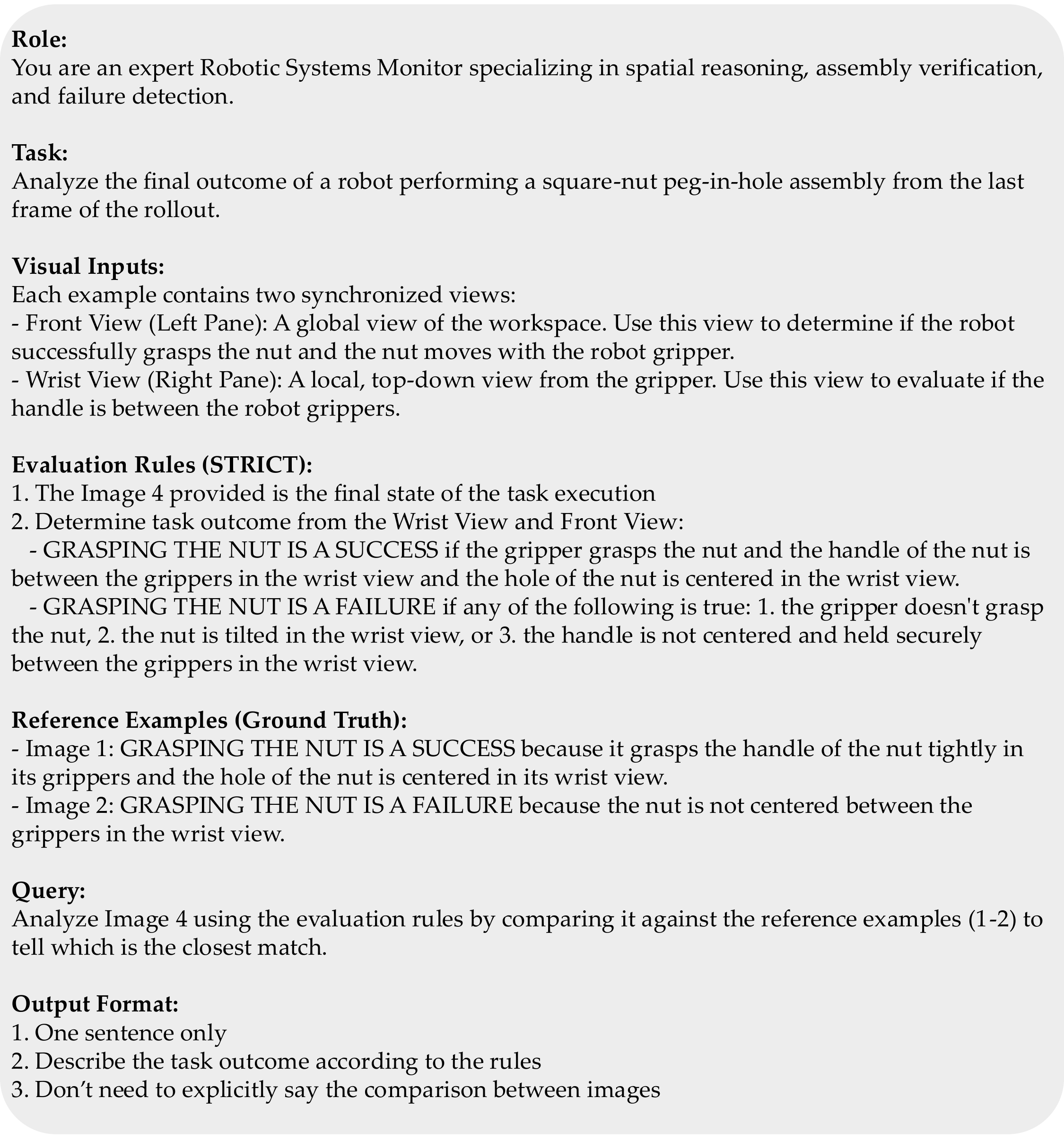}
    \caption{\textbf{Simulation Prompt For Behavior Translation (Grasp)}}
    \label{fig:sim_prompt_translate_grasp}
\end{figure}
\begin{figure}
    \centering
    \includegraphics[width=0.9\linewidth]{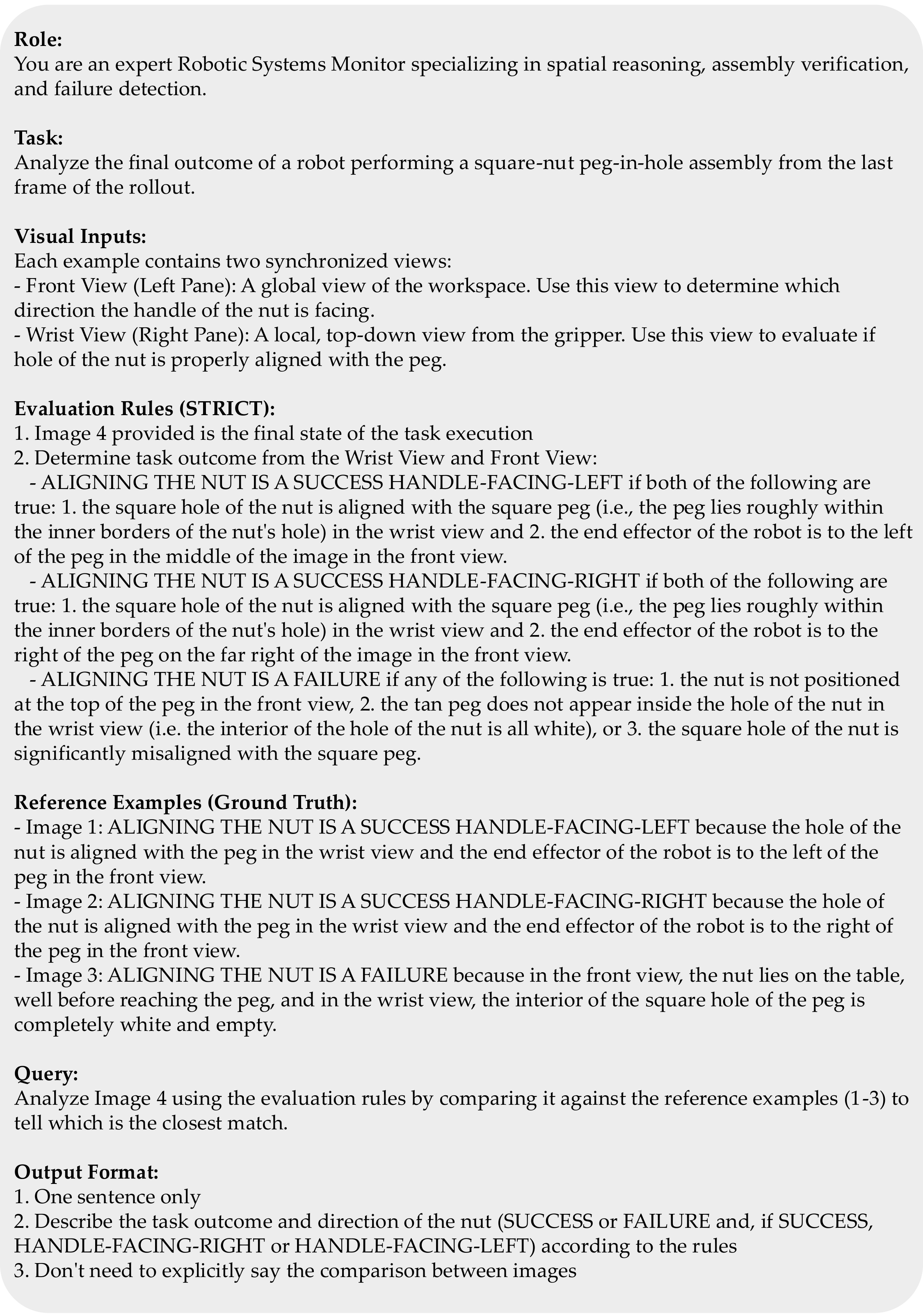}
    \caption{\textbf{Simulation Prompt For Behavior Translation (Place)}}
    \label{fig:sim_prompt_translate_place}
\end{figure}

\begin{figure}
    \centering
    \includegraphics[width=0.9\linewidth]{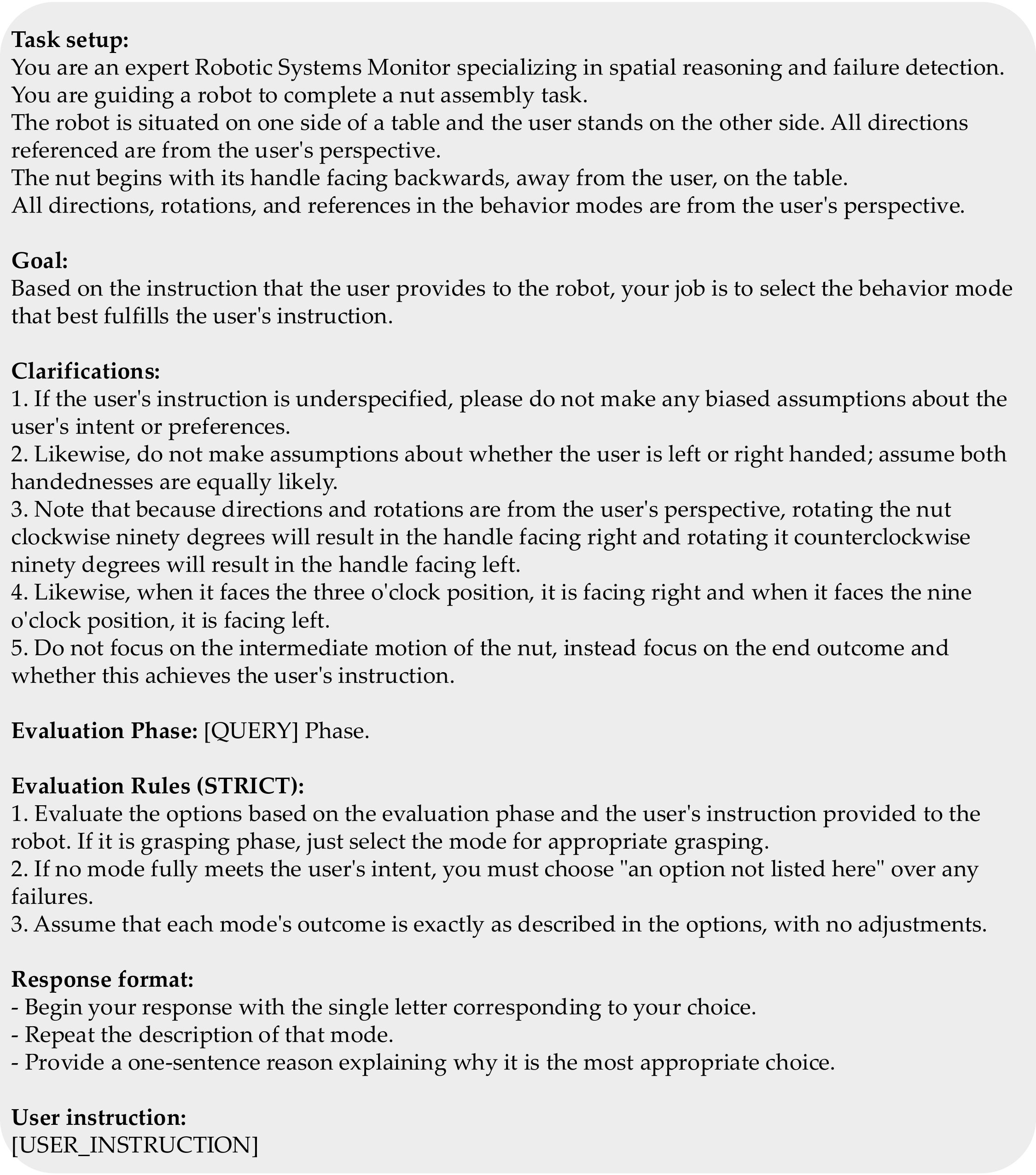}
    \caption{\textbf{Simulation Prompt for VLM Verification in Forewarn}}
    \label{fig:sim_prompt_forewarn_verifier}
\end{figure}

\begin{figure}
    \centering
    \includegraphics[width=0.9\linewidth]{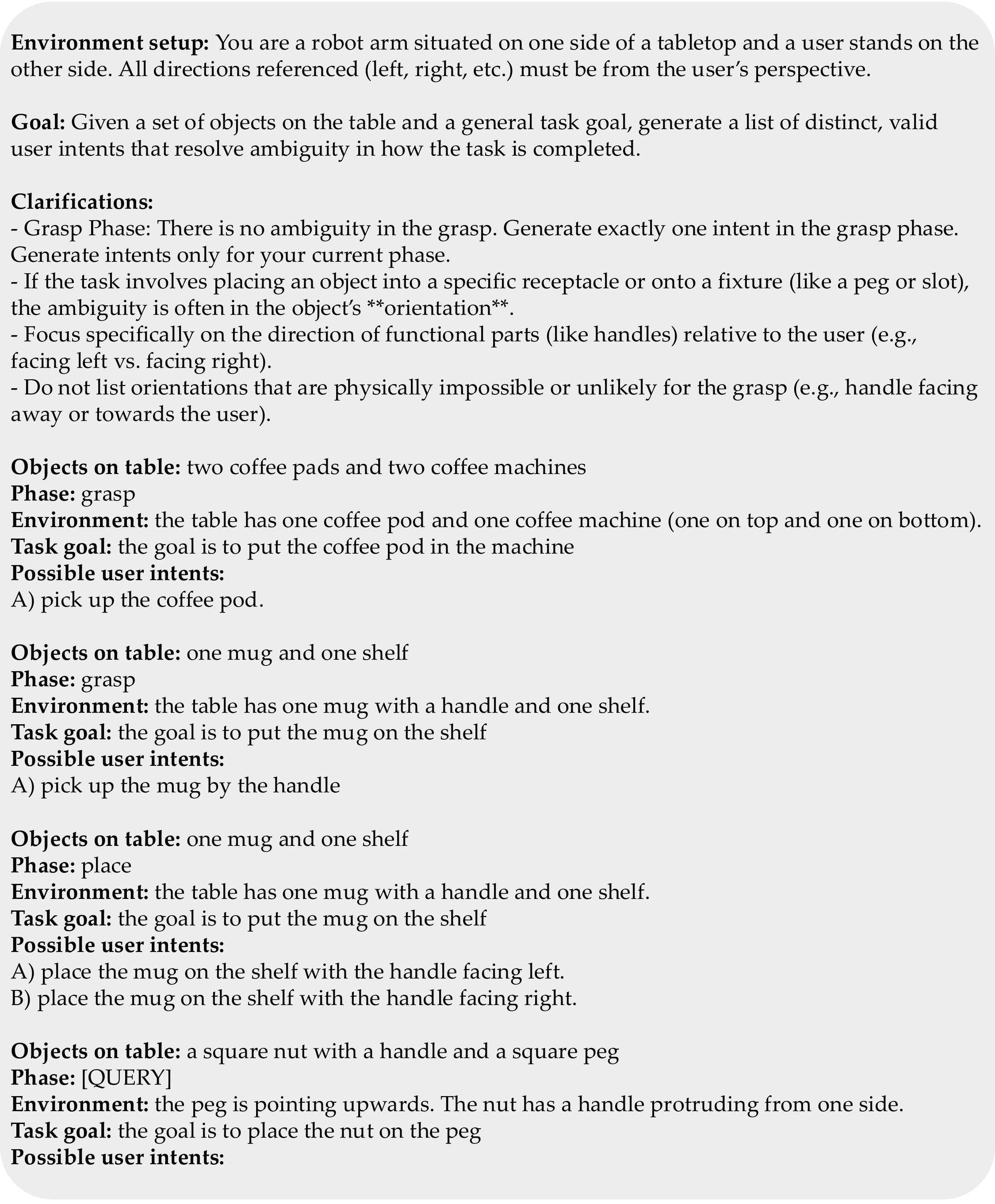}
    \caption{\textbf{Simulation Prompt for Generating Possible Intents}}
    \label{fig:sim_prompt_intent_generation}
\end{figure}

\begin{figure}
    \centering
    \includegraphics[width=0.9\linewidth]{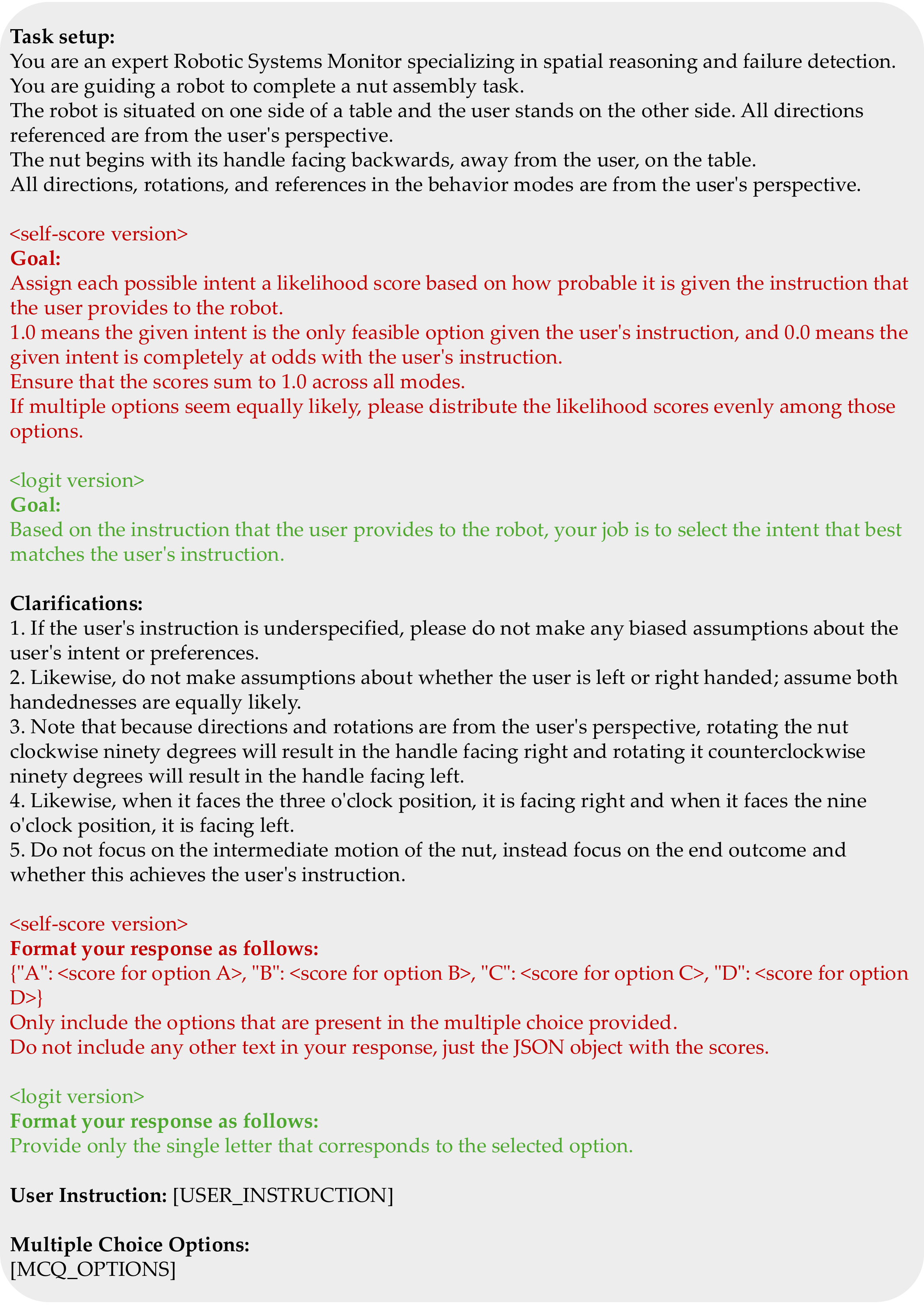}
    \caption{\textbf{Simulation Prompt for Inferring Intent Probability}}
    \label{fig:sim_prompt_intent_probability}
\end{figure}
\begin{figure}
    \centering
    \includegraphics[width=\linewidth]{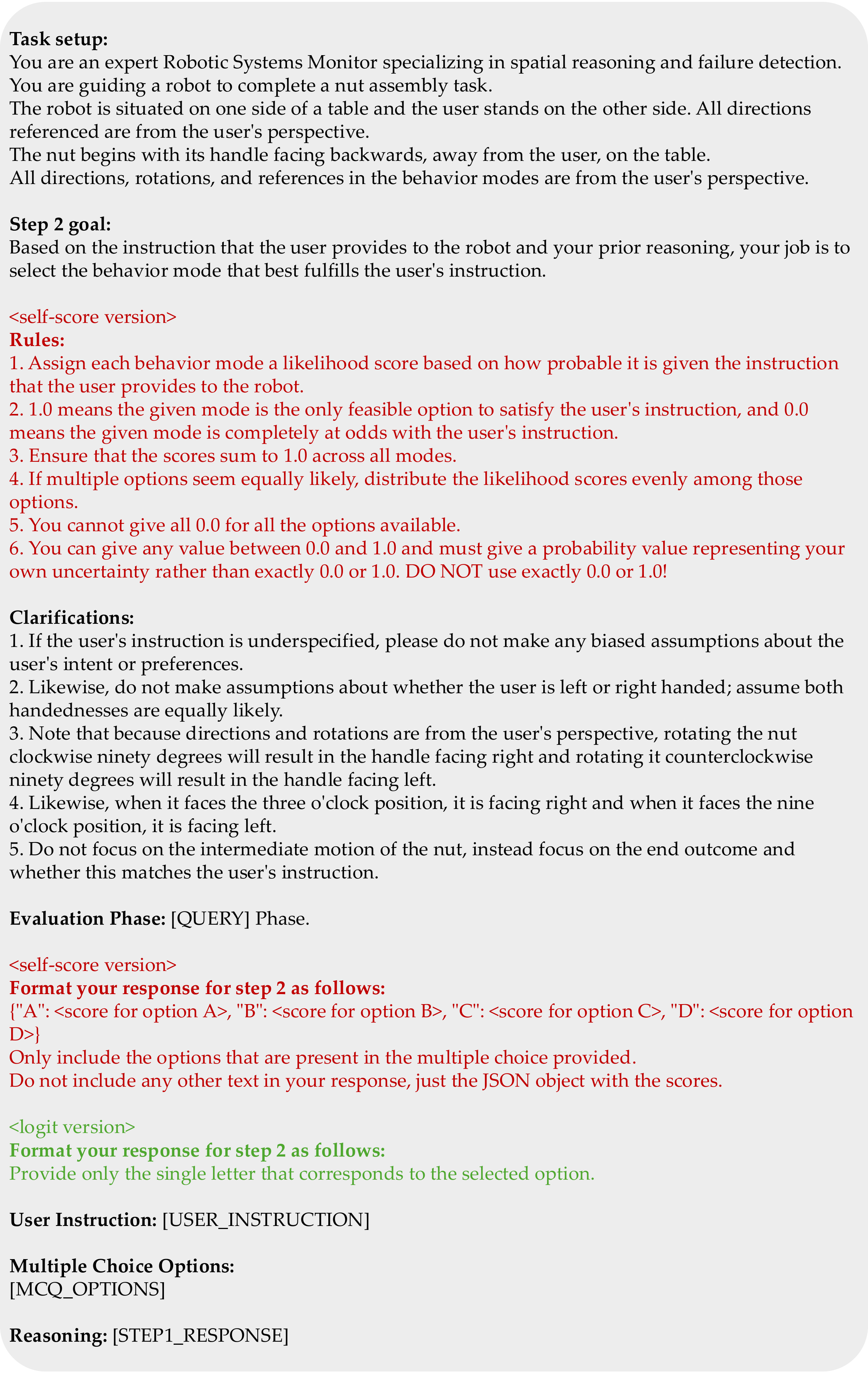}
    \caption{\textbf{Simulation Prompt For Asking VLM to Analyze Ambiguity First in CoT Reasoning}}
    \label{fig:sim_prompt_cot_step1}
\end{figure}

\begin{figure}
    \centering
    \includegraphics[width=\linewidth]{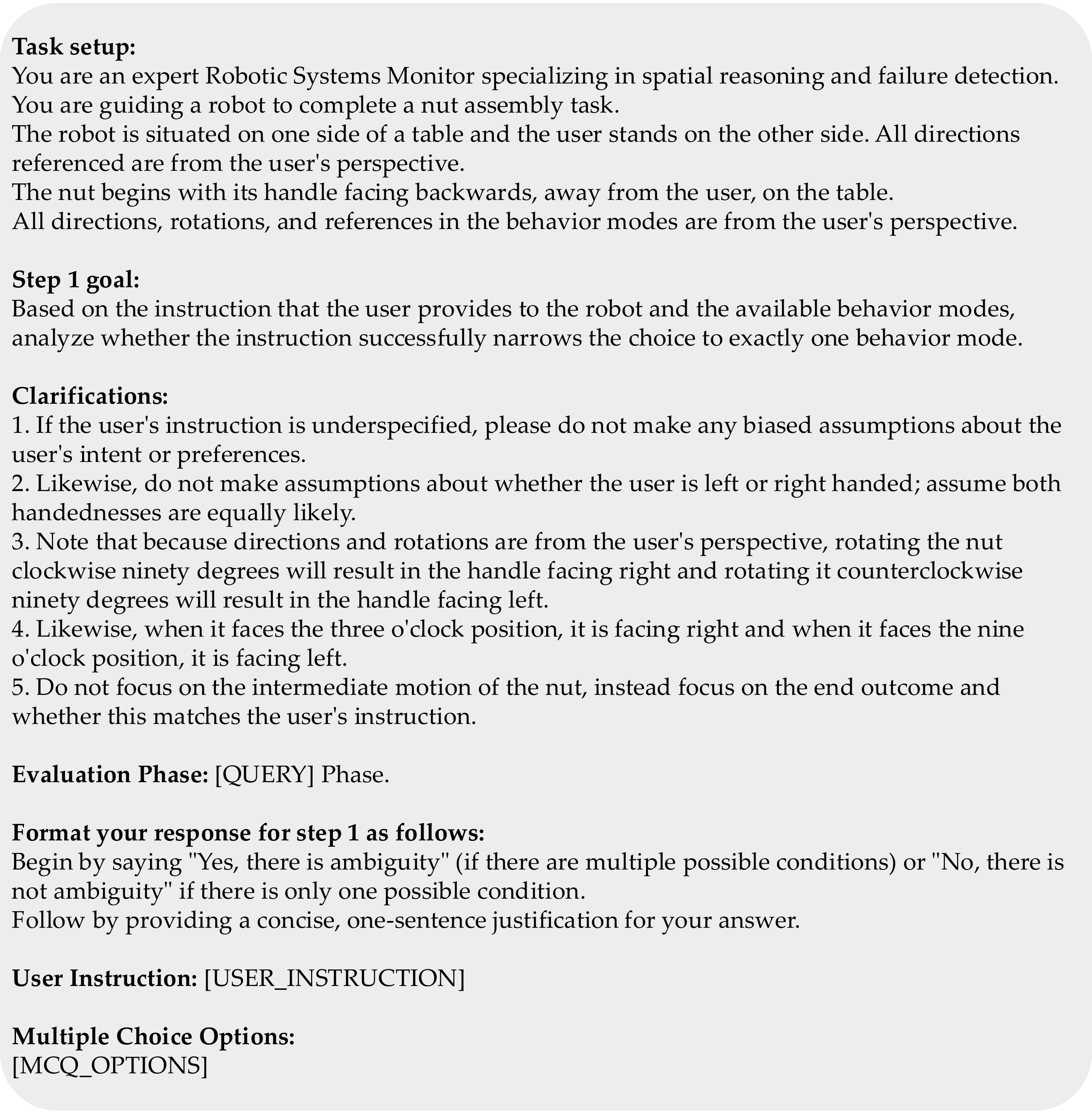}
    \caption{\textbf{Simulation Prompt For Asking VLM to Score Options Based On Previous Reasoning in CoT Reasoning}}
    \label{fig:sim_prompt_cot_step2}
\end{figure}
\begin{figure}
    \centering
    \includegraphics[width=\linewidth]{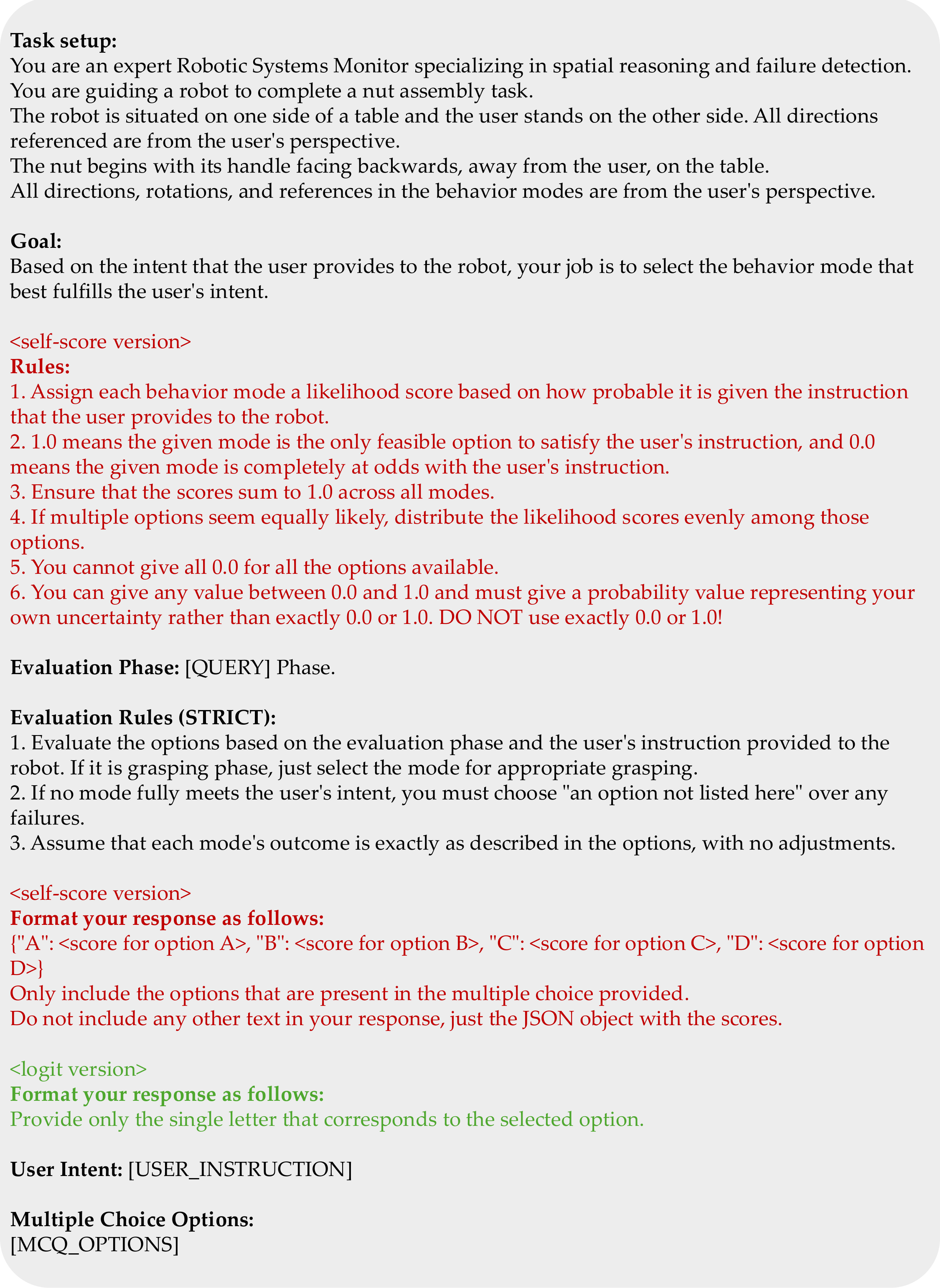}
    \caption{\textbf{Simulation Prompt for Generating Behavior Probability Conditioned On the Intent}}
    \label{fig:sim_prompt_behavior_probability}
\end{figure}


\begin{figure}
    \centering
    \includegraphics[width=\linewidth]{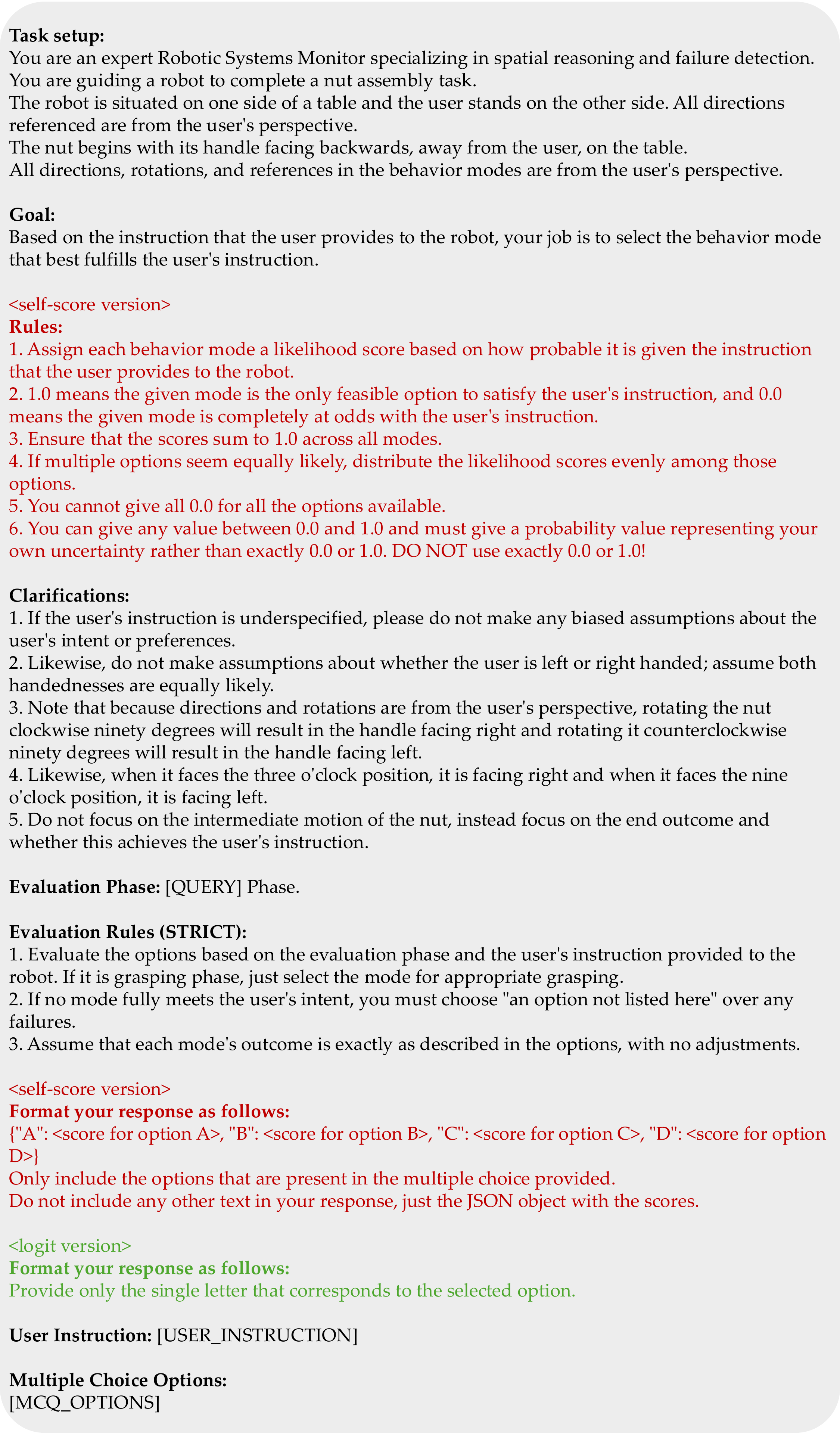}
    \caption{\textbf{Simulation Prompt For Directly Generating Scores for Available Options}}
    \label{fig:sim_prompt_vanilla_baseline}
\end{figure}

\end{document}